%% file: main.tex
\documentclass[final,2p,times,twocolumn]{elsarticle}
\usepackage{amsmath,amssymb,amsfonts,amsthm}
\usepackage{algorithm, algorithmicx, algpseudocode}
\usepackage{textcomp}
\usepackage{graphicx}
\usepackage{tikz}
\usepackage{multirow}
\usepackage{url}
\usepackage{booktabs}
\usepackage{diagbox} % for diagonal cells in tables
\usepackage[caption=false]{subfig}
\usepackage{enumitem}

\input{tex_preamble}

\addtocontents{toc}{\protect\setcounter{tocdepth}{0}}
\journal{Robotics and Autonomous Systems}
\begin{document}

\input{frontmatter}
\input{tex_intro}
\input{tex_overview}
\input{tex_results}
\input{tex_conclusion}

\bibliographystyle{elsarticle-num-names} 
\bibliography{bib.bib}

\input{tex_vitae}

\addtocontents{toc}{\protect\setcounter{tocdepth}{4}}
\include{tex_supp}

\end{document}

%% file: tex_preamble.tex
\newcounter{thmcounter}
\newtheorem{theorem}{Theorem}

\newcommand{\textc}[1]{\mathtt{#1}} % textsl
\newcommand{\mpar}{\textrm{par}}
\newcommand{\mgpar}{\textrm{gpar}}
\newcommand{\mnew}{\textrm{new}}
\newcommand{\moc}{\textrm{oc}}

\newcommand{\mprog}{\mathrm{prog}}
\newcommand{\mnext}{\mathrm{next}}
\newcommand{\mprev}{\mathrm{prev}}
\newcommand{\mcol}{\mathrm{col}}
\newcommand{\mcrn}{\mathrm{crn}}
\newcommand{\mstart}{\mathrm{start}}
\newcommand{\mgoal}{\mathrm{goal}}
\newcommand{\mmnr}{\mathrm{mnr}}
\newcommand{\mmjr}{\mathrm{maj}}
\newcommand{\mthird}{\mathrm{thd}}
\newcommand{\moverlap}{b_\mathrm{over}}
\newcommand{\mnumcrns}{i_\mcrn}
\newcommand{\mbest}{b}
\newcommand{\mcost}{c}
\newcommand{\mtrace}{d}
\newcommand{\mnlets}{\mathbb{M}}
\newcommand{\mnlet}{m}
\newcommand{\mpos}{p}
\newcommand{\mlinks}{\mathbb{E}}

\newcommand{\mnode}{n}
\newcommand{\mtnode}{\hat{\mnode}}
\newcommand{\mnodebest}{\mnode_\mbest}
\newcommand{\mnodepar}{\mnode_\mpar}
\newcommand{\mnodegpar}{\mnode_\mgpar}
\newcommand{\mnodes}{\mathbb{N}}
\newcommand{\mlink}{e}
\newcommand{\mlinkpar}{\mlink_\mpar}
\newcommand{\mtlink}{\hat{\mlink}}

\newcommand{\mntype}{\eta}
\newcommand{\mtdir}{\kappa}
\newcommand{\mtdirs}{\mathbb{K}}
\newcommand{\mtdirpar}{\mtdir_\mpar}
\newcommand{\mtdirnode}{\mtdir_\mnode}
\newcommand{\mfalse}{\textc{False}}
\newcommand{\mtrue}{\textc{True}}
\newcommand{\mback}{\textc{Back}}
\newcommand{\mfront}{\textc{Front}}
\newcommand{\mnvy}{\textc{Vy}}
\newcommand{\mnvu}{\textc{Vu}}
\newcommand{\mney}{\textc{Ey}}
\newcommand{\mneu}{\textc{Eu}}
\newcommand{\mnun}{\textc{Un}}
\newcommand{\mntm}{\textc{Tm}}
\newcommand{\mnoc}{\textc{Oc}}
\newcommand{\mray}{\lambda}
\newcommand{\mrtype}{\rho}
\newcommand{\mrvu}{\textc{Ru}}
\newcommand{\mrvy}{\textc{Ry}}
\newcommand{\mrvn}{\textc{Rn}}
\newcommand{\mx}{\mathbf{x}}
\newcommand{\mxbest}{\mx_\mbest}
\newcommand{\mxstart}{\mx_\mathrm{start}}
\newcommand{\mxgoal}{\mx_\mathrm{goal}}
\newcommand{\mxtrace}{\mx_\mtrace}
\newcommand{\mxpar}{\mx_\mpar}
\newcommand{\mv}{\mathbf{v}}
\newcommand{\mvray}{\mv_\mathrm{ray}}
\newcommand{\mvpar}{\mv_\mpar}
\newcommand{\mvgpar}{\mv_\mgpar}
\newcommand{\mvprog}{\mv_\mprog}
\newcommand{\mcprog}{i_\mprog}
\newcommand{\mvnext}{\mv_\mnext}
\newcommand{\mvprev}{\mv_\mprev}
\newcommand{\mside}{\sigma}
\newcommand{\msidepar}{\mside_\mpar}
\newcommand{\msidenode}{\mside_\mnode}
\newcommand{\msidetrace}{\mside_\mtrace}
\newcommand{\mquery}{q}
\newcommand{\mqtype}{\psi}
\newcommand{\mqcast}{\textc{Cast}}
\newcommand{\mqtrace}{\textc{Trace}}
\newcommand{\mvocpar}{\mv_{\moc\mpar}}

\newcommand{\sdot}{\boldsymbol{\cdot}}
\newcommand{\mdot}{\boldsymbol{.}}
\newcommand{\mora}[1]{\protect\overrightarrow{#1}}

\DeclareMathOperator{\fnode}{N}
\DeclareMathOperator{\ftnode}{\hat{N}}
\DeclareMathOperator{\fpos}{P}
\DeclareMathOperator{\fx}{X}
\DeclareMathOperator{\flink}{E_1}
\DeclareMathOperator{\flinks}{E}
\DeclareMathOperator{\fnlets}{M}
\DeclareMathOperator{\fnlet}{M_1}
\DeclareMathOperator{\fsgn}{sgn}
\DeclareMathOperator{\fquery}{Q}
\DeclareMathOperator{\fray}{\Lambda}
\DeclareMathOperator{\fxcol}{\fx_\mcol}
\DeclareMathOperator{\fbest}{B}

\newcommand{\rs}{RayScan}
\newcommand{\rsp}{RayScan+}
\newcommand{\rtwo}{R2}
\newcommand{\rtwop}{R2+}

\algnewcommand{\IfThenElse}[3]{% \IfThenElse{<if>}{<then>}{<else>}
  \State \algorithmicif\ #1\ \algorithmicthen\ #2\ \algorithmicelse\ #3}
\algnewcommand{\IfThen}[2]{% \IfThen{<if>}{<then>}
  \State \algorithmicif\ #1\ \algorithmicthen\ #2}
\algnewcommand{\Break}{\textbf{break}}
\algnewcommand{\Continue}{\textbf{continue}}
\algnewcommand{\An}{\textbf{and\ }}
\algnewcommand{\Or}{\textbf{or\ }}
\algnewcommand{\Not}{\textbf{not\ }}
\algnewcommand\algorithmicmacro{\textbf{macro}}
\algnewcommand\algorithmicforeach{\textbf{for each}}
\algdef{S}[FOR]{ForEach}[1]{\algorithmicforeach\ #1\ \algorithmicdo}
\algdef{SE}[DOWHILE]{DoWhile}{EndDoWhile}{\algorithmicdo}[1]{\algorithmicwhile\ #1}%
\renewcommand*\Call[2]{\textproc{#1}(#2)} %allow nested Call
\usetikzlibrary {shapes.geometric, shapes.misc, shapes.symbols, shapes.multipart}
\usetikzlibrary {arrows.meta, decorations.markings, decorations.pathreplacing}
\usetikzlibrary{calc, bending, backgrounds, fadings}

\def\u{0.25cm}

\def\ul{0.5cm}
\def\uss{0.1cm}
\def\um{0.15cm}
\colorlet{swatch_blue} {yellow!10!cyan!80!blue}
\colorlet{swatch_obs} {white!30!lightgray}
\colorlet{swatch_bluegray}{blue!30!gray}
\colorlet{swatch_stree} {white!10!magenta!50!red}
\colorlet{swatch_ttree} {black!60!green}
\tikzset {
    blue pt/.style n args={2}{minimum size=2mm, inner sep=0, outer sep=0, circle, fill=swatch_blue, label={[swatch_blue, #1] #2}},
    blue pt/.default={}{},
    blue circ/.style n args={2}{minimum size=2mm, inner sep=0, outer sep=0, circle, draw=swatch_blue, fill=white, thick, label={[swatch_blue, #1] #2}},
    blue circ/.default={}{},
    black pt/.style n args={2}{minimum size=2mm, inner sep=0, outer sep=0, circle, fill=black, label={[black, #1] #2}},
    black pt/.default={}{},
    cross pt/.style n args={2}{minimum size=2mm, inner sep=0, outer sep=0, cross out, draw, label={[black, #1] #2}},
    cross pt/.default={}{},
    trace/.style n args={0}{line width=1mm, orange, {Triangle Cap[]}-{Triangle Cap[]}},
    trace2/.style n args={0}{line width=0.8mm, swatch_blue, -{Triangle Cap[]}},
    prune/.style n args={0}{cross out, draw, black, thick, minimum size=3mm, inner sep=0mm},
    bgframe/.style n args={0}{background rectangle/.style={draw=gray, dotted}, framed, tight background},
    vy pt/.style n args={3}{minimum size=2mm, inner sep=0, outer sep=0, circle, draw=black!60!#3, fill=#3, label={[#3, #1] #2}},
    vy pt/.default={}{}{swatch_blue},
    vu pt/.style n args={3}{minimum size=2mm, inner sep=0, outer sep=0, circle, fill=white, draw=#3, thick, label={[#3, #1] #2}},
    vu pt/.default={}{}{swatch_blue},
    ey pt/.style n args={3}{minimum size=2.5mm, inner sep=0, outer sep=0, diamond, draw=black!60!#3, fill=#3, label={[#3, #1] #2}},
    ey pt/.default={}{}{swatch_blue},
    eu pt/.style n args={3}{minimum size=2.5mm, inner sep=0, outer sep=0, diamond, draw=#3, fill=white, thick, label={[#3, #1] #2}},
    eu pt/.default={}{}{swatch_blue},
    tm pt/.style n args={3}{minimum size=2mm, forbidden sign, inner sep=0, outer sep=0, draw=#3, fill=white, thick, label={[#3, #1] #2}},%dash=on 2pt off 1pt phase 0pt
    tm pt/.default={}{}{swatch_ttree},
    trtm pt/.style n args={3}{tm pt={#1}{#2}{#3}},
    trtm pt/.default={}{}{},
    oc pt/.style n args={3}{minimum size=1.8mm, inner sep=0, rectangle, outer sep=0, draw=#3, fill=white, thick, label={[#3, #1] #2}},
    oc pt/.default={}{}{swatch_ttree},
    un cross/.style={path picture={ 
      \draw[#1] (path picture bounding box.south east) -- (path picture bounding box.north west) (path picture bounding box.south west) -- (path picture bounding box.north east);
    }},
    un pt/.style n args={3}{minimum size=2mm, inner sep=0, outer sep=0, circle, un cross={#3}, draw=#3, fill=white, thick, label={[#3, #1] #2}},
    un pt/.default={}{}{swatch_ttree},
    any pt/.style n args={3}{circle, inner sep=0, outer sep=0, minimum size=3mm, draw=#3, dash=on 1pt off 1pt phase 0, label={[#3, #1] #2}},
    any pt/.default={}{}{},
    svy pt/.style n args={2}{vy pt={#1}{#2}{swatch_stree}},
    svy pt/.default={}{},
    svu pt/.style n args={2}{vu pt={#1}{#2}{swatch_stree}},
    svu pt/.default={}{},
    sey pt/.style n args={2}{ey pt={#1}{#2}{swatch_stree}},
    sey pt/.default={}{},
    seu pt/.style n args={2}{eu pt={#1}{#2}{swatch_stree}},
    seu pt/.default={}{},
    tvy pt/.style n args={2}{vy pt={#1}{#2}{swatch_ttree}},
    tvy pt/.default={}{},
    tvu pt/.style n args={2}{vu pt={#1}{#2}{swatch_ttree}},
    tvu pt/.default={}{},
    tey pt/.style n args={2}{ey pt={#1}{#2}{swatch_ttree}},
    tey pt/.default={}{},
    ttm pt/.style n args={2}{tm pt={#1}{#2}{swatch_ttree}},
    ttm pt/.default={}{},
    toc pt/.style n args={2}{oc pt={#1}{#2}{swatch_ttree}},
    toc pt/.default={}{},
    tun pt/.style n args={2}{un pt={#1}{#2}{swatch_ttree}},
    tun pt/.default={}{},
    separate/.style n args={2}{rounded rectangle, draw=black, fill=white, minimum size=3.5mm, inner xsep=3.5mm, inner ysep=0mm, outer sep=1pt, label={[#1] #2}}, %dash=on 3pt off 1pt phase 0pt
    separate/.default={}{},
    link/.style n args={1}{draw=#1, -{Latex[#1, length=2mm, width=1.5mm]}, dash=on 2pt off 1pt phase 0pt, thick}, %-{Arc Barb[arc=60, reversed, scale=2]
    link/.default={},
    qlink/.style n args={1}{link={#1}, ultra thick, solid},
    qlink/.default={},
    xlink/.style n args={1}{link={#1}, {}-{}},
    xlink/.default={},
    sxlink/.style={xlink={swatch_stree}},
    txlink/.style={xlink={swatch_ttree}},
    slink/.style={link={swatch_stree}},
    tlink/.style={link={swatch_ttree}},
    sqlink/.style={qlink={swatch_stree}},
    tqlink/.style={qlink={swatch_ttree}},
    rayl/.style={draw=black, solid, -{Triangle[angle'=90, open, left] Triangle[angle'=90, open, left]}},
    rayr/.style={draw=black, solid, -{Triangle[angle'=90, open, right] Triangle[angle'=90, open, right]}},
    rayprog/.style={draw=black, solid, -{>>}},
    merge/.style={double distance=1pt, thick}, 
}
\pgfdeclarelayer{background}
\pgfdeclarelayer{foreground}
\pgfsetlayers{background,main,foreground}
\tikzset{
    pics/merge grp/.style n args={6}{ code={ 
        \node [separate={#1}{#2}] at (0, 0) {};
        \draw [merge] (-\um, 0) -- (\um, 0);
        \node (#3) [#4, xshift=-\um] at (0, 0) {};
        \node (#5) [#6, xshift=\um] at (0, 0) {};
    }}
}
\tikzset{
    pics/trace grp/.style n args={4}{ code={ 
        \node [separate={#1}{#2}, inner xsep=3mm] at (0,0) {};
        \node (#3) [trtm pt, xshift=-\uss] at (0,0) {};
        \node (#4) [trtm pt, xshift=\uss] at (0,0) {};
    }}
}

%% file: frontmatter.tex
\begin{frontmatter}
\title{Evolving \rtwo{} to \rtwop{}: Optimal, Delayed Line-of-sight Vector-based Path Planning}

%%%%%%%%%%%%%%%%%%%%%%%%% AUTHORS %%%%%%%%%%%%%%%%%%%%%%%%%%%%%%%%%%%%%%%
\author[auth1]{Yan Kai Lai\corref{cor1}}
\ead{lai.yankai@u.nus.edu}
\cortext[cor1]{Corresponding author}
\author[auth1]{Prahlad Vadakkepat}
\ead{prahlad@nus.edu.sg}
\author[auth1]{Cheng Xiang}
\ead{elexc@nus.edu.sg}
\affiliation[auth1]{organization={National University of Singapore, Department of Electrical and Computer Engineering}}

%%%%%%%%%%%%%%%%%%%%%%%%% ABSTRACT %%%%%%%%%%%%%%%%%%%%%%%%%%%%%%%%%%%%%%%
\begin{abstract}

A vector-based any-angle path planner, \rtwo{}, is evolved in to \rtwop{} in this paper.
By delaying line-of-sight, \rtwo{} and \rtwop{} search times are largely unaffected by the distance between the start and goal points, but are exponential in the worst case with respect to the number of collisions during searches. 
To improve search times, additional discarding conditions in the overlap rule are introduced in \rtwop{}.
In addition, \rtwop{} resolves interminable chases in \rtwo{} by replacing ad hoc points with limited occupied-sector traces from target nodes, and simplifies \rtwo{} by employing new abstract structures and ensuring target progression during a trace.
\rtwop{} preserves the speed of \rtwo{} when paths are expected to detour around few obstacles, and searches significantly faster than \rtwo{} in maps with many disjoint obstacles.

% A novel vector-based path planner, R2 (R in two dimensions), is introduced in this paper. 
% R2 is optimal and online, returning any-angle paths by applying heuristic costs to vector-based searches. 
% R2 delays line-of-sight checks to expand the most promising path that has the least detours from the start and goal points.
% As delayed checks can cause severe path cost underestimates, R2 infers the smallest known convex hull, the best hull, of obstacles while moving around them. 
% To construct the best hull, phantom points are introduced, which are imaginary turning points lying on non-convex corners to guide future searches. Tracing rules are introduced to ensure that the estimated path cost from the best hull increases admissibly and monotonically between queues to the open-list.
% The distance between the start and goal points have little impact on R2's performance when compared to the number of line-of-sight checks that collide with obstacles. 
% While having an exponential search time in the worst case with respect to the number of collided line-of-sight checks, R2 is much faster than state-of-the-art when the optimal path is expected to turn around few obstacles, especially on large maps with few disjoint and non-convex obstacles.
\end{abstract}

%%%%%%%%%%%%%%%%%%%%%%%%% HIGHLIGHTS %%%%%%%%%%%%%%%%%%%%%%%%%%%%%%%%%%%%%%%
\begin{highlights}
\item \rtwop{} simplifies \rtwo{} by ensuring target progression in traces, which eliminates a complicated tracing phase occurring before recursive traces.
\item \rtwop{} simplifies \rtwo{} by introducing new abstract structures ($S$-tree, $T$-tree).
\item \rtwop{} performs much faster than \rtwo{} in maps with many disjoint obstacles.
\item \rtwop{} and \rtwo{} return paths quickly when paths are expected to detour around few obstacles.
\item  By replacing ad hoc points with limited recursive occupied-sector traces, \rtwop{} ensures terminability when no paths can be found
% \item Combines vector-based navigation with any-angle path planning.
% \item Delayed line-of-sight checks to expand the straightest path possible.
% \item Smallest-known convex hull is inferred from traced contour to estimate cost and guide searches.
% \item Phantom points are imaginary turning points to guide future searches.
% \item Fast if path is expected to have few turning points, regardless of distance.
\end{highlights}

%%%%%%%%%%%%%%%%%%%%%%%%% KEYWORDS %%%%%%%%%%%%%%%%%%%%%%%%%%%%%%%%%%%%%%%
\end{frontmatter}

%% file: tex_intro.tex
\section{Introduction}

Vector-based algorithms are optimal any-angle path planners that prioritise ray casts between points to find paths, and search along obstacle contours that obstruct the casts.
By avoiding free-space expansions and searching only along contours, vector-based planners can find paths more quickly than A* \citep{bib:astar}, and other any-angle planners like Theta* \citep{bib:thetastar} and ANYA \citep{bib:anya}.
The performance increase can be between ten to a hundred times faster, especially on large, sparse maps with few obstacles \cite{bib:r2}. 
Vector-based planners are a nascent class of planners with few algorithms, and include \rtwo{} \citep{bib:r2}, \rs{} \citep{bib:rayscan}, and \rsp{} \citep{bib:rayscanp}.

\rtwo{} is a vector-based algorithm that delays line-of-sight checks to find paths rapidly, expanding only the most promising points that lie close to the straight line between the start and goal points.
To prevent underestimates and ensure monotonically increasing costs when line-of-sight checks are delayed, novel concepts such as phantom points and best hulls are introduced in \rtwo{}. 
The concepts infer the convex hulls of non-convex obstacles, allowing cost-to-go estimates to be calculated more reliably when visibility between nodes are not known.

Several problems remain with \rtwo{}, which this work solves with \rtwop{}.
To ensure that \rtwop{} can terminate if no path is found, ad hoc points from \rtwo{} are replaced limited recursive traces from target nodes in \rtwop{}.
By guaranteeing target progression when tracing an obstacle in \rtwop{}, a complicated phase of tracing in \rtwo{} can be removed completely.
Instead of counting the number of nodes placed by a trace before the trace can be interrupted in \rtwo{}, the number of corners traced are counted in \rtwop{}, allowing the interrupting procedure to be decoupled from the placement rule.

While the delayed line-of-sight checks cause \rtwo{}'s performance to be largely invariant to the distance between the start and goal points, 
\rtwo{}'s search complexity becomes exponential with respect to the number of collided casts.
To achieve better performance on dense maps with many disjoint obstacles, 
the overlap rule from \rtwo{} is extended to include a pruning scheme similar to the G-value pruning from \citep{bib:dps}.

The abstract structures in \rtwo{} are refined in \rtwop{}. Two trees are introduced in \rtwop{} instead of one, with one tree rooted at the start point and another at the goal point. 
Links, which are connections between two nodes, are  the basic units of search in \rtwop{}. 
Links reduce the number of duplicate line-of-sight checks between overlapping paths and provide more intuitive geometrical interpretations of searches.

%% file: tex_overview.tex
\section{Concepts in \rtwop{}}
\rtwop{} relies on casts and traces to find the shortest path. 
Like \rtwo{}, \rtwop{} delays line-of-sight checks to expand turning points with the least deviation from the straight line between the start and goal points.
\rtwop{} is an evolved algorithm of \rtwo{}, primarily focusing on resolving interminable chases in \rtwo{}, and in enhancing search time in maps with highly non-convex obstacles and many disjoint obstacles.
This section describes the nomenclature and structures used in \rtwop{}.

The \textbf{tree-direction} determines the direction of an object along a path from the start to goal node.
Suppose two objects lie on the same path.
An object lying in the \textbf{source} direction of the other leads to the start point on the path.
If it lies in the \textbf{target} direction, it leads to the goal point.

\rtwop{} relies on two node trees connected at their leaf nodes.
The \textbf{source-tree} ($S$-tree) is rooted at the start node, and the \textbf{target-tree} ($T$-tree) is rooted at the goal node. 
The $T$-tree behaves more like a directed graph than a tree.
While $S$-tree nodes are connected to one parent node, a $T$-tree node can be connected to more than one parent.
A $T$-tree node is connected to one parent only if it has an unobstructed path to the goal node.

Nodes in \rtwo{} and \rtwop{} have types, depending on cumulative visibility, cost etc. 
The node types in \rtwop{} are summarized in Table \ref{tab:nodetypes}.
% A $\mtdir$-tree ($\mtdir\in\{S,T\}$) node that has cumulative visibility to the $\mtdir$-tree's root node is $\mnvy$ or $\mney$ type.
Two nodes have \textbf{cumulative visibility} if there is an unobstructed path between them.
For brevity, the term is overloaded, and a $\mtdir$-tree node that has cumulative visibility will have an unobstructed path between the node and the $\mtdir$-tree's node.

\input{tab_nodetypes.tex}
\input{tab_legend.tex}

A \textbf{query} in \rtwop{} is similar to an iteration in A*, where the node tree is modified based on the state of an expanded node.
In \rtwop{}, a query can be a \textbf{casting query}, which modifies the node trees after a line-of-sight check; or a \textbf{tracing query}, which modifies the node trees while tracing along an obstacle contour.
Unless otherwise stated, a \textbf{cast} refers to a \textit{casting query} and a \textbf{trace} refers to a \textit{tracing query} for brevity.

A \textbf{link} connects two nodes, and is used to store dynamic information such as sector-rays and cost.
By shifting the dynamic information from nodes to links, nodes need not be duplicated when paths overlap, and the number of duplicate line-of-sight checks can be reduced.
In the implementation, a link points to (\textbf{anchored}) its child node for efficient memory management.
The link's parent node is anchored by a connected parent link.
A query in \rtwop{} expands a link that is anchored on a leaf node. 
As such, nodes and links on the $S$-tree lie in the source direction of a query, and nodes and links on the $T$-tree lie in the target direction of a query.

Tables \ref{tab:nodetypes} and \ref{tab:legend} illustrate the symbols used in figures. 
Fig. \ref{fig:tree} describes the trees with respect to nodes, links and queries. % TODO show tracing query transfer.

\input{fig_tree.tex}

\section{Evolving \rtwo{} to \rtwop{}}
% \rtwop{} is similar to \rtwo{}.
% The same tracing rules are used -- the progression, pruning, angular-sector, occupied-sector, and placement rules.
% Each trace infers the \textbf{best-hull}, which is smallest convex hull known of the traced obstacle, to allow  cost estimates to increase monotonically.
% Two traces in the opposite direction ($L$ and $R$ traces) are generated by collided casts from the collision point, and a \textbf{third-trace} from the source node may be generated if the destination is the goal point.
% The algorithms delay line-of-sight checks and cast greedily to target nodes, verifying line-of-sight only when paths overlap.
% The shortest path is found when a cast reaches a $\mnvy$ $T$-tree node from a $\mnvy$ $S$-tree node.

The following subsections describe the changes made to evolve \rtwo{} to \rtwop{}. 
% In \rtwop{}, phantom points that are on the same best-hull as a trace cannot be cast to by the trace (Sec. \ref{sec:phantom});
In \rtwop{}, short occupied-sector traces from target nodes in \rtwop{} supersedes the ad hoc points from \rtwo{} (Sec. \ref{sec:tgtocsec});
the complicated tracing phase before a recursive trace from the source node in \rtwo{} is replaced by simpler corrective steps in \rtwop{} (Sec. \ref{sec:tgtprog});
the interrupt rule counts corners in \rtwop{} instead of nodes placed (Sec. \ref{sec:interrupt});
and the overlap rule is modified to discard expensive paths (Sec. \ref{sec:overlap}).

\subsection{Limited, Target Recursive Occupied-Sector Trace} \label{sec:tgtocsec}
Recursive occupied-sector traces for target nodes (\textbf{target oc-sec trace}) are not implemented in \rtwo{} due to chases \citep{bib:r2}. 
A chase occurs when two traces in the same direction try to cast to each other but are unable to do so as the traces are on the same contour.
% (i) a recursive trace from a phantom point can lead to the wrong path; and (ii) a recursive occupied sector trace from a target node can try to reach the calling trace if both are on the same contour, resulting in both traces trying to reach each other (a chase). 
For \rtwo{} to be complete in the absence of target oc-sec traces, two ad hoc points $\mnode_{ad,b}$ and $\mnode_{ad,c}$ are introduced in \rtwo{}.
While the ad hoc points can reduce the number of chases for \rtwo{} to find a path, the points do not eliminate chases completely, and \rtwo{} can be interminable if no path exists.

\input{fig_tgtocsec}

To prevent interminable chases in \rtwop{}, limited target oc-sec traces are implemented, and the ad hoc points $\mnode_{ad,b}$ and $\mnode_{ad,c}$ are removed.
If the target node is $\mside$-sided, the limited trace ends at the first $(-\mside)$-side corner from the target node, where a new $\mnoc$ node is placed.
The $\mnoc$ node prevents a subsequent oc-sec trace from occurring, especially if the calling tracing query is on a different best-hull. 
If the query arrives at the same contour as the oc-sec trace, a chase can occur.
Fig. \ref{fig:tgtocsec} illustrates a limited oc-sec trace.

A limited recursive trace is simpler to implement than ad hoc points and a full recursive trace.
Ad hoc points require extensive calculations to calculate intersections, and a target oc-sec trace traces toward source nodes, which require the tracing rules to be adjusted. 

% A recursive occupied-sector trace will never occur from a phantom point that is placed by the same trace.
% A phantom point has the same side as the trace that placed it, and the occupied-sector rule will not examine nodes that has the same side as a trace.
% If the phantom point belongs to a different best-hull, it is placed by a different trace that was interrupted before the current trace. 
% A recursive trace is allowed in this case, but no special instruction is required because such a phantom point would have been converted to a $\mntm$ point in Sec. \ref{sec:phantom}.

% In \rtwop{}, the ad hoc point $\mnode_{ad,a}$ from \rtwo{} is superseded by an $\mnun$ node.

\subsection{Ensuring Target Progression} \label{sec:tgtprog}
\input{fig_tgtprog1}
\input{fig_tgtprog2}

During a trace, the angular progression with respect to a target node (\textbf{target progression}) may be decreasing, especially when the trace is interrupted.
To ensure target progression when a trace is interrupted, \rtwo{} enters a special tracing phase and exits once there is progression for all target nodes.
The special phase is complicated, requiring the rules to be modified, and the query to backtrack to the node where the trace is interrupted.
If target progression is not ensured, nodes may be incorrectly pruned in a subsequent trace, and the algorithm may not be able to find a path.
As target progression seldom decreases, it is not efficient to implement the complicated special phase.

% In \rtwo{}, when recursive angular-sector (\textbf{ang-sec}) and oc-sec traces are called from source nodes (\textbf{source recursive traces)}, the calling trace may not have progressed with respect to the target nodes (no \textbf{target progression}).
% If the recursive traces proceed, the target nodes may be incorrectly pruned in a subsequent trace, and the algorithm would become incomplete.
% % For brevity, let the term \textit{target progression} represent the situation when a trace is progressed with respect to all of its target nodes, and \textit{source progression} represent the situation when a trace is progressed with respect to its source node.
% To ensure target progression, \rtwo{} relies on a special tracing phase once a recursive call is identified.
% As the phase requires special conditions to move to a corner where there is target progression, and to backtrack once target progression is achieved, the special phase complicates the algorithm.

Instead of entering a special phase, \rtwop{} applies short corrective steps whenever target progression becomes likely to decrease (Cases P1 and P2). 
\textbf{Case P1} occurs when the initial traced edge of a recursive angular-sector trace is angled in a way that results in no target progression (see Fig. \ref{fig:tgtprog1}).
Let the recursive trace be $(-\mside)$-sided and the calling trace be $\mside$-sided. 
Case P1 is resolved by placing a $(-\mside)$-sided unreachable $\mnun$ node at the $\mside$-side of the initial edge traced by the recursive angular-sector trace. 
If there is target progression at the initial edge, the $\mnun$ node would be immediately pruned.
If there is no target progression at the initial edge, the $\mnun$ node ensures progression by rerouting the path.
A subsequent query that reaches the $\mnun$ can be discarded as a cheaper path will exist.

\textbf{Case P2} occurs when a trace is allowed to continue after the angular progression with respect to the trace's source node (\textbf{source progression}) has decreased by more than $180^\circ$ in a highly non-convex obstacle.
When the trace resumes source progression, target progression may not have resumed (see Fig. \ref{fig:tgtprog2}).
The case is resolved by queuing a cast from the source node to the \textit{only} target node of the trace.
The target node is a phantom point where the source progression is at a maximum, and the node is converted to an unreachable $\mnun$ node for the cast.
The phantom point is the only target node when Case P2 occurs as (i) it is the only target node when the source progression begins to decrease \citep[Case 1.4 of Theorem 2]{bib:r2}, and (ii) no other target nodes are placed by \rtwop{} when the source progression is decreasing.
By casting to the point where the source progression is at the maximum, Case P2 can be avoided for all subsequent traces.

\subsection{Interrupt Rule} \label{sec:interrupt}
The interrupt rule interrupts traces for queuing, so as to avoid expanding long, non-convex contours that are unlikely to find the shortest path.
A trace that calls a recursive ang-sec or oc-sec trace is not interrupted by the interrupt rule; it is interrupted by the ang-sec or oc-sec rule respectively, instead.

In \rtwo{} a trace is interrupted and queued after several nodes are placed, and the check occurs within the placement rule.
To modularize and simplify the algorithm, \rtwop{} interrupts and queues the trace after several corners are traced. The check occurs before the placement rule, and is called the \textbf{interrupt rule}.

\subsection{Overlap Rule} \label{sec:overlap}
The overlap rule dictates how \rtwo{} and \rtwop{} behave when paths from different queries overlap.
As line-of-sight checks are delayed, overlapping paths cannot be immediately discarded.
Costs have to be verified, and expensive paths can only be discarded if they are guaranteed to remain expensive.
By discarding paths, the overlap rule helps to improve search time.

% As \rtwo{} delays line-of-sight checks, some visible turning points cannot be found immediately.
% Expensive queries cannot be discarded without more specific conditions, and the conditions are specified by the rule.
% The rule discards 
\subsubsection{\rtwo{}'s Overlap Rule}
The overlap rule in \rtwo{} is triggered if the condition in each of the following three cases is met.

\input{fig_overlap1}
In \textbf{Case O1} (see Fig. \ref{fig:overlap1}), a query passes through a corner that contains other paths,  the $S$-tree is shrunk, and for each overlapping path, a cast is queued on the earliest source link that has no verified cost.
When a query passes through a corner that anchors links from other paths, an overlap is identified.
The purpose of the rule is to verify cost-to-come and discard expensive paths. 
As such, every link that is anchored at an $S$-tree $\mnvu$ or $S$-tree $\mneu$ node at the corner are identified.
For each link, the algorithm moves down the $S$-tree along each path, until the first link that has a source $S$-tree $\mnvy$ or $S$-tree $\mney$ node is found.
The algorithm subsequently moves up the tree, removing queries to avoid data races and re-anchoring the target links to $T$-tree $\mnvu$ nodes.
Finally, a cast is queued at the first link to verify cost for the target links.

% A tracing query is interrupted and shifted closer to the start node if a turning point is placed at a corner containing turning points from other queries. 
% The query is shifted to the first node in the source direction with cumulative visibility to the start node (an $\mnvy$-node). 
% Other queries in the target direction of the shifted query are discarded to avoid data races.
% \label{enum:overlap1}

\input{fig_overlap2}
In \textbf{Case O2}, a successful cast finds an expensive cost-to-come path at the target node's corner,
causing the target node to be replaced by an $S$-tree $\mney$ node.
A subsequent cast from a $\mside$-side $S$-tree $\mney$ node that collides will generate only an $\mside$-side trace (see Fig. \ref{fig:overlap2}). 
Once the trace is able to cast to a target node, Case O1 will be called on its path, and a cast is queued on the most recent link with a parent $\mney$ node.

A subsequent, successful cast from a $S$-tree $\mney$ node will cause the target node of the cast to be replaced by an $S$-tree $\mney$ node regardless of the cost.
If the replacement results in a pair of consecutive $\mney$ nodes with different sides, the path over the nodes will be discarded.

A trace with a $\mney$ or $\mneu$ source node will place  $\mneu$ turning points instead of $\mnvu$ turning points.
Source recursive traces cannot be called in such a trace -- the occupied-sector rule will not be triggered for a trace with the same side as the source node, and a query that follows a $(-\mside)$-sided recursive ang-sec trace will never be able to prune the $\mney$ source node.
If the trace is able to prune all $S$-tree $\mney$ nodes in the source direction, the trace resumes normal behavior.

\input{fig_overlap3}
In \textbf{Case O3} (see Fig. \ref{fig:overlap3}), a successful cast finds the cheapest cost-to-come path at the target node's corner.
Paths that pass through a $S$-tree $\mnvy$ node at the corner will have larger costs-to-come than the current path. 
As such, Case O2 is called for every expensive path at the corner, and for every $S$-tree $\mnvy$ node in the target direction along each path.
An $S$-tree $\mnvy$ node along an expensive path will be replaced by an $S$-tree $\mney$ node, and the path will be discarded if it passes through a pair of consecutive $\mney$ nodes with different sides;
an $S$-tree $\mnvu$ node along an expensive path will trigger Case O1 and queue a cast from the most recent $S$-tree $\mney$ node.

% During a casting query, if an expensive, unobstructed path from the start node is found at a turning point, the turning point and subsequent nodes will be marked as expensive, reducing the number of subsequent searches. 
% The current path is expensive if there is a cheaper unobstructed path to the turning point's corner. 
% The turning point in the current query will be converted to an expensive node ($\mney$-node), and 
% % An expensive node cannot be discarded to ensure that \rtwo is complete, as some turning points with line-of-sight are not found due to delayed line-of-sight checks.
% any subsequent collided casts will generate only one tracing query.
% A generated tracing query will have the same side as the $\mney$-node, as the paths around an opposite-sided trace will always pass through the $\mney$-node and can never find the shortest path. 
% To ensure completeness, the same-sided query will always be shifted to the most recent $\mney$-node upon casting.
% \rtwo resumes generation on both sides provided that all $\mney$-nodes are pruned in the same-sided tracing queries.

%  When the cheapest, unobstructed path from the start node is found at a turning point by a casting query, any other more expensive, unobstructed path that passes through the corner will be modified or discarded.
% The corresponding turning points of the expensive paths will be converted to $\mney$-nodes, and nodes in the target direction are converted or discarded based on the condition \ref{enum:overlap2}.

\subsubsection{\rtwop{}'s Overlap Rule}
\rtwop{} extends \rtwo{}'s overlap rule with four additional cases to discard searches, and is similar to the G-value pruning from \citep{bib:dps}
\input{fig_overlap4}
\textbf{Case O4} extends Case O2 for cost-to-go.
If the target node of a successful cast is a $T$-tree $\mney$ or $\mnvy$ typed, the source node's cost-to-go is examined.
If the cost-to-go is larger than the minimum so far, the source node will be replaced by a $T$-tree $\mney$ node.
% A subsequent, successful cast with a target $\mney$ node will cause the cast's source node to be replaced by a $\mney$ node regardless of the cost. 
% If both nodes have different sides, the path over the nodes will be discarded.
Unlike Case O2, Case O4 does not restrict traces, and Case O1 will not be called after a trace.

\input{fig_overlap5}
\textbf{Case O5} extends Case O3 for cost-to-go. 
If a successful cast results in the cheapest cost-to-go to the cast's source node, other paths that pass through the target corner's $T$-tree $\mnvy$ node will be modified based on Case O4.
Unlike Case O3, Case O1 will not be called by Case O5.

% In Case 5, a successful cast finds the cheapest cost-to-go path at the source node.
% Paths with a more expensive cost-to-go at the source node are identified, and the paths would have to have a $T$-tree node with cumulative visibility at the source node's corner.
% The case could therefore be treated as Case 3 occurring at each of the expensive path's corresponding $T$-tree node. 
% Similar nodes lying the source direction of the initial $T$-tree node are recursively marked as expensive.

\input{fig_overlap6and7}
\textbf{Cases O6 and O7} discard more expensive paths at a corner if the local path segment lies closer to the obstacle than the cheapest path (see Fig. \ref{fig:overlap6and7}). Case O6 examines cost-to-come, while Case O7 examines cost-to-go.
% In Case 6, suppose two unobstructed paths with different cost-to-come passes through a convex corner at $\mx_a$.
% In Case 7, the two unobstructed paths have different cost-to-go.
% For each case, if the more expensive path lies between the cheaper path and the closer obstacle edge adjacent to $\mx_a$, the expensive path can be discarded.
For both cases, suppose a turning point $\mx_a$ has line-of-sight to turning points at $\mx_e$ and $\mx_c$, and the paths $(\mx_0, \cdots, \mx_e, \mx_a)$ and $(\mx_0, \cdots, \mx_c, \mx_a)$ are unobstructed. 
$\mx_0$ is the position of the root node, and all points on the path belong to one tree.
Let the path passing through $\mx_e$ be more expensive, and the path passing through $\mx_c$ be cheaper.
The condition for discarding the more expensive path is
\begin{equation}
    \mtdir\mside(\mv_e \times \mv_c) < 0, \label{eq:newex}
\end{equation}
where $\mv_e = \mx_a - \mx_e$, $\mv_c = \mx_a - \mx_e$.

\begin{theorem}
Let a turning point at $\mx_0$ be an $\mtdir$-tree, $\mside$-sided, $\mnvy$-node. 
The turning point has the cheapest cost-to-come or cost-to-go if it is a $(\mtdir=S)$-tree or $(\mtdir=T)$-tree node respectively.
Suppose the shortest known, unobstructed path to $\mx_a$ from the root node of the $\mtdir$-tree passes through $\mx_c$ immediately before reaching $\mx_a$; 
and a longer unobstructed path from the root node passes through $\mx_e$ immediately before reaching $\mx_a$. 
Let $\mv_e = \mx_a - \mx_e$ and $\mv_c = \mx_a - \mx_c$.
If Eq. \ref{eq:newex} is satisfied, the longer path can be discarded and \rtwop{} remains complete.
\label{thm:newex}
\end{theorem}
\begin{proof}
% For a $\mside$-sided turning point to be placed, a $\mside$-sided trace has to occur. 
% Preceding the trace is a $-\mside$-sided trace or a collided ray, which allow queries to the $-\mside$-side of the node to be ignored by the proof, as \rtwop would still be complete.
% The proof would thus focus
% The body of the proof is to show that an expensive path that through $\mx_a$ can be discarded.
In this proof, some terms are shortened for brevity. \textit{Cost} refers to the cost-to-come or cost-to-go, and \textit{path} refers to the unobstructed path from the root node of the $\mtdir$-tree.
Let $c_{j\mid k}$ be the cost of a path $(\mx_0, \cdots, \mx_k, \mx_j)$ at $\mx_j$ that passes through $\mx_k$ before reaching $\mx_j$, and there is line-of-sight between $\mx_j$ and $\mx_k$.
$\mx_0$ is the position of the root node of the $\mtdir$-tree.
If the nodes are in the $S$-tree, $\mx_0$ is the start node and the cost is cost-to-come; if the nodes are in the $T$-tree, $\mx_0$ is the goal node and the cost is cost-to-go.
% A query that passes through $\mx_a$ with a longer path is an \textit{expensive query}, and any subsequent queries are termed as such. 
% For all cases examined by the proof, the expensive query is allowed to continue from $\mx_a$. As such, the path examined by the expensive query has to initially intersect the cheaper path to $\mx_a$.

% By forming triangles  using between points and proofs of contradiction, it can be shown that the expensive query that passes through $\mx_a$ will always be more expensive if it passes through the cheapest known path.

\input{fig_thm1case1}
In \textbf{Case 1.1}, the expensive query that passes through $\mx_e$ continues past $\mx_a$, causing the node at $\mx_a$ to be pruned, and the resulting path to intersect the cheaper path segment $(\mx_c, \mx_a)$. 
From a proof of contradiction, the resulting path is shown to be expensive if it intersects the cheaper path segment $(\mx_c, \mx_a)$.
Let the point of intersection be $\mx_i$.
The segment $(\mx_e, \mx_i)$ is assumed to be unobstructed, as this is the shortest possible distance from $\mx_e$ to $\mx_i$ on $(\mx_c, \mx_a)$.
If $c_{i \mid c} \ge c_{i \mid e}$, then $c_{a \mid i} + c_{i \mid c} \ge c_{a \mid i} + c_{i \mid e}$, which is a contradiction as $c_{a \mid c} < c_{a \mid e}$.
As such, $c_{i \mid c} < c_{i \mid e}$, and a longer, unobstructed path found at $\mx_a$ that intersects the cheaper path segment $(\mx_c, \mx_a)$ has to be expensive.
Case 1.1 is illustrated in Fig. \ref{fig:thm1case1a}.
% A more expensive query will always remain expensive if it passes through $\mx_a$ with a larger cost, and if it finds a path that intersects the cheaper path segment $(\mx_c, \mx_a)$.

In \textbf{Case 1.2}, the longer path intersects the shorter path at the other segments beyond $\mx_c$. 
From Case 1.1, $c_{c \mid c} < c_{c \mid e}$ and it is costlier to reach $\mx_c$ from $\mx_e$. 
By applying the proof of contradiction recursively over the segments beyond $\mx_c$, any path from $\mx_e$ can be shown to be longer when it arrives at the intersection with the shorter path.
Case 1.2 is illustrated in Fig. \ref{fig:thm1case1b}.

\input{fig_thm1case2}
Consider the cases where the node at $\mx_e$ is subsequently pruned, causing a node at $\mx_d$ to be exposed. 
% Let $\mx_f$ be the position where the prune occurs.
% $\mx_f$ lies at the opposite side of the cheaper path from $\mx_e$, and after a line colinear to ($\mx_d, \mx_e)$ is crossed.
For \textbf{Case 2.1}, let the intersection of the line colinear to $(\mx_d, \mx_e)$ with the cheaper path segment $(\mx_c, \mx_a)$ be at the point $\mx_j$; and the intersection of the path with the segment $(\mx_c, \mx_a)$ be at $\mx_i$.
From Case 1.1, $c_{j \mid c} < c_{j \mid e}$, and since $\mx_d$,  $\mx_e$, amd $\mx_j$ are colinear, $c_{j \mid c} < c_{j \mid d}$.
By applying a proof of contradiction, $c_{i \mid c} < c_{i \mid d}$, and
any subsequent path from $\mx_d$ that crosses $(\mx_c, \mx_a)$ will be longer.
Case 2.1 is illustrated in Fig. \ref{fig:thm1case2a}.

Consider \textbf{Case 2.2}, where the longer path from $\mx_d$ intersects the shorter path beyond $\mx_c$. By applying proofs of contradictions from Cases 1.1, 1.2 and 2.1, any subsequent path from $\mx_d$ that crosses the shorter path will be longer.
Case 2.2 is illustrated in Fig. \ref{fig:thm1case2b}.

Consider \textbf{Case 3}, where more nodes are pruned from the longer path.
Repeating the proofs of Cases 2.1 and 2.2, any path originating from the pruned path will be longer at the intersection with the shorter path, provided that pruning stops at a node before the root node.
Case 3 is applicable for $S$-tree nodes as the pruning of the longer path will stop at a $(-\mside)$-sided node.
Case 3 is not applicable for $T$-tree nodes, but is admissible to discard a more expensive cost-to-go path at $\mx_a$ as \rtwop{} is complete.

\input{fig_thm1case3A}
For Case 3, pruning will stop at a $(-\mside)$-sided turning point if the path is on the $S$-tree, where the root node is the start node. 
$\mnode_{-\mside}$ is first shown to exist.
% For pruning to occur, a trace examining the longer path has to be $\mside$-sided, so that the expensive node at $\mx_a$, and turning points at $\mx_d$ and $\mx_e$ can be pruned.
From a proof of contradiction, suppose that a $(-\mside)$-sided node does not exist and all turning points along the longer path are $\mside$-sided.
If all turning points are $\mside$-sided, the unobstructed longer path has to be a straight path, or bend monotonically to the $\mside$-side from the start node before reaching $\mx_a$ (see Fig. \ref{fig:thm1case3A}).
Since $(\mx_c, \mx_a)$ lies on the $\mside$-side of the longer path, the shorter path has to lie on the $\mside$ of the longer path when viewed from the start node.
However, it is impossible for a shorter unobstructed path to $\mx_a$ to exist on the $\mside$-side of a longer path that is straight or bends to the $(-\mside)$-side, and the longer path has to contain at least one $(-\mside)$-sided turning point. Let this $(-\mside)$-sided turning point be $\mnode_{-\mside}$.

\input{fig_thm1case3B}
$\mnode_{-\mside}$ on the longer path cannot be pruned if it is in the $S$-tree.
For $\mnode_{-\mside}$ to be pruned, a trace has to be $(-\mside)$-sided.
Before the prune can occur, 
the $(-\mside)$-sided trace will have to cross the $(-\mside)$-sided sector-ray of $\mnode_{-\mside}$, which points to a previously pruned $(-\mside)$-sided turning point along the longer path (e.g. $\mx_d$).
The sector-ray is formed when a cast from $\mnode_{-\mside}$ had reached the pruned point.
The trace may be discarded, or be interrupted by a recursive ang-sec trace, causing $\mnode_{-\mside}$ to be preserved (see Fig. \ref{fig:thm1case3Ba}).

For $T$-tree nodes in Case 3, a $(-\mside)$-sided turning point can be similarly shown to exist, but unlike the $S$-tree, the turning point can be pruned as sector-rays cannot be defined for nodes in the target direction.
Case 3 is not a problem for $T$-tree nodes, as \rtwop{} will be able to find the shortest path from the $(-\mside)$-sided turning point in another query even if it is pruned by the current query (see Fig. \ref{fig:thm1case3Bb}).

Consider \textbf{Case 4}, where a subsequent query reaches a point that causes the longer path to sweep past the root node and not intersect with the cheaper path. 
As such a path causes a loop, the longer path can be discarded if it does not fulfill Eq. \ref{eq:newex} at $\mx_a$.
\end{proof}

% The proof can be extended for paths with equal costs, provided that the paths are not allowed through checkerboard corners.
% A checkerboard corner occurs between four adjacent cells where the binary occupancy states resemble a checkerboard \cite{bib:r2}.
% A checkerboard corner lies at the same location as another, and are convex if paths can pass through them.
% A query from one corner can lead to the other, and an equal cost path that passes through one corner can be discarded by Eq. (\ref{eq:newex}) is kept at the other corner, and subsequently cause both paths to be discarded.

% The proof is generally applicable if the cost between both paths are equal, provided that paths are not allowed through checkerboard corners.
% A checkerboard corner occurs between four adjacent cells where the binary occupancy states resemble a checkerboard \cite{bib:r2}.
% If paths can pass through checkerboard corners, the corners are convex.
% As a checkerboard corner is at the same location as another, two equal cost paths from different directions can cause Eq. \ref{eq:newex} to be satisfied for a different path when evaluated at a different corner.
% As a subsequent query from one corner can lead to the other corner, a path that is discarded by the earlier corner 
% As \rtwop{} allow paths through checkerboard corners, the proof cannot be applied

%%%%%%%%%%%%%%%%%%%%%%%%%%%%%%%% ALGORITHM %%%%%%%%%%%%%%%%%%%%%%%%%%%%%%%%%
\section{Algorithm}
\input{alg_run}
\input{alg_caster}
\input{alg_tracer}

\input{alg_casterreached}
\input{alg_castercollided}
\input{alg_tracerproc}

The pseudocode in this section shows only the noteworthy steps in the algorithm.
A more detailed version is available in the supplementary material, which describes how the tree is managed to avoid data races and limit the number of link connections for each link.
In the pseudocode, ``source" and ``target'' are abbreviated to ``src" and ``tgt" respectively.
Rays are merged only if the resulting angular sector shrinks.

\rtwop{} is run from Alg. \ref{alg:run}.
Alg. \ref{alg:caster} handles casts, while Alg. \ref{alg:tracer} handles traces. 
Alg. \ref{alg:casterreached} and Alg. \ref{alg:castercollided} are helper functions that manages a successful cast and collided cast respectively, and Alg. \ref{alg:tracerproc} is a helper function that manages nodes and links in the source or target direction of the trace.

%% file: tab_nodetypes.tex
\begin{table}[!ht]
\centering
\caption{Node Types in \rtwop{}}
\label{tab:nodetypes}
\setlength{\tabcolsep}{3pt}
\begin{tabular}{ c  c  p{5.5cm}}
\textbf{Type} & \textbf{Sm.} & \textbf{Description} \\
\hline
$\mnvy$ & \input{sym_vy.tex} & A turning point with cumulative visibility.  \\
\hline
$\mnvu$ & \input{sym_vu.tex}  & A turning point with unknown cumulative visibility. \\
\hline
$\mney$ & \input{sym_ey.tex}  & An expensive $\mnvy$ node. \\
\hline
$\mneu$ & \input{sym_eu.tex} & An expensive $\mnvu$ node. \\
\hline
% $\mnph$* & \input{sym_ph.tex} & A phantom point, or temporary node placed when a trace is interrupted. \\
% \hline
$\mntm$* & \input{sym_tm.tex} & A phantom point, or a temporary node that is placed when a trace is interrupted. \\
\hline
$\mnun$* & \input{sym_un.tex}  & An unreachable node. Supersedes ad hoc points from \rtwo{}.\\
\hline
$\mnoc$* & \input{sym_oc.tex} & A node placed by a target recursive-angular sector trace. \\
\hline
\multicolumn{3}{p{7.5cm}}{
\footnotesize
The \textit{Sm.} column denotes the symbol used in figures.
*$\mntm$, $\mnun$, and $\mnoc$ nodes are $T$-tree nodes, which lie in the target direction of a query.
}
\end{tabular}
\end{table}

%% file: sym_vy.tex
\raisebox{0pt}{
    \tikz{
        \clip (-1.1mm, -1.1mm) rectangle ++(2.2mm, 2.2mm);
        \node [vy pt={}{}{}] at (0,0) {};
    }
}

%% file: sym_vu.tex
\raisebox{0pt}{
    \tikz{
        \clip (-1.1mm, -1.1mm) rectangle ++(2.2mm, 2.2mm);
        \node [vu pt={}{}{}] at (0,0) {};
    }
}

%% file: sym_ey.tex
\raisebox{0pt}{
    \tikz{
        \clip (-1.4mm, -1.4mm) rectangle ++(2.8mm, 2.8mm);
        \node [ey pt={}{}{}] at (0,0) {};
    }
}

%% file: sym_eu.tex
\raisebox{0pt}{
    \tikz{
        \clip (-1.4mm, -1.4mm) rectangle ++(2.8mm, 2.8mm);
        \node [eu pt={}{}{}] at (0,0) {};
    }
}

%% file: sym_tm.tex
\raisebox{0pt}{
    \tikz{
        \clip (-1.1mm, -1.1mm) rectangle ++(2.2mm, 2.2mm);
        \node [tm pt={}{}{}] at (0,0) {};
    }
}

%% file: sym_un.tex
\raisebox{0pt}{
    \tikz{
        \clip (-1.1mm, -1.1mm) rectangle ++(2.2mm, 2.2mm);
        \node [un pt={}{}{}] at (0,0) {};
    }
}

%% file: sym_oc.tex
\raisebox{0pt}{
    \tikz{
        \clip (-1.05mm, -1.05mm) rectangle ++(2.1mm, 2.1mm);
        \node [oc pt={}{}{}] at (0,0) {};
    }
}

%% file: tab_legend.tex
\begin{table}[!ht]
\centering
\caption{Legend of Symbols Used in Figures}
\label{tab:legend}
\setlength{\tabcolsep}{3pt}
\begin{tabular}{ c  p{5.5cm}}
\textbf{Sm.} & \textbf{Description} \\
\hline
\input{sym_link.tex} & A link anchored at node 1, connected to links (not shown) anchored at node 2.  \\
\hline
\input{sym_qlink.tex} & Same as above, and the link is associated with a queued query. \\
\hline
\input{sym_merge.tex} & Links anchored at nodes 1 and 2, and nodes 2 and 3, are connected. Links anchored at 2 and 3 are not connected. \\
\hline
\input{sym_separate.tex} & Multiple nodes at the same corner.\\
\hline
\input{sym_trace.tex} & A pair of disconnected $\mntm$ nodes (\textbf{trace-nodes}) that follows a trace and is not part of any tree. Links that are anchored on the nodes are called \textbf{trace-links}. \\
\hline
\input{sym_rayl.tex} & Left sector-ray of an angular-sector at node 1.\\
\hline
\input{sym_rayr.tex} & Right sector-ray of an angular-sector at node 1.\\
\hline
\input{sym_rayprog.tex} & Progression ray with respect to node 1.\\
\hline
\textcolor{swatch_stree}{$S$-tree} & $S$-tree objects are colored light red. \\
\hline
\textcolor{swatch_ttree}{$T$-tree} & $T$-tree objects are colored dark green. \\
% \hline
% \multicolumn{2}{p{7.5cm}}{
% \footnotesize
% The \textit{Sm.} column denotes the symbol used in figures.
% }
\end{tabular}
\end{table}

%% file: sym_link.tex
\raisebox{-.5ex}{
    \tikz[]{
        \clip (-1.6mm, -1.6mm) rectangle ++(13.2mm, 3.2mm);
        \node (n1) [any pt] at (0,0) {\footnotesize 1};
        \node (n2) [any pt] at (1cm,0) {\footnotesize  2};
        \draw [link] (n1) -- (n2); 
    }
}

%% file: sym_qlink.tex
\raisebox{-.5ex}{
    \tikz{
        \clip (-1.6mm, -1.6mm) rectangle ++(13.2mm, 3.2mm);
        \node (n1) [any pt] at (0,0) {\footnotesize 1};
        \node (n2) [any pt] at (1cm,0) {\footnotesize  2};
        \draw [qlink] (n1) -- (n2); 
    }
}

%% file: sym_merge.tex
\raisebox{-5.5mm}{
    \tikz[]{
        % \clip (-4.2mm, -2.5mm) rectangle ++(8.4mm, 5mm);
        % \node [merge, inner xsep=3.9mm, inner ysep=2.2mm] at (0,0) {};
        % \node (n1) [any pt, xshift=-1.6mm] at (0,0) {\footnotesize 1};
        % \node (n2) [any pt, xshift=1.6mm] at (0,0) {\footnotesize  2};
        \clip (-4.3mm, -4.5mm) rectangle ++(8.6mm, 9mm);
        % \node [separate, circle, minimum size=0, inner sep=3mm] at (0,0) {};
        \node (n1) [any pt, xshift=-2.5mm] at (0,0) {\footnotesize 1};
        \node (n2) [any pt, xshift=1.2mm, yshift=1.9mm] at (0,0) {\footnotesize  2};
        \node (n3) [any pt, xshift=1mm, yshift=-2mm] at (0,0) {\footnotesize  3};
        \draw [merge] (n1) -- (n2);
        \draw [merge] (n1) -- (n3);
    }
}

%% file: sym_separate.tex
\raisebox{-5.5mm}{
    \tikz[]{
        \clip (-4.3mm, -4.5mm) rectangle ++(8.6mm, 9mm);
        \node [separate, inner xsep=3.9mm, inner ysep=4.2mm] at (0,0) {};
        \node (n1) [any pt, xshift=-2mm] at (0,0) {\footnotesize 1};
        \node (n2) [any pt, xshift=1mm, yshift=1.8mm] at (0,0) {\footnotesize  2};
        \node (n3) [any pt, xshift=0.9mm, yshift=-1.9mm] at (0,0) {\footnotesize  3};
    }
}

%% file: sym_trace.tex
\raisebox{-3mm}{
    \tikz[]{
        % \clip (-3mm, -1.6mm) rectangle ++(6mm, 3.2mm);
        \pic at (0,0) {trace grp={}{}{}{}};
        % \node [separate] at (0,0) {};
        % \node [trtm pt={}{}{}, xshift=-\uss] at (0,0) {};
        % \node [trtm pt={}{}{}, xshift=\uss] at (0,0) {};
    }
}

%% file: sym_rayl.tex
\raisebox{-.5ex}{
    \tikz{
        \clip (-1.6mm, -1.6mm) rectangle ++(11.6mm, 3.2mm);
        \node (n1) [any pt] at (0,0) {\footnotesize 1};
        \draw [rayl] (n1) -- ++(0:1cm); 
    }
}

%% file: sym_rayr.tex
\raisebox{-.5ex}{
    \tikz{
        \clip (-1.6mm, -1.6mm) rectangle ++(11.6mm, 3.2mm);
        \node (n1) [any pt] at (0,0) {\footnotesize 1};
        \draw [rayr] (n1) -- ++(0:1cm); 
    }
}

%% file: sym_rayprog.tex
\raisebox{-.5ex}{
    \tikz{
        \clip (-1.6mm, -1.6mm) rectangle ++(11.6mm, 3.2mm);
        \node (n1) [any pt] at (0,0) {\footnotesize 1};
        \draw [rayprog] (n1) -- ++(0:1cm); 
    }
}

%% file: fig_tree.tex
\tikzset{
        pics/tree/.style n args={0}{ code={ 
            \fill [swatch_obs] 
                (2*\ul,7*\u) -- ++(90:3*\ul) coordinate (x0) -- ++(0:13*\u)
                -- ++(270:\ul) coordinate (x5) -- ++(0:\ul) -- ++(270:\ul) -- ++(0:\ul) 
                -- ++(270:6*\u) coordinate (x1) -- ++(180:3*\u) coordinate (x2) -- ++(90:4*\u)
                -- ++(180:\ul) -- ++(90:\ul) -- ++(180:\ul) -- ++(90:\ul) coordinate (x3)
                -- ++(180:4*\ul)-- ++(270:\ul) -- ++(0:\ul) -- ++(270:\ul);
            \fill [swatch_obs] 
                (4*\ul, \u) coordinate (x20) -- ++(90:3*\u) coordinate (x21) -- ++(0:5*\u) coordinate (x22) -- ++(270:3*\u) coordinate (x23);
            % \fill [swatch_obs] (0, 0.5) coordinate (x0) -- ++(0, 2) coordinate (x1) -- ++(3.25, 0) coordinate (x2) -- ++(0, -0.5) coordinate (x3) -- ++(-1.75, 0) coordinate (x4) -- ++(0, -1.5) coordinate (x5) -- cycle;
            % \node (Nstart) [svy pt={below=1.5mm, right=0.5mm}{$\mx_s$}] at (2.5, 0) {};
            % \node (Ngoal) [blue pt={above=0mm, left=0.25mm}{$\mx_t$}] at (3, 3) {};
            % \path (Nstart.center) -- (Ngoal.center) node (xcol) [coordinate, pos=2/3] {};
            % \path (x4) -- (x5) node (x4-5) [coordinate, pos=1/6] {};
            % \path (x4) -- (x5) node (x4-5b) [coordinate, pos=1/2]{};
            % \path (x5) -- (x0) node (x5-0) [coordinate, pos=1/6] {};
            % \path (x0) -- (x1) node (x0-1) [coordinate, pos=1/8] {};
            % \path (x0) -- (x1) node (x0-1b) [coordinate, pos=3/8] {};
            % \path (x1) -- (x2) node (x1-2) [coordinate, pos=1/13] {};
            % \path (x1) -- (x2) node (x1-2b) [coordinate, pos=3/13] {};
            \node (nstart) [svy pt={xshift=6mm, yshift=-1mm}{$\mxstart$}] at (0, 2*\ul) {};
            \node (ngoal) [tvy pt={}{$\mxgoal$}] at (12*\ul, 3*\ul) {};
            \path (nstart) -- (ngoal);
            \draw [tlink] (ngoal) -- ++(0:\ul); % goal link
            \draw [slink] (nstart) -- ++(240:\ul); % start link
            
            % \draw [link={red}] (nstart) -- (ngoal); 
        }},
    }

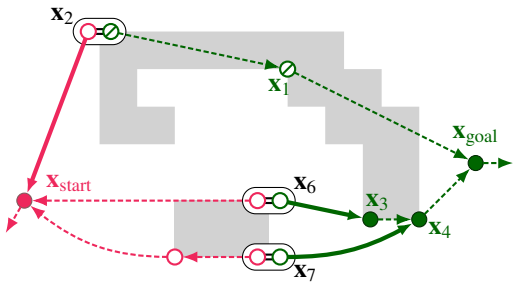
\begin{figure}[!ht]
\centering
\begin{tikzpicture}[]
    \clip (-0.25cm, -0.25cm) rectangle ++(6.75cm, 4cm);
    \pic at (0, 0) {tree};
    \pic at (x0) {merge grp={left}{135:$\mx_2$}{n0 svu}{svu pt}{n0 ttm}{ttm pt}};
    \node (n3 tph) [ttm pt={below, xshift=-1mm, yshift=-1mm}{$\mx_1$}] at (x3) {};
    \node (n1 tvy) [tvy pt={below right}{$\mx_4$}] at (x1) {};
    \node (n2 tvy) [tvy pt={above, xshift=1mm, yshift=-1mm}{$\mx_3$}, xshift=\uss] at (x2) {};
    \node [merge={right, xshift=2mm, yshift=0mm}{$\mx_6$}] at (x22) {};
    \pic at (x22) {merge grp={right}{45:$\mx_6$}{n22 svu}{svu pt}{n22 tvu}{tvu pt}};
    \pic at (x23) {merge grp={right}{-45:$\mx_7$}{n23 svu}{svu pt}{n23 tvu}{tvu pt}};
    \node (n20 svu) [svu pt={}{}] at (x20) {};

    \draw [tlink] (n2 tvy) edge (n1 tvy) (n1 tvy) edge (ngoal);
    \draw [tqlink] (n23 tvu) edge[bend right=15] (n1 tvy);
    \draw [slink] (n23 svu) edge (n20 svu) (n20 svu) edge[bend left=20] (nstart);
    \draw [tqlink] (n22 tvu) edge (n2 tvy);
    \draw [slink] (n22 svu) edge (nstart);
    
    \draw [sqlink] (n0 svu) -- (nstart);
    \draw [tlink] (n0 ttm) edge (n3 tph) (n3 tph) edge (ngoal);

    % \fill [swatch_obs] (0.25, 2)  coordinate (x3) -- ++(0, 0.5) -- ++(3.5, 0) -- ++(0, -3) -- ++(-0.5, 0) coordinate (x1) -- ++(0, 2.5)  coordinate(x2) -- cycle;
    % \fill [swatch_obs] (0.75, 0)  -- ++(0, 0.5) coordinate (xs) -- ++(0.5, 0) -- ++(0, -0.5) -- cycle;
    % \fill [swatch_obs] (0, 1) -- ++(0, 0.5) -- ++(2.75, 0) coordinate (x4) -- ++(0, -0.5) coordinate (x5) -- cycle;

    % \path (x1) -- ++(90:1) node (Ncol1) [cross pt] {};
    % \draw [trace] (Ncol1.center) -- (x2) -- (x3) -- ++(90:0.25);
    
    % \node (Ns) [blue pt={left}{$\mx_s$}] at (xs) {};
    % \node (Nss) [blue pt={above=1.5mm, left=-2mm}{$\mx_ss$}] at (0.5, 0) {};
    % \node (Nt) [coordinate] at (3.5, 2.75) {};
    % \path (Ns.center) -- (Nt.center) 
    %     node (Ncol) [cross pt={below right, yshift=-1mm}{$\Mcol$}, pos=2/9] {}
    %     node [cross pt, pos=6/9] {};
    % \draw [trace2] (Ncol.center) -- (x5) -- (x4) -- ++(180:0.25);

    % \path (Nss.center) -- (Ns.center) node (xss-s) [coordinate, pos=6] {};
    % \draw [dotted] (Nss) -- (Ns) -- (xss-s);
    % \draw [-Latex] (Ns) -- (Nt) node [pos=0.45, sloped, above] {$\Mray_L$};
    % \node (N3) [black pt={below left}{$\Mx$}] at (x3) {};

\end{tikzpicture}
\caption{
A brief illustration of \rtwop{}'s trees, nodes and links.
The $S$-tree (light red) is rooted at the start node at $\mxstart$. 
The $T$-tree (dark green) is rooted at the goal node at $\mxgoal$. 
Both trees are connected at their leaf nodes.
A link that connects to nothing is anchored at the start node and another at the goal node.
The corners are labelled according to the order they are found, not including unlabelled corners.
An interrupted trace is queued at $\mx_2$, and a cast is each queued from $\mx_6$ and $\mx_7$.
}
\label{fig:tree}
\end{figure}

%% file: fig_tgtocsec.tex
\tikzset{
    pics/tgtocsec/.style n args={0}{ code={

    }},
}
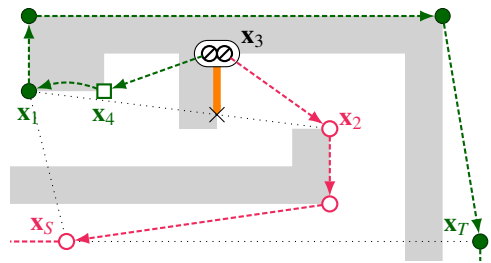
\begin{figure}[!ht]
\centering
\begin{tikzpicture}[]
    \clip (-\u, -\u) rectangle ++(13*\ul, 7*\ul);
    \path (\ul, 0) coordinate (xsrc);
    \path (12*\ul, 0) coordinate (xtgt);

    % first collision obs
    \fill [swatch_obs] 
        (0*\ul, 4*\ul) coordinate (xa1) -- ++(90:2*\ul) coordinate (xa2) -- ++(0:11*\ul) coordinate (xa3) 
        -- ++(90:-8*\ul) coordinate (xa4) -- ++(0:-\ul) coordinate (xa5) 
        -- ++(90:7*\ul) coordinate (xa6) -- ++(0:-8*\ul) coordinate (xa7) -- ++(90:-\ul) coordinate (xa8); 

    % second collision obs
    \fill [swatch_obs] 
        (-\ul, \ul) coordinate (xb1) -- ++(90:\ul) coordinate (xb2) 
        -- ++(0:8*\ul) -- ++(90:\ul) -- ++(0:\ul) coordinate (xb3) -- ++(90:-2*\ul) coordinate (xb4);

    \draw [dotted] (xsrc) -- (xtgt);
    \draw [dotted] (xsrc) -- (xa1);
    \node (ntgt) [tvy pt={left, yshift=1mm, xshift=0}{$\mx_T$}] at (xtgt) {};
    \node (nsrc) [svu pt={left, yshift=1mm, xshift=0mm}{$\mx_S$}] at (xsrc) {};
    \draw [sxlink] (nsrc) -- ++(0:-2*\ul);
    \draw [txlink] (ntgt) -- ++(90:-\ul);

    % third obstacle
    \fill [swatch_obs] (4*\ul, 3*\ul) coordinate (xc1) -- ++(90:2*\ul) coordinate (xc2) -- ++(0:\ul) coordinate (xc3) -- ++(90:-2*\ul) coordinate (xc4);

    % collided cast path and trace
    \draw [dotted] (xb3) -- (xa1)
        node (xcol) [coordinate, pos=3/8] {};
    \draw [trace] (xcol) -- (xc3) -- ++(0:-\u);
    \node [cross pt] at (xcol) {};

    % obs a nodes.
    \node (na8 toc) [toc pt={below, yshift=-2mm}{$\mx_4$}] at (xa8) {};
    \node (na1 tvy) [tvy pt={below, yshift=-2mm}{$\mx_1$}] at (xa1) {};
    \node (na2 tvy) [tvy pt] at (xa2) {};
    \node (na3 tvy) [tvy pt] at (xa3) {};

    % obs b nodes.
    \node (nb4 svu) [svu pt] at (xb4) {};
    \node (nb3 svu) [svu pt={right}{$\mx_2$}] at (xb3) {};

    % obs c nodes.
    \pic at (xc3) {trace grp={right}{45:$\mx_3$}{nc3 ttrtm}{nc3 strtm}};

    % links
    \draw [tlink] (nc3 ttrtm) edge (na8 toc) (na8 toc) edge[bend right=20] (na1 tvy) (na1 tvy) edge (na2 tvy) (na2 tvy) edge (na3 tvy) (na3 tvy) edge (ntgt);
    \draw [slink] (nc3 strtm) edge (nb3 svu) (nb3 svu) edge (nb4 svu) (nb4 svu) edge (nsrc);
\end{tikzpicture}
\caption{
After a cast from $\mx_2$ and $\mx_1$ collides, an $R$-trace occurs and calls a target oc-sec trace at $\mx_3$. The oc-sec trace begins from $\mx_1$, and stops once an $\mnoc$ node is placed at $\mx_4$. 
By stopping at the first corner and preventing subsequent oc-sec traces from occurring from an $\mnoc$ node, chases are prevented from happening.
}
\label{fig:tgtocsec}
\end{figure}

%% file: fig_tgtprog1.tex
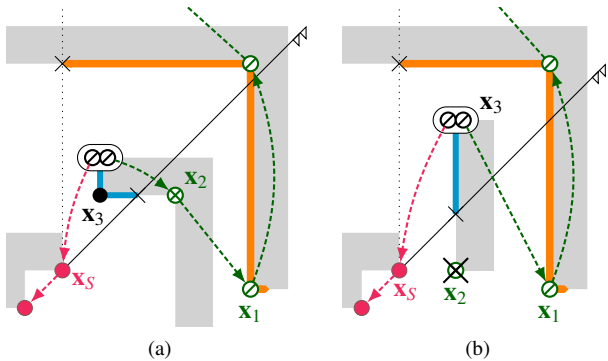
\begin{figure}[!ht]
\centering
\subfloat[\label{fig:tgtprog1a} ] {%
    \centering
    \begin{tikzpicture}[]
        \clip (-\u, -\u) rectangle ++(8*\ul, 17*\u);

        % obs a, containing src nodes
        \fill [swatch_obs] (-\ul, 0) -- ++(90:2*\ul) -- ++(0:2*\ul) -- ++(90:-\ul) coordinate (xsrc) -- ++(0:-\ul) -- ++(90:-\ul) coordinate (xssrc);

        % obs b, containing first trace
        \fill [swatch_obs] (-\ul, 6*\ul+\u) -- ++(90:\ul) -- ++(0:8*\ul) -- ++(90:-7*\ul) -- ++(0:-\ul) coordinate (xb ttm) -- ++(90:6*\ul) coordinate (xb tph);

        % obs c, containing rcr trace
        \fill [swatch_obs] (2*\ul, 3*\ul) coordinate (xc1) -- ++(90:\ul) coordinate (xc trace)  -- ++(0:3*\ul) -- ++(90:-5*\ul) -- ++(0:-\ul) -- ++(90:4*\ul) coordinate (xc tun);

        % define tgt wrt start
        \path (xsrc) -- ++(90:10*\ul) coordinate (xtgt);
        
        % collision coordinates and paths
        \draw [dotted] (xsrc) -- +(90:8*\ul)
            node (xb col) [coordinate] at +(90:11*\u) {};
        \path (xssrc) -- (xsrc) 
            node (xc col) [coordinate, pos=3] {}
            node (xray) [coordinate, pos=7.5] {};

        % traces
        \draw [trace] (xb col) -- (xb tph) -- (xb ttm) -- ++(0:\u);
        \draw [trace2] (xc col) -- (xc1) -- (xc trace) -- ++(0:\u);

        % obs a nodes and ray
        \draw [rayr] (xsrc) -- (xray);
        \node (nsrc) [svy pt={below right}{$\mx_S$}] at (xsrc) {};
        \node (nssrc) [svy pt={}{}] at (xssrc) {};

        % obs b nodes
        \node [cross pt] at (xb col) {};
        \node (nb tph) [ttm pt={}{}] at (xb tph) {};
        \node (nb ttm) [ttm pt={below, yshift=-2mm}{$\mx_1$}] at (xb ttm) {};

        % obs c nodes
        \node [cross pt] at (xc col) {};
        \pic at (xc trace) {trace grp={}{}{nc strtm}{nc ttrtm}};
        \node [black pt={below}{-150:$\mx_3$}] at (xc1) {};
        \node (nc tun) [tun pt={right, yshift=1mm}{$\mx_2$}] at (xc tun) {};

        % links
        \draw [slink] (nc strtm) edge[bend right=10] (nsrc) (nsrc) edge (nssrc);
        \draw [tlink] (nc ttrtm) edge[bend left=10] (nc tun) (nc tun) edge (nb ttm) (nb ttm) edge[bend right=20] (nb tph) (nb tph) edge (xtgt);
    \end{tikzpicture}
} \hfill
\subfloat[\label{fig:tgtprog1b}] {%
    \centering
    \begin{tikzpicture}[]
        \clip (-\u, -\u) rectangle ++(7*\ul, 17*\u);
        
        % obs a, containing src nodes
        \fill [swatch_obs] (-\ul, 0) -- ++(90:2*\ul) -- ++(0:2*\ul) -- ++(90:-\ul) coordinate (xsrc) -- ++(0:-\ul) -- ++(90:-\ul) coordinate (xssrc);

        % obs b, containing first trace
        \fill [swatch_obs] (-\ul, 6*\ul+\u) -- ++(90:\ul) -- ++(0:7*\ul) -- ++(90:-7*\ul) -- ++(0:-\ul) coordinate (xb ttm) -- ++(90:6*\ul) coordinate (xb tph);

        % obs c, containing rcr trace
        \fill [swatch_obs] (5*\u, \ul) coordinate (xc tun) -- ++(90:4*\ul) coordinate (xc trace)  -- ++(0:\ul) -- ++(90:-4*\ul);
        
        % define tgt wrt start
        \path (xsrc) -- ++(90:10*\ul) coordinate (xtgt);
        
        % collision coordinates and paths
        \draw [dotted] (xsrc) -- +(90:8*\ul)
            node (xb col) [coordinate] at +(90:11*\u) {};
        \path (xssrc) -- (xsrc) 
            node (xc col) [coordinate, pos=2.5] {}
            node (xray) [coordinate, pos=6.5] {};

        % traces
        \draw [trace] (xb col) -- (xb tph) -- (xb ttm) -- ++(0:\u);
        \draw [trace2] (xc col) -- (xc trace) -- ++(0:\u);

        % obs a nodes and ray
        \draw [rayr] (xsrc) -- (xray);
        \node (nsrc) [svy pt={below, xshift=1mm, yshift=-1.5mm}{$\mx_S$}] at (xsrc) {};
        \node (nssrc) [svy pt={}{}] at (xssrc) {};

        % obs b nodes
        \node [cross pt] at (xb col) {};
        \node (nb tph) [ttm pt={}{}] at (xb tph) {};
        \node (nb ttm) [ttm pt={below, yshift=-2mm}{$\mx_1$}] at (xb ttm) {};

        % obs c nodes
        \node [cross pt] at (xc col) {};
        \pic at (xc trace) {trace grp={right}{45:$\mx_3$}{nc strtm}{nc ttrtm}};
        \node (nc tun) [tun pt={below, yshift=-2mm}{$\mx_2$}] at (xc tun) {};
        \node [prune] at (xc tun) {};

        % links
        \draw [slink] (nc strtm) edge[bend right=10] (nsrc) (nsrc) edge (nssrc);
        \draw [tlink] (nc ttrtm) edge (nb ttm) (nb ttm) edge[bend right=20] (nb tph) (nb tph) edge (xtgt);
        
    \end{tikzpicture}
}    
\caption{
A recursive angular-sector trace places a $\mnun$ node to ensure target progression.
(a) An $R$-trace that reached $\mx_1$ has triggered an $L$-sided recursive angular-sector trace.
The initial edge of the recursive trace lies between $\mx_3$ and the collision point of the sector ray.
There is no target progression for the initial edge when viewed from $\mx_1$ (traces to the left of $\mx_1$), but placing an $L$-sided $\mnun$ node at $\mx_2$ will result in target progression (traces to right when viewed from $\mx_2$).
(b) If the initial edge has target progression, the $\mnun$ node at $\mx_2$ will be pruned immediately by the recursive trace.
}
\label{fig:tgtprog1}
\end{figure}

%% file: fig_tgtprog2.tex
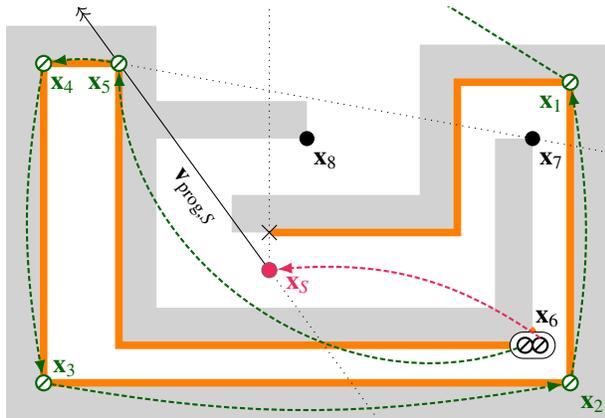
\begin{figure}[!ht]
\centering
    \begin{tikzpicture}[]
        \clip (0, 0) rectangle ++(33*\u, 22*\u);
        
        % obs a, containing src nodes
        \fill [swatch_obs] 
            (0,0) -- ++(90:10*\ul+\u) -- ++(0:4*\ul) -- ++(90:-2*\ul) -- ++(0:4*\ul) 
            -- ++(90:-\ul) coordinate (x11) -- ++(0:-4*\ul) -- ++(90:-4*\ul-\u) 
            -- ++(0:9*\ul) -- ++(90:4*\ul+\u) -- ++(0:\ul) coordinate (x10)
            -- ++(90:-5*\ul-\u) coordinate (x9) -- ++(0:-11*\ul) coordinate (x8) -- ++(90:7*\ul+\u) coordinate (x7)
            -- ++(0:-2*\ul) coordinate (x6) -- ++(90:-8*\ul-\u) coordinate (x5) -- ++(0:14*\ul) coordinate (x4)
            -- ++(90:8*\ul) coordinate (x3) -- ++(0:-3*\ul)  coordinate (x2) -- ++(90:-4*\ul) coordinate (x1)
            -- ++(0:-6*\ul) -- ++(90:\ul) -- ++(0:5*\ul) -- ++(90:4*\ul) -- ++(0:5*\ul) -- ++(90:-10*\ul);
        
        % define src and tgt and collisions
        \path (7*\ul, 4*\ul) coordinate (xsrc);
        \path (xsrc) -- ++(90:10*\ul) coordinate (xtgt);
        \path (xsrc) -- ++(90:\ul) coordinate (xcol);

        % draw trace
        \draw [trace] (xcol) -- (x1) -- (x2) -- (x3) -- (x4) -- (x5) -- (x6) -- (x7) -- (x8) -- (x9) -- ++(90:\u);

        % draw rays and lines
        \draw [rayprog] (xsrc) -- (x7) 
            node (xsrcto7) [coordinate, pos=14/11] {}
            node (x7tosrc) [coordinate, pos=-1] {}
            node [sloped, below, pos=0.4] {$\mv_{\mprog,S}$}
            (x7) -- (xsrcto7);
        \draw [dotted] (xsrc) -- (x7tosrc);
        \draw [dotted] (x7) -- (x10)
            node (x7to10) [coordinate, pos=27/22] {}
            (x10) -- (x7to10);
        \draw [dotted] (xsrc) -- ++(90: 7*\ul); % rayr

        % nodes and collisions
        \node (nsrc) [svy pt={shift={(4mm, -2mm)}}{center:$\mx_S$}] at (xsrc) {};
        \node [cross pt] at (xcol) {};
        \node (n3 tph) [ttm pt={below, xshift=-2.5mm, yshift=-1.5mm}{$\mx_1$}] at (x3) {};
        \node (n4 tph) [ttm pt={below, xshift=3mm, yshift=-1.5mm}{$\mx_2$}] at (x4) {};
        \node (n5 tph) [ttm pt={right, yshift=1mm}{$\mx_3$}] at (x5) {};
        \node (n6 tph) [ttm pt={below, xshift=2.5mm, yshift=-1.5mm}{$\mx_4$}] at (x6) {};
        \node (n7 tph) [ttm pt={below, xshift=-2.5mm, yshift=-1.5mm}{$\mx_5$}] at (x7) {};
        \pic at (x9) {trace grp={shift={(2mm, 3.5mm)}}{center:$\mx_6$}{n9 ttrtm}{n9 strtm}};
        \node (n10) [black pt={below, xshift=2.5mm, yshift=-1.5mm}{$\mx_7$}] at (x10) {};
        \node (n11) [black pt={below, xshift=2.5mm, yshift=-1.5mm}{$\mx_8$}] at (x11) {};

        % links
        \draw [slink] (n9 strtm.90) to[bend right=20] (nsrc);
        \draw [tlink] (n9 ttrtm) edge[bend left=55] (n7 tph)
            (n7 tph) edge[bend right=10] (n6 tph)
            (n6 tph) edge[bend right=10] (n5 tph)
            (n5 tph) edge[bend right=10] (n4 tph)
            (n4 tph) edge[bend right=10] (n3 tph)
            (n3 tph) edge (xtgt);

    \end{tikzpicture}
\caption{
    A cast occurs from the source node (at $\mx_S$) if the source progression has decreased by more than $180^\circ$ (at $\mx_6$). 
    The maximum source progression is indicated by the source progression ray $\mv_{\mprog,S}$.
    The phantom point at $\mx_5$ is replaced by an $\mnun$ node before a cast is queued from the source node to the $\mnun$ node.
    If no cast occurs, the target progression ray will reach a maximum at $\mx_7$. 
    As such, if a source recursive trace is called at $\mx_8$, there will be no target progression.
    % Links are omitted from the figure.
}
\label{fig:tgtprog2}
\end{figure}

%% file: fig_overlap1.tex
\tikzset{
    pics/overlap1/.style n args={0}{ code={ 
        \fill [swatch_obs] (2*\ul, 5*\ul) -- ++(90:1.5*\ul) -- ++(0:3.5*\ul) coordinate (x4) -- ++(-90:1.5*\ul) coordinate (x3);

        \path (0.5*\ul, 0.5*\ul) coordinate (x1);
        \path (0.5*\ul, 3.5*\ul) coordinate (x2);
        \path (4*\ul, 2*\ul) coordinate (x5);
        \path (4*\ul, 3.5*\ul) coordinate (x6);
        \path (x2) -- ++(70:6*\ul) coordinate (x2a);
        \path (x3) -- ++(60:4*\ul) coordinate (x3a);
        \path (x4) -- ++(135:2*\ul) coordinate (x4a);
        \path (x5) -- ++(40:5*\ul) coordinate (x5a);
    }},
}

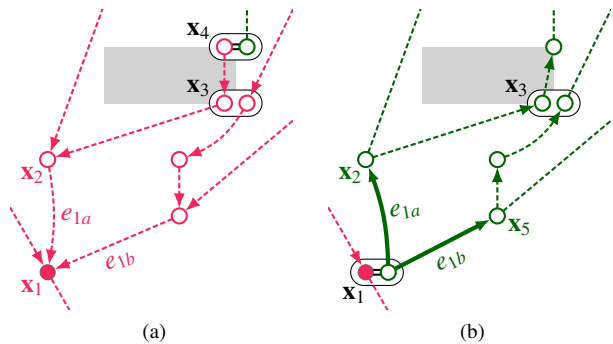
\begin{figure}[!ht]
\centering
\subfloat[\label{fig:overlap1a} ] {%
    \centering
    \begin{tikzpicture}[]
        \clip (-\u, -\u) rectangle ++(7.5*\ul, 8*\ul);
        \pic at (0, 0) {overlap1};

        \node (n1 svy) [svy pt={shift={(-2mm, -2mm)}}{center:$\mx_1$}] at (x1) {};
        \draw [sxlink] (n1 svy) -- ++(-60:3*\ul);
        \draw [slink] (n1 svy) +(120:2*\ul) -- (n1 svy);
        \node (n2 svu) [svu pt={below, shift={(-2mm, -2mm)}}{center:$\mx_2$}] at (x2) {};
        \node [separate={left}{135:$\mx_3$}] at (x3) {};
        \node (n3 svu a) [svu pt, shift={(-\um, 0)}] at (x3) {};
        \node (n3 svu b) [svu pt, shift={(\um, 0)}] at (x3) {};
        \pic at (x4) {merge grp={left}{135:$\mx_4$}{n4 svu}{svu pt}{n4 tvu}{tvu pt}};
        \node (n5 svu) [svu pt={shift={(3, -1.5mm)}}{center:$\mx_5$}] at (x5) {};
        \node (n6 svu) [svu pt={}{}] at (x6) {};

        \draw [slink] (n5 svu) -- (n1 svy)
                node [swatch_stree, pos=0.5, sloped, below] {$\mlink_{1b}$};
        \draw [slink] (n2 svu) edge[bend left=10] 
                node [pos=0.5, label={[swatch_stree, rotate=0, shift={(3mm, 0)}] center:$\mlink_{1a}$}] {} 
                (n1 svy);
        \draw [slink] (x2a) -- (n2 svu);
                % node [swatch_stree, pos=0.5, sloped, above] {$\mlink_{2}$};
        \draw [slink] (x5a) -- (n5 svu);
                % node [swatch_stree, pos=0.5, sloped, above] {$\mlink_{5}$};
                
        \draw [slink]  
            (n4 svu) edge (n3 svu a) 
            (n3 svu a) edge (n2 svu);
        \draw [slink]
            (x3a) edge (n3 svu b)
            (n3 svu b) edge[bend left=20] (n6 svu)
            (n6 svu) edge (n5 svu);
        \draw [tlink]
            (n4 tvu) edge ++(90:2*\ul);
    \end{tikzpicture}
} \hfill
\subfloat[\label{fig:overlap1b}] {%
    \centering
    \begin{tikzpicture}[]
        \clip (-\u, -\u) rectangle ++(7.5*\ul, 8*\ul);
        \pic at (0, 0) {overlap1};

        % \pic at ($(x2) + (.25*\ul, 0)$) {merge grp={shift={(-1mm, 4mm)}}{center:$\mx_2$}{n2 svu}{svu pt}{n2 tvu}{tvu pt}};
        \pic at ($(x1) + (\um, 0)$) {merge grp={shift={(-3mm, -3mm)}}{center:$\mx_1$}{n1 svy}{svy pt}{n1 tvu}{tvu pt}};
        % \node (n1 svy) [svy pt={shift={(-2mm, -2mm)}}{center:$\mx_1$}] at (0.5*\ul, 0.5*\ul) {};
        \draw [sxlink] (n1 svy) -- ++(-60:3*\ul);
        \draw [slink] (n1 svy) +(120:2*\ul) -- (n1 svy);
        \node (n2 tvu) [tvu pt={shift={(-2mm, -2mm)}}{center:$\mx_2$}] at (x2) {};
        \node [separate={left}{135:$\mx_3$}] at (x3) {};
        \node (n3 tvu a) [tvu pt, shift={(-\um, 0)}] at (x3) {};
        \node (n3 tvu b) [tvu pt, shift={(\um, 0)}] at (x3) {};
        \node (n4 tvu) [tvu pt] at (x4) {};
        \node (n5 tvu) [tvu pt={shift={(3mm, -1.5mm)}}{center:$\mx_5$}] at (x5) {};
        % \pic at ($(x5) + (-.25*\ul, 0)$) {merge grp={shift={(0, -4mm)}}{center:$\mx_5$}{n5 svu}{svu pt}{n5 tvu}{tvu pt}};
        \node (n6 tvu) [tvu pt={}{}] at (x6) {};
        
        \draw [tqlink] (n1 tvu) -- (n5 svu)
                node [swatch_ttree, pos=0.5, sloped, below] {$\mlink_{1b}$};
        \draw [tqlink] (n1 tvu) to[bend right=10] 
            node [pos=0.5, label={[swatch_ttree, rotate=10, shift={(3mm, 0)}] center:$\mlink_{1a}$}] {} (n2 svu);
                                
        \draw [tlink] (n2 tvu) -- (x2a);
                % node [swatch_ttree, pos=0.5, sloped, above] {$\mlink_{2}$};
        \draw [tlink] (n5 tvu) -- (x5a);
                % node [swatch_ttree, pos=0.5, sloped, above] {$\mlink_{5}$};
        \draw [tlink]
            (n2 tvu) edge (n3 tvu a)
            (n3 tvu a) edge (n4 tvu);
        \draw [tlink]
            (n5 tvu) edge (n6 tvu)
            (n6 tvu) edge[bend right=20] (n3 tvu b)
            (n3 tvu b) edge (x3a);
        \draw [tlink]
            (n4 tvu) edge ++(90:3*\ul);

    \end{tikzpicture}
}    
\caption{
When overlapping paths are identified, Case O1 of the overlap rule shrinks the $S$-tree, and for each path, queues a query at the most recent link with a source $\mney$ or $\mnvy$ node.
(a) A query (only tracing query is shown) from $\mx_2$ passes through $\mx_3$ and finds links from other paths at $\mx_3$. 
For every $S$-tree node at $\mx_3$ that is $\mneu$ or $\mnvu$ type, the anchored links are searched and the first link ($\mlink_{1a}$ and $\mlink_{1b}$) that is connected to a parent $\mney$ or $\mnvy$ node is identified.
(b) Links in the target direction of the first link are searched. Queued queries are removed from the links, and their anchored $S$-tree nodes are converted to $T$-tree $\mnvu$ nodes. A cast is then queued for each first link.
}
\label{fig:overlap1}
\end{figure}

%% file: fig_overlap2.tex
\tikzset{
    pics/overlap2/.style={ code={ 
        % \begin{pgfonlayer}{background}
        %     \fill [swatch_obs] % obs a
        %         (0, -\ul) -- ++(90:6*\ul) -- ++(0:#1) coordinate (xa1) -- ++(-90:1.5*\ul) coordinate (xa2) -- ++(180:{#1-\ul}) -- ++(-90:4.5*\ul);
        %     \fill [swatch_obs] % obs b
        %         (2*\ul, \ul) -- ++(90:\ul) -- ++(0:1.5*\ul) coordinate (xb1) -- ++(-90:\ul);
        % \end{pgfonlayer}
        % \path (\ul, 6*\ul) coordinate (xtgt);
        % \node (nsrc) [svy pt] at (4.5*\ul, 0) {};
        % \node (ntgt) [tvu pt] at (xtgt) {};
        % \draw [tlink] (ntgt) edge ++(180:2*\ul);
        % \draw [slink] (nsrc) edge ++(-90:2*\ul);
        
        \begin{pgfonlayer}{background}
            \fill [swatch_obs] % obs b
                (1*\ul, 3.5*\ul) coordinate (xb1) -- ++(90:\ul) -- ++(0:2.5*\ul) coordinate (xb3) -- ++(-90:\ul);
            \fill [swatch_obs] % obs a
                (2.5*\ul, 1*\ul) -- ++(90:\ul) -- ++(0:1.5*\ul) coordinate (xa3) -- ++(-90:\ul);
        \end{pgfonlayer}
        \path (4*\ul, 0) coordinate (xsrc);
        \path (0.5*\ul, 7*\ul) coordinate (xtgt);

        \node (nsrc) [svy pt] at (xsrc) {};
        \node (ntgt) [tvu pt] at (xtgt) {};

        \draw [slink] (nsrc) -- ++(-120:2*\ul);
        \draw [tlink] (ntgt) -- ++(180:2*\ul);
    }},
}
\tikzset{
    pics/overlap2i/.style={code={
        \begin{pgfonlayer}{background}
            \fill [swatch_obs] % obs c
                (0, 5.5*\ul) -- ++(90:\ul) -- ++(0:1.5*\ul) -- ++(90:1*\ul) coordinate (xc1) -- ++(0:1*\ul) coordinate (xc2) -- ++(-90:2*\ul) coordinate (xc3);
        \end{pgfonlayer}
    }},
}
\tikzset{
    pics/overlap2ii/.style={code={
        \begin{pgfonlayer}{background}
            \fill [swatch_obs] % obs c
                (0, 5.5*\ul) -- ++(90:\ul) -- ++(0:4*\ul) coordinate (xc2) -- ++(-90:\ul) coordinate (xc3);
        \end{pgfonlayer}
    }},
}

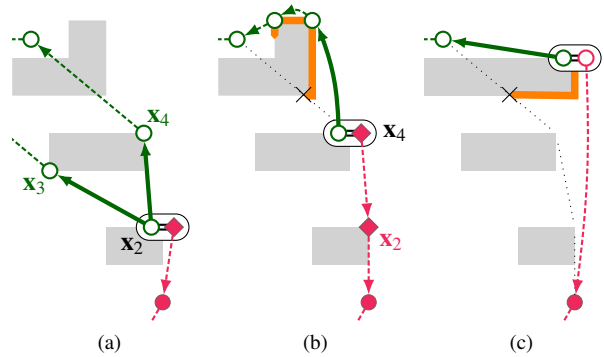
\begin{figure}[!ht]
\centering
\subfloat[\label{fig:overlap2a} ] {%
    \centering
    \begin{tikzpicture}[]
        \clip (0, -\u) rectangle ++(5*\ul, 8.5*\ul);
        \pic at (0, 0) {overlap2};
        \pic at (0, 0) {overlap2i};

        \pic at (xa3) {merge grp={shift={(-4mm, -2.5mm)}}{center:$\mx_2$}{na3 tvu}{tvu pt}{na3 sey}{sey pt}};
        
        \node (nb1) [tvu pt={shift={(-2mm, -2mm)}}{center:$\mx_3$}] at (xb1) {};
        \node (nb3) [tvu pt={shift={(2mm, 2mm)}}{center:$\mx_4$}] at (xb3) {};

        \draw [slink] (na3 sey) edge (nsrc);
        \draw [tqlink] (na3 tvu) edge (nb1);
        \draw [tlink] (nb1) -- ++(140:3*\ul);
        \draw [tqlink] (na3 tvu) edge (nb3);
        \draw [tlink](nb3) edge (ntgt);
    \end{tikzpicture}
} \hfill
\subfloat[\label{fig:overlap2b}] {%
    \centering
    \begin{tikzpicture}[]
        \clip (0, -\u) rectangle ++(5*\ul, 8.5*\ul);
        \pic at (0, 0) {overlap2};
        \pic at (0, 0) {overlap2i};

        \draw [dotted] (nb3) -- (ntgt)
            node (xcol) [coordinate, pos=2/5] {};
        \draw [trace] (xcol) -- (xc3) -- (xc2) -- (xc1) -- ++(-90:\u);
        \node [cross pt] at (xcol) {};
        
        \node (na3) [sey pt={shift={(3mm, -2mm)}}{center:$\mx_2$}] at (xa3) {};
        \node (nc2) [tvu pt={}{}] at (xc2) {};

        \pic at (xb3) {merge grp={shift={(6mm, 0mm)}}{center:$\mx_4$}{nb3 tvu}{tvu pt}{nb3 sey}{sey pt}};
        \node (nc1) [tvu pt] at (xc1) {};

        \draw [slink] 
            (nb3 sey) edge (na3)
            (na3) edge (nsrc);
        \draw [tqlink] (nb3 tvu) 
            edge[bend right=10] 
            % node [pos=0.5, label={[swatch_ttree, rotate=20, shift={(2.5mm,0)}] center:$\mlink_5$}] {} 
            (nc2.-45);
        \draw [tlink] 
            (nc2) edge[bend right=45] (nc1)
            (nc1) edge (ntgt);
            
        % \draw [tqlink] (na3 tvu) edge (nb1);
        % \draw [tlink] (nb1) -- ++(170:3*\ul);
        % \draw [tqlink] (na3 tvu) edge (nb3);
        % \draw [tlink](nb3) edge (ntgt);
    \end{tikzpicture}
} \hfill
\subfloat[\label{fig:overlap2c}] {%
    \centering
    \begin{tikzpicture}[]
        \clip (0, -\u) rectangle ++(5*\ul, 8.5*\ul);
        \pic at (0, 0) {overlap2};
        \pic at (0, 0) {overlap2ii};

        \draw [dotted] 
            (nsrc) -- (na3) -- (nb3) 
            (nb3) edge
                node (xcol) [coordinate, pos=2/5] {}
                (ntgt);
        \draw [trace] (xcol) -- (xc3) -- (xc2) -- ++(180:\u);
        \node [cross pt] at (xcol) {};
        
        % \pic at (xc2) {merge grp={shift={(0, 3.5mm)}}{center:$\mx_5$}{nc2 tvu}{tvu pt}{nc2 svu}{svu pt}};
        \pic at (xc2) {merge grp={}{}{nc2 tvu}{tvu pt}{nc2 svu}{svu pt}};
        
        \draw [tqlink] 
            (nc2 tvu) edge (ntgt);
        \draw [slink]
            (nc2 svu) edge[bend left=5] (nsrc);

    \end{tikzpicture}
}    
\caption{
    Case O2 of the overlap rule handles queries with expensive cost-to-come paths. 
    (a) After a successful cast to $\mx_2$, the path is found to have a larger cost-to-come than the minimum at $\mx_2$, and the target node is replaced by an $S$-tree $\mney$ node.
    (b) If a cast from an $\mney$ node is successful, the target node is replaced by an $S$-tree $\mney$ node ($\mx_3, \mx_4$). 
    If consecutive $\mney$ nodes have different sides, the path is discarded ($\mx_3$).
    An unsuccessful cast will generate a trace with the same side as the $\mney$ node ($\mx_4$) and call Case O1 when the trace becomes castable.
    (c) The trace resumes normal behavior after all $\mney$ source nodes ($\mx_2, \mx_4$) are pruned from the path.
}
\label{fig:overlap2}
\end{figure}

%% file: fig_overlap3.tex
\tikzset{
    pics/overlap3/.style n args={0}{ code={ 
        \begin{pgfonlayer}{background}
            \fill [swatch_obs] %lowest
                (3*\ul, 1*\ul) -- ++(90:\ul) -- ++(0:2*\ul) coordinate (xa3)  -- ++(-90:\ul);
            \fill [swatch_obs] % center
                (1*\ul, 4*\ul) coordinate (xb1) -- ++(90:\ul) -- ++(0:3.5*\ul) coordinate (xb3) -- ++(-90:\ul);
            \fill [swatch_obs] % top left
                (0.5*\ul, 6.5*\ul) coordinate (xc1) -- ++(90:\ul) -- ++(0:2.5*\ul) coordinate (xc3) -- ++(-90:\ul);    
            % \fill [swatch_obs] % top right
                % (5.5*\ul, 7*\ul) coordinate (xd1) -- ++(90:1.5*\ul) -- ++(0:1.5*\ul) -- ++(-90:1.5*\ul);
        \end{pgfonlayer}
        \path (5*\ul, 0) coordinate (xsrc);
        \node (nsrc) [svy pt={}{}] at (xsrc) {};
        \draw [slink] (nsrc) edge ++(-90:2*\ul);
        \node (nsrc2) [svy pt] at (5.8*\ul, 0.5*\ul) {};
    }},
}

\begin{figure}[!ht]
\centering
\subfloat[\label{fig:overlap3a} ] {%
    \centering
    \begin{tikzpicture}[]
        \clip (-\u, -\u) rectangle ++(7.5*\ul, 9*\ul);
        \pic at (0, 0) {overlap3};

        % obs a (lowest) nodes
        \node [separate={shift={(-6mm, 3mm)}}{center:$\mx_2$}, inner xsep=5mm] at (xa3) {};
        \draw [merge] ($(xa3) + (-2.5mm, 0)$) -- (xa3);
        \node (na3 tvu) [tvu pt, shift={(-3mm, 0)}] at (xa3) {};
        \node (na3 svy a) [svy pt, shift={(0, 0)}] at (xa3) {};
        \node (na3 svy b) [svy pt, shift={(3mm, 0)}] at (xa3) {};

        %obs b nodes
        \node (nb1 svy) [svy pt={shift={(-2mm,-2mm)}}{center:$\mx_3$}] at (xb1) {};
        \node (nb3 svy) [svy pt={shift={(2mm,2mm)}}{center:$\mx_4$}] at (xb3) {};

        %obs c nodes
        \node (nc1 svu) [svu pt={shift={(-2mm,-2mm)}}{center:$\mx_5$}] at (xc1) {};
        \node (nc3 svy) [svy pt={shift={(2mm,2mm)}}{center:$\mx_6$}] at (xc3) {};

        %obs d nodes
        % \node (nd1 svu) [svu pt={shift={(-2mm,-2mm)}}{center:$\mx_5$}] at (xd1) {};

        \draw [tlink]
            (na3 tvu) -- ++(170:8*\ul);
        \draw [slink] % passing thru na3 svy a
            (na3 svy a) edge[bend right=0] (nsrc);
        \draw [slink] % passing thru na3 svy b
            (nb1 svy) +(150:3*\ul) edge (nb1 svy)
            (nb1 svy) edge[bend left=20] (na3 svy b)
            (nc1 svu) +(150:3*\ul) edge (nc1 svu)
            (nc1 svu) edge (nb3 svy)
            (nc3 svy) +(160:3*\ul) edge (nc3 svy)
            (nc3 svy) edge[bend left=15] (nb3 svy)
            (nb3 svy) edge[bend left=10] (na3 svy b)
            (na3 svy b) edge (nsrc2)
            (nsrc2) -- ++(-90:2*\ul);
        % \draw [slink] % passing thru na3 svy c
        %     (nd1 svu) +(100:2*\ul) edge (nd1 svu)
        %     (nd1 svu) edge (na3 svy c)
        %     (na3 svy c) edge ++(-85:4*\ul);

    \end{tikzpicture}
} \hfill
\subfloat[\label{fig:overlap3b}] {%
    \centering
    \begin{tikzpicture}[]
        \clip (-\u, -\u) rectangle ++(7.5*\ul, 9*\ul);
        \pic at (0, 0) {overlap3};

        % obs a (lowest) nodes
        \node [separate={shift={(-3.5mm, 3.5mm)}}{center:$\mx_2$}, inner xsep=5mm] at (xa3) {};
        \draw [merge] ($(xa3) + (-2.5mm, 0)$) -- (xa3);
        \node (na3 tvu) [tvu pt, shift={(-3mm, 0)}] at (xa3) {};
        \node (na3 svy) [svy pt, shift={(0, 0)}] at (xa3) {};
        \node (na3 sey) [sey pt, shift={(3mm, 0)}] at (xa3) {};

        %obs b nodes
        \pic at (xb3) {merge grp={shift={(3.5mm, 3mm)}}{center:$\mx_4$}{nb3 tvu}{tvu pt}{nb3 sey}{sey pt}};
        % \node (nb3 sey) [sey pt={shift={(-2mm,-2mm)}}{center:$\mx_4$}] at (xb3) {};

        %obs c nodes
        % \pic at (xc1) {merge grp={shift={(-3.5mm, -2.5mm)}}{center:$\mx_5$}{nc1 tvu}{tvu pt}{nc1 svu}{svu pt}};
        \node (nc1 tvu) [tvu pt={shift={(-2mm,-2mm)}}{center:$\mx_5$}] at (xc1) {};
        \node (nc3 sey) [sey pt={shift={(2mm,2mm)}}{center:$\mx_6$}] at (xc3) {};

        \draw [tlink]
            (na3 tvu) to ++(170:8*\ul)
            (nc1 tvu) to ++(150:3*\ul);
        \draw [tqlink]
            (nb3 tvu) edge (nc1 tvu);
        \draw [slink] % passing thru na3 svy a
            (na3 svy) edge[bend right=0] (nsrc);
        \draw [slink] % passing thru na3 sey b
            (nc3 sey) +(160:3*\ul) edge (nc3 sey)
            (nc3 sey) edge[bend left=15] (nb3 sey)
            (nb3 sey) edge[bend left=10] (na3 sey)
            (na3 sey) edge (nsrc2)
            (nsrc2) to ++(-90:2*\ul);

    \end{tikzpicture}
}    
\caption{
    Case O3 of the overlap rule handles the case when a cast finds more expensive cost-to-come paths anchored at the same $S$-tree $\mnvy$ node at the destination ($\mx_2$).
    Each expensive path is handled like Case O2. The $S$-tree $\mnvy$ nodes of each expensive path from $\mx_2$ are converted to $\mney$ nodes ($\mx_3, \mx_4, \mx_6$). 
    A link connecting a consecutive pair of $\mney$ nodes with different sides is discarded ($\mx_2$ to $\mx_3$).
    Case O1 is called if an $\mnvu$ is encountered ($\mx_5$), where the $S$-tree is shrunk, target queries are discarded, and a new cast is queued from the first link with a parent $\mney$ node ($\mx_4$ to $\mx_5$).
}
\label{fig:overlap3}
\end{figure}
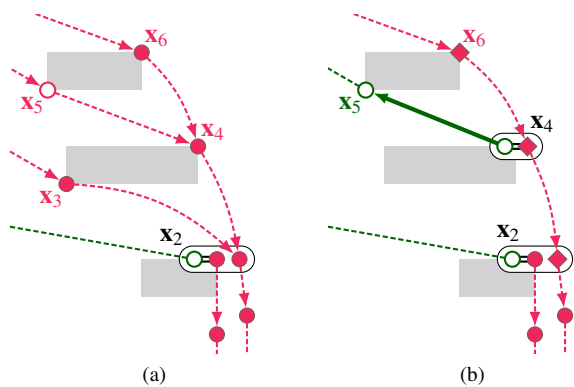

%% file: fig_overlap4.tex
\tikzset{
    pics/overlap4/.style n args={0}{ code={ 
        \begin{pgfonlayer}{background}
            \fill [swatch_obs] %lowest
                (3*\ul, 1*\ul) coordinate (xa1) -- ++(90:\ul) -- ++(0:3*\ul) coordinate (xa3) -- ++(-90:\ul) coordinate (xa4);
            \fill [swatch_obs] % mid
                (1.5*\ul, 3.5*\ul) coordinate (xb1) -- ++(90:\ul) -- ++(0:4*\ul) coordinate (xb3) -- ++(-90:\ul);
            \fill [swatch_obs]
                (1*\ul, 5.5*\ul) coordinate (xc1) -- ++(90:\ul) -- ++(0:2*\ul) coordinate (xc3) -- ++(-90:\ul);
        \end{pgfonlayer}

        \path (1*\ul, 7.5*\ul) coordinate (xtgt);
        \path (7*\ul, 0) coordinate (xsrc);
        \node (nsrc) [svu pt] at (xsrc) {};
        \node (ntgt) [tvy pt] at (xtgt) {};
        \draw [tlink] (ntgt) -- ++(45:2*\ul);
        \draw [slink] (nsrc) -- ++(-90:2*\ul);
        
    }},
}

\begin{figure}[!ht]
\centering
\subfloat[\label{fig:overlap4a} ] {%
    \centering
    \begin{tikzpicture}[]
        \clip (0, -\u) rectangle ++(7.5*\ul, 8.5*\ul);
        \pic at (0, 0) {overlap4};

        \pic at (xc1) {merge grp={shift={(-3.5mm, -3mm)}}{center:$\mx_2$}{nc1 tey}{tey pt}{nc1 svu}{svu pt}};

        \node (nb1) at (xb1) [svu pt={shift={(-2mm, -2mm)}}{center:$\mx_4$}] {};
        \node (nb3) at (xb3) [svu pt={shift={(2mm, 2mm)}}{center:$\mx_3$}] {};

        \draw [sqlink] 
            (nc1 svu) edge (nb1)
            (nc1 svu) edge (nb3);
        \draw [slink]
            (nb1) edge (nsrc)
            (nb3) edge ++(-45:4*\ul);
        \draw [tlink]
            (nc1 tey) edge (ntgt);

    \end{tikzpicture}
} \hfill
\subfloat[\label{fig:overlap4b}] {%
    \centering
    \begin{tikzpicture}[]
        \clip (0, -\u) rectangle ++(7.5*\ul, 8.5*\ul);
        \pic at (0, 0) {overlap4};

        \begin{pgfonlayer}{background}
            \draw [dotted] (xsrc) -- (nb1)
                node (xcol) [coordinate, pos=2/7] {};
            \draw [trace] (xcol) -- (xa4) -- (xa3) -- ++(180:\u);
            \draw [trace] (xcol) -- (xa1) -- ++(90:\u);
            \node at (xcol) [cross pt] {};
        \end{pgfonlayer}
        
        \node (nc1) at (xc1) [tey pt={shift={(-2mm, -2mm)}}{center:$\mx_2$}] {};
        \node (nb1) at (xb1) [tey pt={shift={(-2mm, -2mm)}}{center:$\mx_4$}] {};

        \pic at (xa1) {merge grp={}{}{na1 tvu}{tvu pt}{na1 svu}{svu pt}};
        \pic at (xa3) {merge grp={}{}{na3 tvu}{tvu pt}{na3 svu}{svu pt}};

        \draw [slink]
            (na1 svu) edge[bend right=10] (nsrc)
            (na3 svu) edge (nsrc);
        \draw [tlink]
            (nc1) edge (ntgt)
            (nb1) edge (nc1);
        \draw [tqlink]
            (na1 tvu) edge (nb1)
            (na3 tvu) edge[bend right=10] (nb1);

    \end{tikzpicture}
}    
\caption{
    Case O4 extends Case O2 for cost-to-go.
    (a) In a successful cast, Case O4 is triggered when the cost-to-go is larger than the minimum at the source node, causing the source node to be replaced by a $T$-tree $\mney$ node.
    A successful cast to a $T$-tree $\mney$ node will cause the cast's source node to be replaced by an $\mney$ node regardless of the cost ($\mx_3, \mx_4$).
    A consecutive pair of $\mney$ nodes with different sides will cause the path passing through the nodes to be discarded ($\mx_3$).
    (b) Unlike Case O2, there are no restrictions to traces, and Case O1 will not be called.
}
\label{fig:overlap4}
\end{figure}
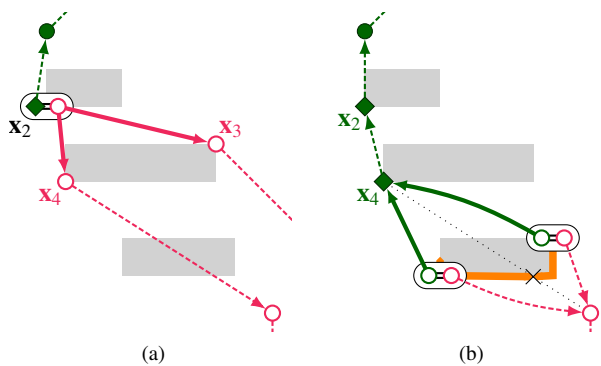

%% file: fig_overlap5.tex
\tikzset{
    pics/overlap5/.style n args={1}{ code={ 
        \begin{pgfonlayer}{background}
            \fill [swatch_obs] %lowest
                (4*\ul, 0*\ul) coordinate (xa1) -- ++(90:\ul) -- ++(0:2.5*\ul) coordinate (xa3) -- ++(-90:\ul) coordinate (xa4);
            \fill [swatch_obs] % mid
                (2.5*\ul, 2.5*\ul) coordinate (xb1) -- ++(90:\ul) -- ++(0:3.5*\ul) coordinate (xb3) -- ++(-90:\ul);
            \fill [swatch_obs]
                (2*\ul, 5.5*\ul) coordinate (xc1) -- ++(90:\ul) -- ++(0:2*\ul) coordinate (xc3) -- ++(-90:\ul);
        \end{pgfonlayer}

        \node [separate={shift={(5mm, -3.5mm)}}{center:$\mx_2$}, inner xsep=5mm] at (xc1) {};
        \draw [merge] ($(xc1) + (2.5mm, 0)$) -- (xc1);
        \node (nc1 tvy) at (xc1) [tvy pt, shift={(0, 0)}] {};
        \node (nc1 svu) at (xc1) [svu pt, shift={(3mm, 0)}] {};
        \node (nc1 tey) at (xc1) [#1, shift={(-3mm, 0)}]  {};
        
        \path (2*\ul, 7.5*\ul) coordinate (xtgt);
        \node (ntgt) at (xtgt) [tvy pt] {};
        \node (ntgt2) at (1.2*\ul, 7*\ul) [tvy pt] {};
        \draw [tlink] 
            (ntgt) edge ++(90:2*\ul)
            (ntgt2) edge ++(90:2*\ul)
            (nc1 tvy) edge (ntgt);
        \draw [slink]
            (nc1 svu) -- ++(-10:8*\ul);

        \node (nb1) at (xb1) [#1={shift={(-2mm,-2mm)}}{center:$\mx_4$}] {};
        \node (na3) at (xa3) [tvu pt={shift={(2mm,2mm)}}{center:$\mx_5$}] {};
        \node (na1) at (xa1) [#1={shift={(-2mm,-2mm)}}{center:$\mx_6$}] {};
        \draw [tlink]
            (na3) +(-30:6*\ul) edge (na3)
            (na1) +(-20:6*\ul) edge (na1);
        \draw [tlink]
            (na1) edge[bend left=10] (nb1)
            (na3) edge (nb1)
            (nb1) edge[bend left=10] (nc1 tey)
            (nc1 tey) edge (ntgt2);
    }},
}

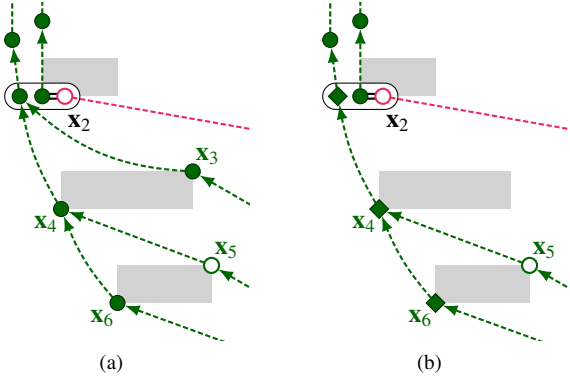
\begin{figure}[!ht]
\centering
\subfloat[\label{fig:overlap5a} ] {%
    \centering
    \begin{tikzpicture}[]
        \clip (0, -\ul) rectangle ++(7.5*\ul, 9*\ul);
        \pic at (0, 0) {overlap5={tvy pt}};

        \node (nb3) at (xb3) [tvy pt={shift={(2mm,2mm)}}{center:$\mx_3$}] {};
        \draw [tlink]
            (nb3) +(-30:6*\ul) edge (nb3);
        \draw [tlink]
            (nb3) edge[bend left=20] (nc1 tey);
    \end{tikzpicture}
} \hfill
\subfloat[\label{fig:overlap5b}] {%
    \centering
    \begin{tikzpicture}[]
        \clip (0, -\ul) rectangle ++(7.5*\ul, 9*\ul);
        \pic at (0, 0) {overlap5={tey pt}};
    \end{tikzpicture}
}    
\caption{
Case O5 extends Case O3 to cost-to-go.
(a) A successful cast finds the smallest cost-to-go at the source node's corner ($\mx_2$). More expensive cost-to-go paths at $\mx_2$ are scanned, and the relevant $T$-tree $\mnvy$ nodes along the path are converted to $T$-tree $\mney$ nodes.
(b) A path will be discarded if it passes through a consecutive pair of $\mney$ nodes with different sides ($\mx_3$). 
Unlike Case O3, Case O5 does not call Case O1.
}
\label{fig:overlap5}
\end{figure}

%% file: fig_overlap6and7.tex
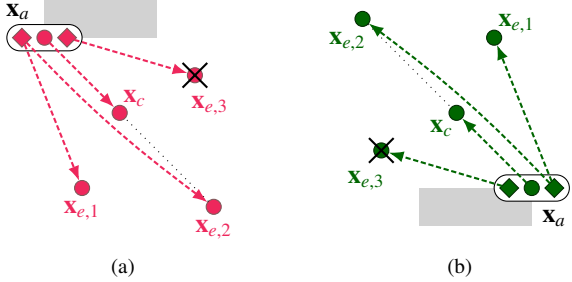
\begin{figure}[!ht]
\centering
\subfloat[\label{fig:overlap6and7a} ] {%
    \centering
    \begin{tikzpicture}[]
        \clip (-\u, -\u) rectangle ++(7*\ul, 7*\ul);

        \fill [swatch_obs]
            (1*\ul, 5*\ul) coordinate (xa) -- ++(90:\ul) -- ++(0:3*\ul) -- ++(-90:\ul);

        \node at (xa) [separate={shift={(-3.5mm, 3.5mm)}}{center:$\mx_a$}, inner xsep=5mm] {};
        \node (na sey a) at (xa)  [sey pt, shift={(-3mm, 0)}] {};
        \node (na svy)   at (xa)  [svy pt, shift={(0, 0)}] {};
        \node (na sey b) at (xa)  [sey pt, shift={(3mm, 0)}] {};
        \node (nc) at (3*\ul, 3*\ul) [svy pt={shift={(2mm,2mm)}}{center:$\mx_c$}] {};
        \node (ne1) at (2*\ul, \ul) [svy pt={shift={(0,-3mm)}}{center:$\mx_{e,1}$}] {};
        \node (ne2) at (5.5*\ul, 0.5*\ul) [svy pt={shift={(0,-3mm)}}{center:$\mx_{e,2}$}] {};
        \node (ne3) at (5*\ul, 4*\ul) [svy pt={shift={(2mm,-4mm)}}{center:$\mx_{e,3}$}] {};
        \node  at (ne3) [prune] {};

        \draw [dotted]
            (ne2) -- (nc);
        \draw [slink]
            (na sey a) edge (ne1)
            (na sey a) edge[bend right=5] (ne2)
            (na sey b) edge (ne3)
            (na svy) edge (nc);
    \end{tikzpicture}
} \hfill
\subfloat[\label{fig:overlap6and7b}] {%
    \centering
    \begin{tikzpicture}[]
        \clip (-\u, -\u) rectangle ++(7*\ul, 7*\ul);

        \fill [swatch_obs]
            (2*\ul, 0*\ul) -- ++(90:\ul) -- ++(0:3*\ul) coordinate (xa) -- ++(-90:\ul); % 6,1

        \node at (xa) [separate={shift={(3mm, -4mm)}}{center:$\mx_a$}, inner xsep=5mm] {};
        \node (na tey b) at (xa)  [tey pt, shift={(-3mm, 0)}] {};
        \node (na tvy)   at (xa)  [tvy pt, shift={(0, 0)}] {};
        \node (na tey a) at (xa)  [tey pt, shift={(3mm, 0)}] {};
        \node (nc) at (3*\ul, 3*\ul) [tvy pt={shift={(-2mm,-2mm)}}{center:$\mx_c$}] {};
        \node (ne1) at (4*\ul, 5*\ul) [tvy pt={shift={(3mm, 2mm)}}{center:$\mx_{e,1}$}] {};
        \node (ne2) at (0.5*\ul, 5.5*\ul) [tvy pt={shift={(-2mm,-3mm)}}{center:$\mx_{e,2}$}] {};
        \node (ne3) at (1*\ul, 2*\ul) [tvy pt={shift={(-2mm,-4mm)}}{center:$\mx_{e,3}$}] {};
        \node  at (ne3) [prune] {};

        \draw [dotted]
            (ne2) -- (nc);
        \draw [tlink]
            (na tey a) edge (ne1)
            (na tey a) edge[bend right=5] (ne2)
            (na tey b) edge (ne3)
            (na tvy) edge (nc);
    \end{tikzpicture}
}    
\caption{
    Case O6 is shown in (a) and Case O7 in (b).
    The cheapest path passes through $\mx_a$ and $\mx_c$. 
    More expensive paths pass through $\mx_a$ via $\mx_{e,1}$, $\mx_{e,2}$, or $\mx_{e,3}$. 
    The expensive path from $\mx_{e,3}$ is discarded as it does not satisfy $\mtdir\mside(\mv_e \times \mv_c) < 0$.
}
\label{fig:overlap6and7}
\end{figure}

%% file: fig_thm1case1.tex
\begin{figure}[!ht]
\centering
\subfloat[\label{fig:thm1case1a} ] {%
    \centering
    \begin{tikzpicture}[]
        \clip (-\u, 0) rectangle ++(7.5*\ul, 6*\ul);
        \fill [swatch_obs] (\ul, 5*\ul) coordinate (xa) rectangle ++(3*\ul, \ul);

        \node (na) at (xa) [vy pt={shift={(-2mm, 3mm)}}{center:$\mx_a$}{}] {};
        \node (nc) at (2*\ul, 2*\ul) [vy pt={shift={(-2mm, -2mm)}}{center:$\mx_c$}{}] {};
        \node (ne) at (4*\ul, 2.5*\ul) [vy pt={shift={(2mm,3mm)}}{center:$\mx_e$}{}] {};

        \draw [link] (na) -- (nc)
            node (xi) [coordinate, pos=0.4] {};
        \path (ne) -- (xi)
            node (xf) [coordinate, pos=3] {};
        \draw [link] (xf) -- (ne);
        \node (ni) at (xi) [circle, minimum size=1mm, inner sep=0, fill=black, label={[shift={(-2mm, -2mm)}] center:$\mx_i$}] {};
        \draw [link] 
            (ne) edge ++(-10:5*\ul)
            (nc) edge ++(-60:5*\ul);
        \draw [dotted] (ne) -- (na);
    \end{tikzpicture}
} \hfill
\subfloat[\label{fig:thm1case1b}] {%
    \centering
    \begin{tikzpicture}[]
        \clip (-\u, 0) rectangle ++(7.5*\ul, 6*\ul);
        \fill [swatch_obs] (\ul, 5*\ul) coordinate (xa) rectangle ++(3*\ul, \ul);

        \node (na) at (xa) [vy pt={shift={(-2mm, 3mm)}}{center:$\mx_a$}{}] {};
        \node (nc) at (1*\ul, 3.5*\ul) [vy pt={shift={(-3mm, 2mm)}}{center:$\mx_c$}{}] {};
        \node (ncc) at (2.5*\ul, 0.5*\ul) [vy pt={}{}{}] {};
        \node (ne) at (4*\ul, 1.5*\ul) [vy pt={shift={(2mm,3mm)}}{center:$\mx_e$}{}] {};

        \draw [link] (na) -- (nc);
        \draw [link] (nc) -- (ncc)
            node (xi) [coordinate, pos=0.4] {};
        \path (ne) -- (xi)
            node (xf) [pos=3, coordinate] {};
        \draw [dotted] 
            (ne) edge (nc)
            (ne) edge (na);
        \draw [link] 
            (xf) edge (ne);
        \node (ni) at (xi) [circle, minimum size=1mm, inner sep=0, fill=black, label={[shift={(-2mm, -2mm)}] center:$\mx_i$}] {};
        \draw [link] 
            (ne) edge ++(-10:5*\ul)
            (ncc) edge ++(-60:5*\ul);
    \end{tikzpicture}
}    
\caption{
Theorem \ref{thm:newex}'s Case 1.1 is shown in (a) and 1.2 is shown in (b).
In both cases, if it is more expensive to reach $\mx_a$ for a query $\mquery_e$ from $\mx_e$, than another query $\mquery_c$ from $\mx_c$, then the query $\mquery_e$ will find a longer path to $\mx_i$ than $\mquery_c$.
$\mx_i$ lies along the shorter path found by $\mquery_c$.
}
\label{fig:thm1case1}
\end{figure}
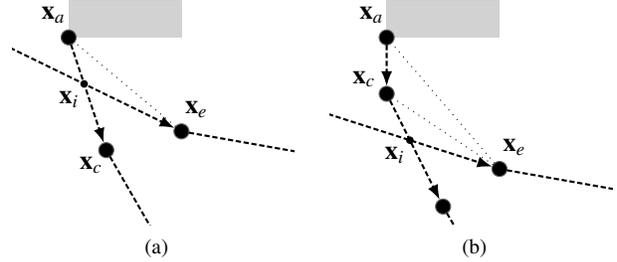

%% file: fig_thm1case2.tex
\begin{figure}[!ht]
\centering
\subfloat[\label{fig:thm1case2a} ] {%
    \centering
    \begin{tikzpicture}[]
        \clip (-\u, 0) rectangle ++(7.5*\ul, 6*\ul);
        \fill [swatch_obs] (\ul, 5*\ul) coordinate (xa) rectangle ++(3*\ul, \ul);

        \node (na) at (xa) [vy pt={shift={(-2mm, 3mm)}}{center:$\mx_a$}{}] {};
        \node (nc) at (2.5*\ul, 0.5*\ul) [vy pt={shift={(-3mm, 0)}}{center:$\mx_c$}{}] {};
        \node (ne) at (4*\ul, 3*\ul) [vy pt={shift={(2mm,2mm)}}{center:$\mx_e$}{}] {};

        \draw [link] (na) -- (nc)
            node (xj) [coordinate, pos=0.3] {}
            node (xi) [coordinate, pos=0.7] {};
        \path (ne) -- (xj)
            node (nd) [pos=-0.8, vy pt={shift={(2mm,-3mm)}}{center:$\mx_d$}{}] {}
            node (xej) [pos=3, coordinate] {};
        \path (nd) -- (xi)
            node (xf) [coordinate, pos=3] {};
        \draw [dotted] 
            (nd) edge (xej)
            (ne) edge (na);
        \draw [link] 
            (xf) edge (nd);
        \node (ni) at (xi) [circle, minimum size=1mm, inner sep=0, fill=black, label={[shift={(3mm, 3mm)}] center:$\mx_i$}] {};
        \node (nj) at (xj) [circle, minimum size=1mm, inner sep=0, fill=black, label={[shift={(-2mm, -2mm)}] center:$\mx_j$}] {};
        \draw [link] 
            (nd) edge ++(30:5*\ul)
            (nc) edge ++(-60:5*\ul);
    \end{tikzpicture}
} \hfill
\subfloat[\label{fig:thm1case2b}] {%
    \centering
    \begin{tikzpicture}[]
        \clip (-\u, 0) rectangle ++(7.5*\ul, 6*\ul);
        \fill [swatch_obs] (\ul, 5*\ul) coordinate (xa) rectangle ++(3*\ul, \ul);

        \node (na) at (xa) [vy pt={shift={(-2mm, 3mm)}}{center:$\mx_a$}{}] {};
        \node (nc) at (1*\ul, 4*\ul) [vy pt={shift={(-3mm, 1mm)}}{center:$\mx_c$}{}] {};
        \path (3*\ul, -0.5*\ul) coordinate (xcc);
        \node (ne) at (4*\ul, 2.5*\ul) [vy pt={shift={(2mm,2mm)}}{center:$\mx_e$}{}] {};

        \draw [link] (na) -- (nc);
        \draw [link] (nc) -- (xcc)
            node (xj) [coordinate, pos=0.2] {}
            node (xi) [coordinate, pos=0.55] {};
        \path (ne) -- (xj)
            node (nd) [pos=-0.8, vy pt={shift={(2mm,-3mm)}}{center:$\mx_d$}{}] {}
            node (xej) [pos=3, coordinate] {};
        \path (nd) -- (xi)
            node (xf) [coordinate, pos=3] {};
        \draw [dotted] 
            (nd) edge (xej)
            (ne) edge (nc)
            (ne) edge (na);
        \draw [link] 
            (xf) edge (nd);
        \node (ni) at (xi) [circle, minimum size=1mm, inner sep=0, fill=black, label={[shift={(3mm, -1.5mm)}] center:$\mx_i$}] {};
        \node (nj) at (xj) [circle, minimum size=1mm, inner sep=0, fill=black, label={[shift={(-2mm, -2mm)}] center:$\mx_j$}] {};
        \draw [link] 
            (nd) edge ++(30:5*\ul);
            % (ncc) edge ++(-60:5*\ul);
    \end{tikzpicture}
}    
\caption{
Theorem \ref{thm:newex}'s Case 1.1 is shown in (a) and 1.2 is shown in (b).
In both cases, if it is more expensive to reach $\mx_a$ for a query $\mquery_e$ from $\mx_e$, than another query $\mquery_c$ from $\mx_c$, then the query $\mquery_e$ will find a longer path to $\mx_i$ than $\mquery_c$.
$\mx_i$ lies along the shorter path found by $\mquery_c$.
}
\label{fig:thm1case2}
\end{figure}
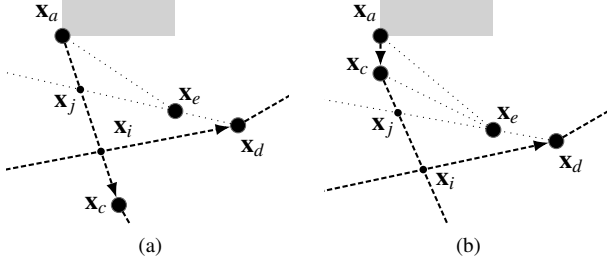

%% file: fig_thm1case3A.tex
\begin{figure}[!ht]
\centering
\subfloat[\label{fig:thm1case3Aa} ] {%
    \centering
    \begin{tikzpicture}[]
        \clip (-\u, -\u) rectangle ++(7.5*\ul, 6.5*\ul);

        \node (n0) at (6.5*\ul, 0*\ul) [vy pt={shift={(0,3mm)}}{center:$\mx_0$}{}] {};
        \fill [swatch_obs] (\ul, 5*\ul) coordinate (xa) rectangle ++(\ul, \ul);
        \fill [swatch_obs] (2*\ul, 3.5*\ul) coordinate (xe) rectangle ++(\ul, \ul);        
        \fill [swatch_obs] (3*\ul, 2*\ul) coordinate (xd) rectangle ++(\ul, \ul);        
        \fill [swatch_obs] (4.25*\ul, 1*\ul) coordinate (x1) rectangle ++(\ul, \ul);
        \path (xe) -- (xa)
            node (xae) [coordinate, pos=3] {};
        \path (x1) -- (xd)
            node (xd1) [coordinate, pos=7] {};
        \path (n0) -- (x1)
            node (x10) [coordinate, pos=5] {};
        \draw [dotted]
            (xae) edge (xa)
            (xd1) edge (xd)
            (x10) edge (x1);

        \node (na) at (xa) [vy pt={shift={(2mm, 2mm)}}{center:$\mx_a$}{}] {};
        \node (ne) at (xe) [vy pt={shift={(2mm, 2mm)}}{center:$\mx_e$}{}] {};
        \node (nd) at (xd) [vy pt={shift={(2mm, 2mm)}}{}{}] {};
        \node (n1) at (x1) [vy pt={}{}{}] {};
        \node (nc) at (1.5*\ul, 1*\ul) [vy pt={shift={(-2mm,-2mm)}}{center:$\mx_c$}{}] {};
            
        \draw [link]
            (na) edge (ne)
            (ne) edge (nd)
            (nd) edge (n1)
            (n1) edge (n0);
        \draw [link]
            (na) edge (nc)
            (nc) edge[bend right=15] (n0);
        
    \end{tikzpicture}
} \hfill
\subfloat[\label{fig:thm1case3Ab}] {%
    \centering
    \begin{tikzpicture}[]
        \clip (-\u, -\u) rectangle ++(7.5*\ul, 6.5*\ul);

        \node (n0) at (6*\ul, 0*\ul) [vy pt={shift={(-3mm,-1mm)}}{center:$\mx_0$}{}] {};
        \fill [swatch_obs] (0, 5*\ul) coordinate (xa) rectangle ++(\ul, \ul);
        \fill [swatch_obs] (1.5*\ul, 4*\ul) coordinate (xe) rectangle ++(\ul, \ul);        
        \fill [swatch_obs] (3*\ul, 3.5*\ul) coordinate (xd) rectangle ++(\ul, \ul);        
        \fill [swatch_obs] (6*\ul, 3*\ul) coordinate (x1) rectangle ++(-\ul, -\ul);
        \path (xe) -- (xa)
            node (xae) [coordinate, pos=3] {};
        \path (xd) -- (xe)
            node (xed) [coordinate, pos=3] {};
        \path (x1) -- (xd)
            node (xd1) [coordinate, pos=7] {};
        \draw [dotted]
            (xae) edge (xa)
            (xed) edge (xe)
            (xd1) edge (xd);

        \node (na) at (xa) [vy pt={shift={(2mm, 2mm)}}{center:$\mx_a$}{}] {};
        \node (ne) at (xe) [vy pt={shift={(2mm, 2mm)}}{center:$\mx_e$}{}] {};
        \node (nd) at (xd) [vy pt={shift={(2mm, 2mm)}}{}{}] {};
        \node (n1) at (x1) [vy pt={shift={(-3mm, -2mm)}}{center:$\mnode_{-\mside}$}{}] {};
        \node (nc) at (2*\ul, 2.5*\ul) [vy pt={shift={(-2mm,-2mm)}}{center:$\mx_c$}{}] {};
            
        \draw [link]
            (na) edge (ne)
            (ne) edge (nd)
            (nd) edge (n1)
            (n1) edge (n0);
        \draw [link]
            (na) edge (nc)
            (nc) edge(n0);
        
    \end{tikzpicture}
}    
\caption{
In Case 3 of Theorem 1, a node $\mnode_{-\mside}$ is shown to exist on the longer path.
(a) If $\mnode_{-\mside}$ does not exist, the longer path will bend monotonically to one side ($\mside$-side for $S$-tree, $(-\mside)$-side for $T$-tree) when viewed from the root node, and a cheaper path from $\mx_c$ cannot exist.
(b) As such, there must be at least one $\mnode_{-\mside}$ node to bend the path passing through $\mx_e$ to make it longer than the path passing through $\mx_c$.
}
\label{fig:thm1case3A}
\end{figure}
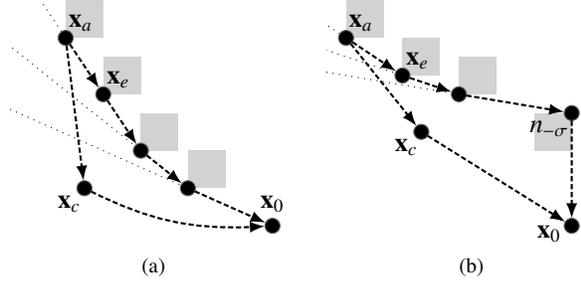

%% file: fig_thm1case3B.tex
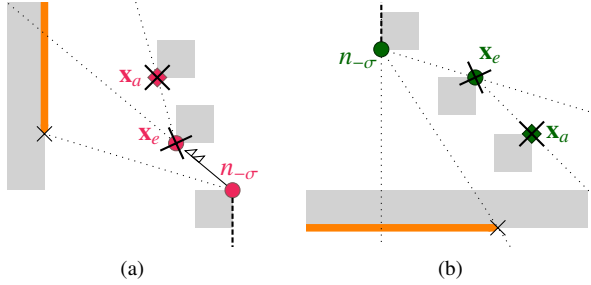
\begin{figure}[!ht]
\centering
\subfloat[\label{fig:thm1case3Ba} ] {%
    \centering
    \begin{tikzpicture}[]
        \clip (-\u, -\u) rectangle ++(7.5*\ul, 6.5*\ul);

        \path (6*\ul, -6*\ul) coordinate (x0) {};
        \fill [swatch_obs] (4*\ul, 4*\ul) coordinate (xa) rectangle ++(\ul, \ul);
        \fill [swatch_obs] (4.5*\ul, 2.25*\ul) coordinate (xe) rectangle ++(\ul, \ul);        
        \fill [swatch_obs] (6*\ul, 1*\ul) coordinate (x1) rectangle ++(-\ul, -\ul);
        \fill [swatch_obs] (0, 1*\ul) -- ++(90:5*\ul) -- ++(0:\ul) coordinate (t2) -- ++(-90:5*\ul) coordinate (t1);

        \path (t1) -- (t2) 
            node (tcol) [pos=0.3, coordinate] {};
        \draw [trace] (tcol) -- (t2) -- ++(90:\u);
        
        \path (xe) -- (xa)
            node (xae) [coordinate, pos=4] {};
        \path (x1) -- (xe)
            node (xe1) [coordinate, pos=7] {};
        \draw [dotted]
            (x1) edge (tcol)
            (xae) edge (xe)
            (xe1) edge (xe);

        \node (na) at (xa) [sey pt={shift={(-3.5mm, 0)}}{center:$\mx_a$}] {};
        \node at (xa) [prune] {};
        \node (ne) at (xe) [svy pt={shift={(-3.5mm, 1mm)}}{center:$\mx_e$}] {};
        \node at (xe) [prune, rotate=-20] {};
        \node (n1) at (x1) [svy pt={shift={(1mm, 2.5mm)}}{center:$\mnode_{-\mside}$}] {};
        \node [cross pt] at (tcol) {};

        \path (n1) -- (ne)
            node (xray) [pos=0.95, coordinate] {};
        \draw [rayr] (n1) -- (xray);
            
        \draw [link]
            (n1) edge (x0);
        
    \end{tikzpicture}
} \hfill
\subfloat[\label{fig:thm1case3Bb}] {%
    \centering
    \begin{tikzpicture}[]
        \clip (-\u, -\u) rectangle ++(7.5*\ul, 6.5*\ul);

        \fill [swatch_obs] (5.5*\ul, 2.5*\ul) coordinate (xa) rectangle ++(-\ul, -\ul);
        \fill [swatch_obs] (4*\ul, 4*\ul) coordinate (xe) rectangle ++(-\ul, -\ul);        
        \fill [swatch_obs] (1.5*\ul, 4.75*\ul) coordinate (x1) rectangle ++(\ul, \ul);
        \fill [swatch_obs] (-1*\ul, 0) coordinate (t2) -- ++(90:1*\ul) -- ++(0:8*\ul) -- ++(-90:1*\ul) coordinate (t1);
        \path (x1) -- ++(90:3*\ul) coordinate (x0) {};
        \path (x1) -- ++(-90:8*\ul) coordinate (x10) {};

        \path (t1) -- (t2) 
            node (tcol) [pos=0.3, coordinate] {};
        \path (x1) -- (tcol)
            node (xcol1) [pos=2, coordinate] {};
        \draw [trace] (tcol) -- (t2) -- ++(90:\u);
        
        \path (xe) -- (xa)
            node (xae) [coordinate, pos=4] {};
        \path (x1) -- (xe)
            node (xe1) [coordinate, pos=7] {};
        \draw [dotted]
            (x1) edge (xcol1)
            (xae) edge (xe)
            (x1) edge (xe1)
            (x0) edge (x10);

        \node (na) at (xa) [tey pt={shift={(3.5mm, 0)}}{center:$\mx_a$}] {};
        \node at (xa) [prune] {};
        \node (ne) at (xe) [tvy pt={shift={(2mm, 3mm)}}{center:$\mx_e$}] {};
        \node at (xe) [prune, rotate=-20] {};
        \node (n1) at (x1) [tvy pt={shift={(-3mm, -1.5mm)}}{center:$\mnode_{-\mside}$}] {};
        \node [cross pt] at (tcol) {};

        \draw [link]
            (n1) edge (x0);
       
    \end{tikzpicture}
}    
\caption{
The nodes at $\mx_a$ and $\mx_e$ would have been pruned from the longer path when Case 3 of Theorem 1 is considered.
Case 3 is applicable (a) if the path contains $S$-tree nodes as the node $\mnode_{-\mside}$ cannot be pruned. Before $\mnode_{-\mside}$ can be pruned, the trace would have been discarded, or a recursive ang-sec trace would have been called which preserves $\mnode_{-\mside}$.
Case 3 is not applicable (b) if the path contains $T$-tree nodes, causing $\mnode_{-\mside}$ to be prunable.
However, as \rtwop{} is complete, $\mnode_{-\mside}$ will be found if the shortest path passes through it.
}
\label{fig:thm1case3B}
\end{figure}

%% file: alg_run.tex
\begin{algorithm}[!ht]
\begin{algorithmic}[1]
\caption{Main \rtwop{} algorithm.}
\label{alg:run}
\Function{Run}{$\mnode_\mathrm{start}, \mnode_\mathrm{goal}$}
    \State $\mlink \gets$ link from $\mnode_\mathrm{start}$ to $\mnode_\mathrm{goal}$.
    \State Queue $(\mqcast, \mlink)$.
    \While {open-list is not empty}
        \State Poll query ($\mqtype, \mlink$).
        \If {$\mqtype = \mqcast$}   \Comment{Casting query polled.}
            \IfThen {\Call{Caster}{$\mlink$} returns path} {\Return path}
        \Else   \Comment{Tracing query polled.}
            \State Trace from tgt node of $\mlink$.
            % \State $\mlinks_T \gets $ tgt links of $\mlink$.
            % \State $\mside \gets $ side of tgt node of $\mlink$. 
            % \State \Call{Tracer}{$\mside, \{\mlink\}, \mlinks_T$}
        \EndIf
        \State Do Case O1 for all nodes in overlap-buffer.
    \EndWhile
    \State \Return $\{\}$ 
    \Comment{No path.}
\EndFunction
\end{algorithmic}
\end{algorithm}

%% file: alg_caster.tex
\begin{algorithm}[!ht]
\begin{algorithmic}[1]
\caption{Handles casting queries.}
\label{alg:caster}
\Function{Caster}{$\mlink$}
    \State $\mray \gets $ ray from src node of $\mlink$ to tgt node of $\mlink$.
    \If {$\mray$ has line-of-sight }
        \IfThen{\Call{CastReached}{$\mray, \mlink$} returns path}{\Return path}
    \Else
        \State \Call{CastCollided}{$\mray, \mlink$}
    \EndIf
    \State \Return $\{\}$ 
    \Comment{No path.}
\EndFunction
\end{algorithmic}
\end{algorithm}

%% file: alg_tracer.tex
\begin{algorithm}[!ht]
\begin{algorithmic}[1]
\caption{Handles tracing queries.}
\label{alg:tracer}
\Function{Tracer}{$\mtrace=(\mxtrace, \msidetrace, \cdots)$}
    \DoWhile
        \Comment{$\mtrace$ encapsulates a tracing query.}
        \If {traced to src node}
            \State \Break
        \ElsIf {\underline{progression rule} finds no prog. w.r.t. src node}
            \IfThen {queued cast to phantom pt}{ \Break}
        \ElsIf{\Call{TracerProc}{$T, \mtrace$} has no more tgt nodes}
            \State \Break
        \ElsIf {\Call{TracerProc}{$S, \mtrace$} has no more src nodes}
            \State \Break
        \ElsIf {\underline{interrupt rule} queues a tracing query}
            \State \Break
        \ElsIf {\underline{placement rule} has cast to all tgt nodes}
            \State \Break
            \Comment{In the placement rule, casting queries are queue, and/or overlaps (Case O1) are identified at $\mxtrace$ and pushed to the overlap-buffer.}
        \EndIf
        \State $\mxtrace \gets \msidetrace$-side corner of $\mxtrace$.
    \EndDoWhile{$\mxtrace$ not out of map}
\EndFunction
\end{algorithmic}
\end{algorithm}

%% file: alg_casterreached.tex
\begin{algorithm}[!ht]
\begin{algorithmic}[1]
\caption{Handles successful casting queries.}
\label{alg:casterreached}
\Function{CastReached}{$\mray, \mlink$}
    \State $\mnode_S \gets $ src node of $\mlink$.
    \State $\mnode_T \gets $ tgt node of $\mlink$.
    \If {$\mnode_S$ is $\mnvy$ \An $\mnode_T$ is $\mnvy$}
        \State \Return path
    \ElsIf {$\mnode_T$ is $\mntm$ \An no turn. pt. placeable at $\mnode_T$}
        \State Trace from $\mnode_T$.
    \ElsIf {$\mnode_S$ is not ($\mney$ or $\mnvy$) \An $\mnode_T$ is not ($\mney$ or $\mnvy$)}
        \State Queue $(\mqcast,\mlink_T)$ for each tgt link $\mlink_T$ of $\mlink$.
        \State Push $\mnode_T$ into overlap-buffer if overlap (Case O1).
    \ElsIf{$\mnode_S$ is ($\mney$ or $\mnvy$)}
        \State $\mside \gets $ side of $\mnode_T$.
        \State Merge $\mray$ to $-\mside$ sector-ray of $\mlink$.
        \ForEach{tgt link $\mlink_T$ of $\mlink$}
            \State Merge $\mray$ to $\mside$-side sector-ray of $\mlink_T$.
            \State Queue $(\mqcast,\mlink_T)$.
        \EndFor
        \State Do overlap rule (Cases O2, O3, O6) at $\mnode_T$.
        \State Push $\mnode_T$ into overlap-buffer if Case O1.
    \Else   
        \Comment{$\mnode_T$ is ($\mney$ or $\mnvy$)}
        \State Queue $(\mqcast,\mlink_S)$ for src link $\mlink_S$ of $\mlink$.
        \State Do overlap rule (Cases O4, O5, O7) at $\mnode_S$.
        \State Push $\mnode_S$ into overlap-buffer if Case O1.
    \EndIf
    \If{$\mnode_T$ is $\mntm$}
        \Comment{$S$-tree turn. pt. was placed at $\mnode_T$}
        \State Queue $(\mqcast,\mlink_T)$ for each castable tgt link $\mlink_T$ of $\mlink$.
        \State For non-castable tgt links, begin a trace from next corner of $\mnode_T$. 
    \EndIf
    \State \Return $\{\}$ \Comment{No path.}
\EndFunction
\end{algorithmic}
\end{algorithm}

%% file: alg_castercollided.tex
\begin{algorithm}[!ht]
\begin{algorithmic}[1]
\caption{Handles casting queries that collide.}
\label{alg:castercollided}
\Function{CastCollided}{$\mray, \mlink$}
    \State $\mnode_S \gets $ src node of $\mlink$.
    % \State $\mnode_T \gets $ tgt node of $\mlink$.
    \State $\mside_\mmjr \gets $ side of $\mnode_S$.
    \If {$\mnode_S$ is not $\mney$}
        \State $\mlink_{\mmnr,S} \gets \mlink$
        \Comment{Prepare \underline{minor trace}.}
        \State Merge $\mray$ into $\mside_\mmjr$-side sector-ray of $\mlink_{\mmnr,S}$.
        \State $(-\mside_\mmjr)$-sided trace from $(-\mside_\mmjr)$-side corner of collision point.
        % \State $\mlink_{\mmnr,T} \gets $ new link connected to tgt links of $\mlink$.
        % \State $\mx_\mmnr \gets (-\mside_\mmjr)$-side corner of collision point.
        % \State \Call{Tracer}{$-\mside_\mmjr, \mx_\mmnr, \{\mlink_{\mmnr,S}\}, \{\mlink_{\mmnr,T}\}$} 

        \State $\mlink_{\mthird,S} \gets \mlink$
        \Comment{Prepare \underline{third trace}.}
        \State Merge $\mray$ into $\mside_\mmjr$-side sector-ray of $\mlink_{\mthird,S}$.
        \State Create $\mnun$ node and link at $\mnode_S$.
        \State Create $\mnoc$ node and link at $(-\mside_\mmjr)$-side corner of $\mnode_S$.
        \State $\mside_\mmjr$-sided trace from $\mside_\mmjr$-side corner of $\mnode_S$.
        % \State $\mlink_{\mnun,T} \gets $ new link connected to tgt links of $\mlink$.
        % \State Anchor $\mlink_{\mnun,T}$ at new $\mnun$-node at $\mnode_S$.
        % \State $\mlink_{\mnoc,T} \gets $ new link connected to $\mlink_{\mnun,T}$.
        % \State Anchor $\mlink_{\mnoc,T}$ at new $\mnoc$-node at $-\mside_\mmjr$ corner of $\mnode_S$.
        % \State $\mlink_{\mthird,T} \gets $ new link connected to $\mlink_{\mnoc,T}$.
        % \State $\mx_\mthird \gets $ $\mside_\mmjr$-side corner of $\mnode_S$.
        % \State \Call{Tracer}{$\mside_\mmjr, \mx_\mthird, \{\mlink_{\mthird,S}\}, \{\mlink_{\mthird,T}\}$}
    \EndIf
    \State $\mlink_{\mmjr,S} \gets \mlink$
    \Comment{Prepare \underline{major trace}.}
    \State Merge $\mray$ into $(-\mside_\mmjr)$-side sector-ray of $\mlink_{\mmjr,S}$.
    \State $\mside_\mmjr$-sided trace from $\mside_\mmjr$-side corner of $\mnode_S$.
    
    % \State $\mlink_{\mmjr,T} \gets $ new link connected to tgt links of $\mlink$.
    % \State $\mx_\mmjr \gets \mside_\mmjr$-side corner of collision point.
    % \State \Call{Tracer}{$\mside_\mmjr, \mx_\mmjr, \{\mlink_{\mmjr,S}\}, \{\mlink_{\mthird,T}\}$}
\EndFunction
\end{algorithmic}
\end{algorithm}

%% file: alg_tracerproc.tex
\begin{algorithm}[!ht]
\begin{algorithmic}[1]
\caption{Processes trace in one tree direction.}
\label{alg:tracerproc}
\Function{TracerProc}{$\mtdir, \mtrace=(\mxtrace, \msidetrace, \cdots)$}
    \ForEach{$\mtdir$ link $\mlink$ of $\mtrace$}
        % \Comment{Trace has one src ($\mtdir=S$) link and $\ge 1$ tgt ($\mtdir=T$) links}
        \State $\mnode \gets$ $\mtdir$ node of $\mlink$.
        % \If {$\mtdir=T$ and \underline{progression rule} finds no prog. w.r.t. $\mnode$}
        %     \State \Continue 
        \If{$\mtdir=S$ and \underline{angular-sector rule} discards trace}
            \State \Continue
        \ElsIf{$\mnode$ is start or goal node}
            \State \Continue
        \ElsIf{$\msidetrace = $ side of $\mnode$}
            \State Do \underline{pruning rule}.
        \Else   
            \Comment{$\msidetrace \ne$ side of $\mnode$}
            \State Do \underline{ocupied-sector rule}.
        \EndIf
    \EndFor
    \State \Return $\mtrace$
\EndFunction
\end{algorithmic}
\end{algorithm}

%% file: tex_results.tex
\section{Methodology}

The method used is the same as \citep{bib:r2}.
Algorithms are run on benchmarks, which are obtained from \citep{bib:bench}.
Each map in the benchmark contains between a few hundred to several thousand \textit{scenarios}, which are shortest path problems between two points.

As \rtwo{} and \rtwop{} do not pre-process the map and runs on binary occupancy grids, their results are compared with equivalent state-of-the-art algorithms ANYA and \rsp{}.
As such, state-of-the-art algorithms that are not online or do not run on binary occupancy grids, such as Polyanya \citep{bib:polyanya} and Visibility Graphs \citep{bib:vg}, are not compared.

For \rsp{}, the skip, bypass, and block extensions are selected as it is the fastest online configuration.
\rsp{} requires a map to be scaled twice and the start and goal points to be shifted by one unit in both dimensions.
As such, the tested maps are scaled twice, and the chosen algorithms are run on the same scenarios as \rsp{}.
% In the results, the map names will be appended with the label ``\_scale2" for clarity.
% While the names are different as the method in \citep{bib:r2}, the maps and scenarios are the same.

Unlike \rtwo{}, \rtwop{} allows a path to pass through a checkerboard corner. 
A checkerboard corner is located at a vertex where the four diagonally adjacent cells have occupancy states resembling a checkerboard.
The passage through a checkerboard corner simplifies the algorithm by avoiding ambiguity when the starting point is located at a checkerboard corner.
To ensure that the returned paths are correct, the costs of \rtwo{} and \rtwop{} are verified against the visibility graph implementations and other algorithms, and the costs are found to agree.

To test the impact of Cases O6 and O7 of the overlap rule on search time, \rtwop{} is further re-run as the variant ``\rtwop{}N7" with the cases disabled.

The tests are run on Ubuntu 20.04 in Windows Subsystem for Linux 2 (WSL2) and on a single core of an Intel i9-11900H (2.5 GHz), with Turbo-boost disabled. The machine and software is the same as \citep{bib:r2}.
\rtwo{} and \rtwop{} are available at \citep{bib:r2code}.

\input{fig_results}
\section{Results}
\input{tab_average}
\input{tab_points}
In this section, a \textbf{speed-up} is the ratio of an algorithm's search time to \rtwop{}'s search time.
The speed-ups for selected maps are shown in Fig. \ref{fig:results}.
The average search times are  shown in Table \ref{tab:average},
and Table \ref{tab:points} show the average speed-ups with respect to the 3, 10 and 30 turning points.
As passage through checkerboard corners have negligible impact on search times (see below), the costs and number of turning points used for comparisons are based on paths that can pass through checkerboard corners.
Colinear turning points are removed from all results to avoid double counts in the comparisons.

The middle column of Fig. \ref{fig:results} shows the average speed-ups with respect to the number of turning points on the shortest path. 
The right column of Fig. \ref{fig:results} shows the benchmark characteristics by plotting the shortest paths' cost with  the number of turning points.
The correlation between the cost and number of turning points is indicated by $r$, and the ratio of the number of corners to the number of free cells is indicated in $\rho$.

As \rtwo{} and \rtwop{} are exponential in the worst case with respect to the number of collided casts, the algorithms are expected to perform poorly in benchmarks with high $r$.
A high $r$ indicates that paths are likely to turn around more obstacles as they get longer, implying that the maps have highly non-convex obstacles and many disjoint obstacles.
In such maps, collisions are highly likely to occur, and \rtwo{} and \rtwop{} are likely to be slow.

Unlike the other algorithms, \rtwop{} and \rtwop{}N7 allow passage through checkerboard corners.
As such, the shortest paths of the other algorithms are different from \rtwop{} for the maps ``random512-10-1" (25.25\% identical) and ``random512-20-2" (9.38\% identical). 
Coincidentally, \rtwop{} differs significantly from \rtwo{} in the search times for only the two maps (see Table \ref{tab:points}).
The difference is caused by Cases O6 and O7 of the overlap rule instead of the shortcuts through the checkerboard corners.
\rtwop{}N7 performs similarly to \rtwo{} for the two maps (see Table \ref{tab:average}). 
As the only differences between \rtwop{}N7 and \rtwop{} are Cases O6 and O7 of the overlap rule, the difference in speeds between \rtwo{} and \rtwop{} is caused primarily by the cases.
As such, allowing passage through checkerboard corners have negligible impact on the search times for the maps tested.

Cases O6 and O7, which are similar to the G-cost pruning from \citep{bib:dps}, improve search time significantly in maps with many small disjoint obstacles like ``random512-10-1", instead of maps with highly non-convex obstacles  like ``maze512-8-0".
The difference in performance is due to Cases O6 and O7 being able to discard expensive paths only when the path costs are verified.
As queries are able to move around small obstacles faster than highly non-convex ones, path costs can be verified more quickly.
The effect is increased when more collisions occur, causing \rtwop{} to perform significantly faster than \rtwo{} in maps with more disjoint obstacles.

\rtwop{} performs much faster than \rtwo{} in maps with many disjoint obstacles.
While being simpler than \rtwo{}, \rtwop{} has similar performance to \rtwo{} in other maps, preserving the speed advantage that \rtwo{} has over other algorithms when the shortest path is expected to turn around few obstacles.

% For all maps, \rtwop{} tends to perform slightly worse than \rtwo{}. The difference may have been due to implementation differences and the addition of a few constant time checks to simplify \rtwo{}'s algorithm to \rtwop{}.

%% file: fig_results.tex
\begin{figure*}[!ht]
%%% TO LATEX PARSERS: IF THIS IMAGE NEEDS TO BE SHRUNK, PLS INFORM CORRESPONDING AUTHOR.
\centering
\includegraphics[width=\textwidth]{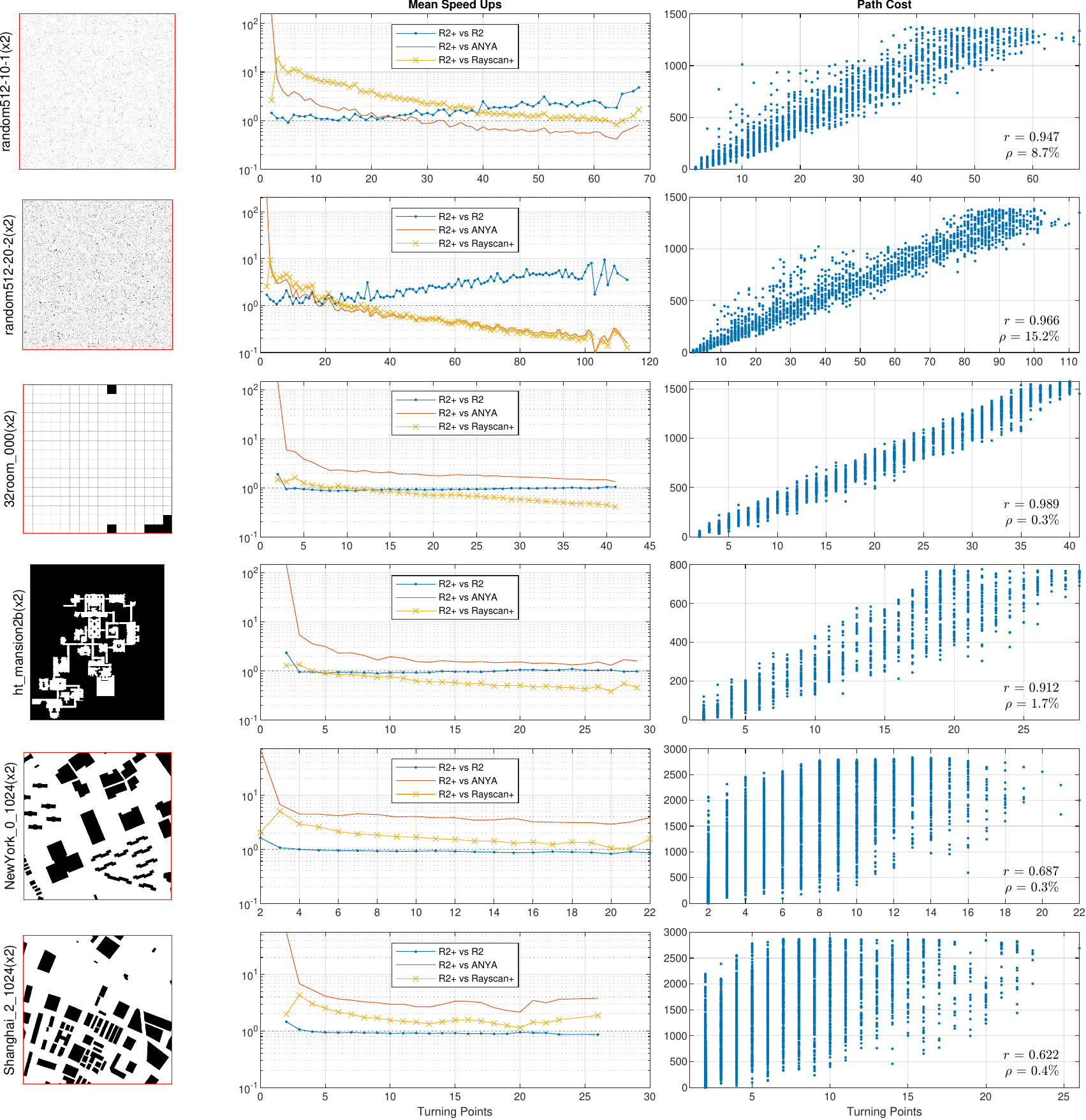}
\caption{
Results for selected maps. All maps are scaled twice, and the start and goal points shifted by one unit to accommodate \rsp{}.
\rtwop{} and \rtwo{} performs well on maps with convex obstacles and few disjoint obstacles, \rtwop{} performs better than \rtwo{} on maps with many disjoint obstacles.
}
\label{fig:results}
\end{figure*}

%% file: tab_average.tex
\begin{table*}[!ht]
\centering
\caption{Benchmark Characteristics and Average Search Times}
\setlength{\tabcolsep}{3pt}
\begin{tabular}{ c c c c c c c c c c }
\toprule
Map & $P$ & $G$ & $r$ & $\rho$ (\%) & \rtwop{} & \rtwop{}N7 & R2 & ANYA & RS+ \\
\hline
bg512/AR0709SR & 13 & 953.3 & 0.608 & 0.144 & 38.463 & 38.321 & 35.177 & 204.542 & 54.128 \\
\hline
bg512/AR0504SR & 22 & 1019.0 & 0.792 & 0.570 & 150.338 & 166.378 & 155.799 & 570.845 & 204.597 \\
\hline
bg512/AR0603SR & 42 & 2228.5 & 0.963 & 1.299 & 771.426 & 851.360 & 846.224 & 1040.011 & 335.797 \\
\hline
da2/ht\_mansion2b & 29 & 776.2 & 0.912 & 1.748 & 302.421 & 345.923 & 324.229 & 471.444 & 152.418 \\
\hline
da2/ht\_0\_hightown & 18 & 1061.9 & 0.908 & 0.876 & 251.374 & 301.591 & 288.901 & 1031.213 & 273.415 \\
\hline
dao/hrt201n & 31 & 905.8 & 0.942 & 2.751 & 427.294 & 477.768 & 442.440 & 634.030 & 193.464 \\
\hline
dao/arena & 5 & 100.5 & 0.428 & 1.315 & 4.289 & 4.215 & 4.403 & 98.121 & 2.750 \\
\hline
maze/maze512-32-0 & 56 & 4722.4 & 0.987 & 0.037 & 287.388 & 287.549 & 292.550 & 904.761 & 132.083 \\
\hline
maze/maze512-16-0 & 145 & 6935.0 & 0.994 & 0.143 & 2398.408 & 2395.561 & 2413.348 & 1708.210 & 369.548 \\
\hline
maze/maze512-8-0 & 205 & 4792.2 & 0.992 & 0.511 & 9942.258 & 9923.029 & 10443.452 & 2836.640 & 895.688 \\
\hline
random/random512-10-1 & 68 & 1372.1 & 0.947 & 8.667 & 13894.789 & 28495.995 & 28860.096 & 9175.201 & 19476.182 \\
\hline
random/random512-20-2 & 113 & 1386.8 & 0.966 & 15.219 & 113042.158 & 481993.141 & 480362.351 & 32668.658 & 29605.503 \\
\hline
room/32room\_000 & 41 & 1579.2 & 0.989 & 0.272 & 1003.133 & 1220.866 & 1021.478 & 1656.847 & 558.525 \\
\hline
room/16room\_000 & 69 & 1477.7 & 0.992 & 1.065 & 4671.431 & 6415.172 & 5411.462 & 3713.668 & 1641.957 \\
\hline
street/Denver\_2\_1024 & 16 & 2835.8 & 0.770 & 0.028 & 96.485 & 102.252 & 91.920 & 910.048 & 416.744 \\
\hline
street/NewYork\_0\_1024 & 22 & 2834.8 & 0.687 & 0.310 & 316.994 & 324.427 & 299.855 & 1273.206 & 511.943 \\
\hline
street/Shanghai\_2\_1024 & 26 & 2885.7 & 0.622 & 0.404 & 508.290 & 541.504 & 491.929 & 1520.395 & 750.569 \\
\hline
street/Shanghai\_0\_1024 & 22 & 2816.5 & 0.511 & 0.258 & 266.808 & 267.314 & 256.614 & 973.980 & 265.440 \\
\hline
street/Sydney\_1\_1024 & 24 & 2844.5 & 0.698 & 0.128 & 159.619 & 164.025 & 152.603 & 958.804 & 368.688 \\
\bottomrule
\multicolumn{10}{p{16.5cm}}{
\footnotesize
All maps are scaled twice and start and goal coordinates shifted by one unit to accommodate \rsp{} (RS+).
$r$ is the correlation coefficient between the number of turning points and the shortest path cost for all scenarios in each map. 
$\rho$ is the ratio of the number of corners to the number of free cells on the map. 
$P$ is the largest number of turning points and $G$ is the largest path cost among all scenarios.
}
\end{tabular}
\label{tab:average}
\end{table*}

%% file: tab_points.tex
\begin{table*}[!ht]
\centering
\caption{Average Speed Ups for 3, 10, and 30 Turning Points}
\setlength{\tabcolsep}{3pt}
\begin{tabular}{ c c c c c | c c c c | c c c c }
\hline
\multirow{2}{*}{Map} & \multicolumn{4}{c | }{3 Turning Pts.} & \multicolumn{4}{c | }{10 Turning Pts.} & \multicolumn{4}{c}{30 Turning Pts.} \\
\cline{2-13}
& $g_{3}$ & R2 & ANYA & RS+ & $g_{10}$ & R2 & ANYA & RS+ & $g_{30}$ & R2 & ANYA & RS+ \\
\hline
bg512/AR0709SR & 418.2 & 0.961 & 8.49 & 2.43 & 663.2 & 0.905 & 3.65 & 1.11 & -- & -- & -- & -- \\
\hline
bg512/AR0504SR & 242.3 & 1.11 & 7.12 & 3.66 & 747.1 & 1.04 & 4.07 & 1.53 & -- & -- & -- & -- \\
\hline
% bg512/AR0014SR & 231.0 & 1.05 & 5.49 & 2.89 & 584.2 & 0.916 & 2.35 & 1.23 & -- & -- & -- & -- \\
% \hline
% bg512/AR0304SR & 283.5 & 1.06 & 5.87 & 2.44 & 744.7 & 0.899 & 2.97 & 0.955 & -- & -- & -- & -- \\
% \hline
% bg512/AR0702SR & 220.5 & 0.977 & 6.26 & 2.08 & 725.5 & 0.902 & 3.65 & 1.09 & -- & -- & -- & -- \\
% \hline
% bg512/AR0205SR & 155.1 & 1.06 & 5.25 & 2.38 & 517.0 & 0.952 & 2.79 & 1.27 & 1246.3 & 1.24 & 2.28 & 0.613 \\
% \hline
% bg512/AR0602SR & 162.3 & 1.08 & 5.93 & 2.26 & 456.6 & 0.967 & 2.13 & 0.943 & 1388.2 & 1.06 & 1.2 & 0.383 \\
% \hline
bg512/AR0603SR & 212.1 & 1.09 & 4.8 & 2.75 & 595.5 & 0.987 & 2.05 & 0.925 & 1645.0 & 1.09 & 1.43 & 0.431 \\
\hline
da2/ht\_mansion2b & 59.4 & 0.954 & 5.46 & 1.35 & 249.1 & 0.944 & 1.99 & 0.79 & -- & -- & -- & -- \\
\hline
da2/ht\_0\_hightown & 134.4 & 0.981 & 5.95 & 2.19 & 584.1 & 1.11 & 4.82 & 1.39 & -- & -- & -- & -- \\
\hline
dao/hrt201n & 81.7 & 1.07 & 5.22 & 1.59 & 285.3 & 1.02 & 1.98 & 0.794 & 848.7 & 0.92 & 1.73 & 0.399 \\
\hline
dao/arena & 69.5 & 0.917 & 9.12 & 0.627 & -- & -- & -- & -- & -- & -- & -- & -- \\
\hline
maze/maze512-32-0 & 206.3 & 0.975 & 6.42 & 0.991 & 807.4 & 1.01 & 4.53 & 0.76 & 2413.4 & 1.02 & 3.68 & 0.536 \\
\hline
maze/maze512-16-0 & 122.2 & 0.964 & 6.5 & 0.952 & 442.5 & 1.01 & 2.57 & 0.829 & 1385.9 & 1.01 & 1.62 & 0.372 \\
\hline
maze/maze512-8-0 & 55.1 & 1.03 & 4.81 & 0.963 & 226.8 & 0.951 & 1.83 & 0.57 & 748.0 & 0.991 & 0.965 & 0.336 \\
\hline
random/random512-10-1 & 42.2 & 1.09 & 7.21 & 18.4 & 264.6 & 1.12 & 2.22 & 7.01 & 791.9 & 1.44 & 0.882 & 2.22 \\
\hline
random/random512-20-2 & 16.6 & 1.3 & 7.05 & 9.03 & 143.8 & 1.11 & 1.52 & 2.5 & 476.6 & 1.87 & 0.817 & 0.932 \\
\hline
room/32room\_000 & 79.0 & 0.979 & 6.23 & 1.4 & 344.8 & 0.911 & 2.31 & 1.03 & 1137.8 & 1.01 & 1.71 & 0.603 \\
\hline
room/16room\_000 & 42.9 & 0.995 & 5.62 & 1.69 & 184.7 & 0.901 & 1.65 & 1.01 & 614.8 & 0.98 & 1.06 & 0.562 \\
\hline
street/Denver\_2\_1024 & 774.3 & 0.962 & 10.8 & 5.85 & 2329.0 & 0.956 & 9 & 4.27 & -- & -- & -- & -- \\
\hline
street/NewYork\_0\_1024 & 865.3 & 1.09 & 6.84 & 5.13 & 1846.0 & 0.95 & 4.14 & 1.72 & -- & -- & -- & -- \\
\hline
street/Shanghai\_2\_1024 & 1025.3 & 1.09 & 7.05 & 4.38 & 1891.7 & 0.936 & 3.06 & 1.6 & -- & -- & -- & -- \\
\hline
street/Shanghai\_0\_1024 & 1371.7 & 1.13 & 7.34 & 2.71 & 1558.5 & 0.961 & 3.21 & 0.95 & -- & -- & -- & -- \\
\hline
street/Sydney\_1\_1024 & 878.0 & 1.06 & 8.98 & 4.32 & 1995.3 & 0.93 & 5.17 & 2.15 & -- & -- & -- & -- \\
\hline
\multicolumn{13}{p{16cm}}{
\footnotesize
Maps are scaled twice and points shifted by one unit to accommodate \rsp{}.
$g_i$ refers to the average path cost for the shortest paths with $i$ turning points.
``R2", ``ANYA" and ``RS+" (\rsp{}) are the speedups of \rtwop{} with respect to the algorithms.
The higher the ratio, the faster \rtwop{} is compared to an algorithm.
}
\end{tabular}
\label{tab:points}
\end{table*}

%% file: tex_conclusion.tex
\section{Conclusion}
In this work, \rtwo{}, a vector-based any-angle path planner, is evolved into \rtwop{}.
Novel mechanisms are introduced in \rtwop{} to simplify the algorithm, and allow the algorithm to terminate.
\rtwop{} prevents chases from occurring by superseding ad hoc points in \rtwo{} with a short, recursive angular-sector trace from target nodes.

\rtwo{} and \rtwop{} are able to outperform state-of-the-art algorithms like ANYA and \rsp{} when paths are expected to have few turning points.
\rtwo{} and \rtwop{} are fast due to delayed line-of-sight checks to expand the most promising turning points, which are points that deviate the least from the straight line between the start and goal points.

While fast when the shortest paths are expected to have few turning points, \rtwo{} and \rtwop{} are exponential in the worst case with respect to collided line-of-sight checks in the worst case. 
To improve average search time, \rtwo{} discards paths that have expensive nodes that cannot be pruned.
\rtwop{} improves upon \rtwo{} by discarding paths that intersect cheaper paths, allowing \rtwop{} to outperform \rtwo{} in maps with many disjoint obstacles.

\rtwop{} is a superior algorithm to \rtwo{}, and supersedes \rtwo{}.
\rtwop{} is terminable, and simpler to implement than \rtwo{}.
\rtwop{} outperforms \rtwo{} in maps with many disjoint obstacles, while preserving the performance of \rtwo{} in other maps.
Future works may investigate ways to improve the speed of \rtwop{} in maps with highly non-convex obstacles, and improve the algorithm's complexity with respect to collided casts.

%% file: tex_vitae.tex
\section{Vitae}
\includegraphics[width=1in,height=1.25in,clip,keepaspectratio]{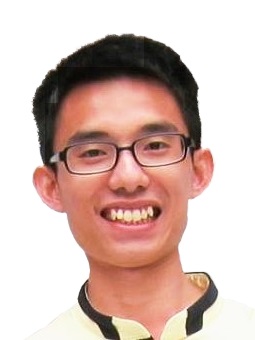}
\textbf{Yan Kai, Lai}
received his B.Eng. degree in Electrical Engineering from National University of Singapore, Singapore, in 2019. 
He is currently working toward the Ph.D. degree in robotics path planning with the
National University of Singapore.
His research interests include path-planning and simultaneous localisation and mapping in mobile ground robots.

\includegraphics[width=1in,height=1.25in,clip,keepaspectratio]{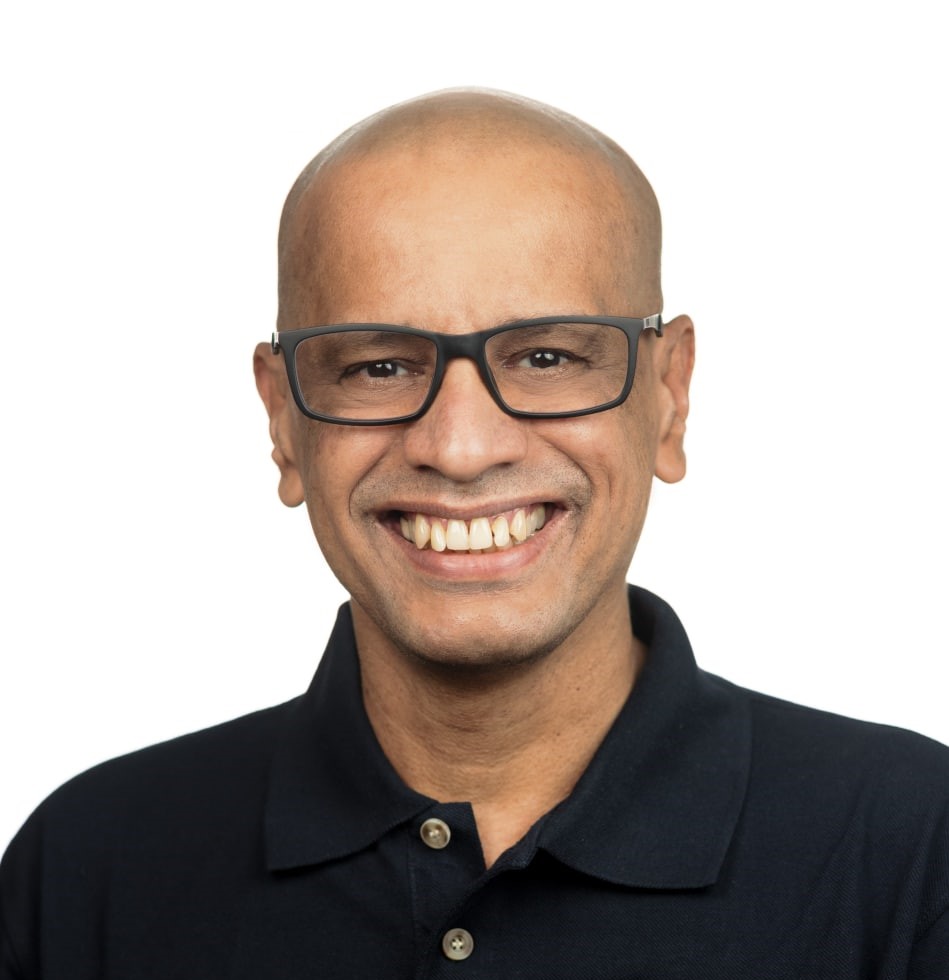}\textbf{Prahlad Vadakkepat}
(M’00–SM’05) received the B.Eng. degree (First Class Hons.) in Electrical Engineering from the Calicut University, Kerala, India, in 1986, and the M.Tech. and Ph.D. degrees from the Indian Institute of Technology Madras, Chennai, India, in 1989 and 1996, respectively. 

Since 1999, he is with the National University of Singapore, Singapore, where he is an Associate Professor. He is the Founder Secretary of the Federation of International Robot-Soccer Association and served as General Secretary from 2000 to 2016. His current research interests include robotics, AI, humanoid robotics, distributed robotics systems and frugal innovation. 

Prof. Vadakkepat was a recipient of several international prizes for his teams of humanoid robots and soccer robots. He was the Editor-in-Chief of the Springer Reference Book on Humanoid Robotics.

\includegraphics[width=1in,height=1.25in,clip,keepaspectratio]{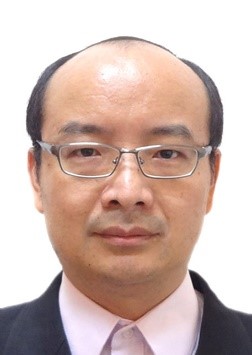} \textbf{Cheng, Xiang} 
received the B.S. degree from Fudan University in 1991; M.S. degree from the Institute of Mechanics, Chinese Academy of Sciences in 1994; and M.S. and Ph.D. degrees in electrical engineering from Yale University in 1995 and 2000, respectively. He is an Associate Professor and the area director of Control, Intelligent Systems and Robotics in the Department of Electrical and Computer Engineering at the National University of Singapore. His research interests include artificial  intelligence, intelligent control, and systems biology.

%% file: tex_supp.tex
\onecolumn
% \section{Supplementary Material: \rtwop{}}
\begin{center}
\textbf{\large Supplementary Material: 
\rtwop{}: Vector-Based Optimal Any-Angle 2D Path Planning with Non-convex Obstacles}
\end{center}
\setcounter{algorithm}{0}
\setcounter{equation}{0}
\setcounter{figure}{0}
\setcounter{table}{0}
\setcounter{page}{1}
\setcounter{section}{0}
\setcounter{paragraph}{0}
\makeatletter
\renewcommand\paragraph{\@startsection{paragraph}{4}{\z@}
                         {2ex \@plus .1ex \@minus .2ex}%
                         {.1ex \@plus .0ex}%1.5ex \@plus .2ex}%
                         {\normalfont\normalsize\itshape}}
\setcounter{secnumdepth}{4}% Number up to paragraphs
\renewcommand{\thealgorithm}{\thesubsubsection} %.\arabic{algorithm}}
% \makeatletter
% \@addtoreset{algorithm}{subsection}% algorithm counter resets every chapter
\makeatother

\tableofcontents

\section{Introduction}
The supplementary material attempts to describe the implementation details of \rtwop{}.
The programming classes of object-oriented programming and any conditions of implementation are described.

\subsection{Pseudocode Convention}
The pseudocode is Python-like and all objects are passed by reference.
Class objects are usually represented by tuples $(\cdot)$. If an object is a collection of other objects, they are usually represented by ordered sets $\{\cdot\}$ unless otherwise specified.
To access a member of an object, the emboldened period operator ($\mdot$) is used.
If an object is an ordered set, its elements can be accessed with the square bracket operator $[\cdot]$, with 1 begin the first element.

The reader may need to keep in mind any memory invalidation that can occur, such as iterator invalidations, when an object is removed from a set in the pseudocode.
The reader is encouraged to introduce additional members to objects described below to cache repeated calculations.

%%%%%%%%%%%%%%%%%%%%%%%%%%%% COMMON NOMENCLATURE %%%%%%%%%%%%%%%%%%%%%%%%%%%%%%%%%%%%%%%%%%%%%
\subsection{Common Nomenclature}
\rtwop{} is an any-angle 2D path planner that finds the shortest path between two points, the \textbf{start} and \textbf{goal} points.
\rtwop{} searches from the start point, and returns the shortest path if it exists.
The shortest path is a set of coordinates that begins from the goal point and ends at the start point. Intermediate points are turning points of the any-angle path, and may include colinear points.

\rtwop{} relies on two trees of nodes in the search process. 
The \textbf{source-tree} ($S$-tree) is rooted at the start node, which lies on the start point, while the \textbf{target-tree} ($T$-tree) is rooted at the goal node, which lies on the goal point.
Each tree branches from its root node, and is ``joined" to each other at their leaf nodes.

A link is a branch on the tree that connects two nodes. 
For efficient memory management, a link points to only one node (\textbf{anchored} at a node).
A link can be connected to other links anchored at other nodes, indirectly connecting the nodes.
While a node forms a vertex of an examined path, a link forms a segment of the path.
By storing information specific to a path being expanded, a link reduces duplicate line-of-sight checks between nodes of overlapping paths.

A \textbf{query} is used to describe an intermediate expansion. The query has an associated path cost estimate and can be queued into the open-list.
In A*, a query can be thought of as an iteration, which expands a node and examines the neighboring nodes.
In \rtwop{}, a query expands a link, performing a line-of-sight check over the link (\textbf{casting query}, or \textbf{cast}), or traces an obstacle contour beginning from the link (\textbf{tracing query}, or \textbf{trace}).
The link that is being expanded is anchored on a leaf node (\textbf{leaf link}) of the $S$-tree or $T$-tree.
A query and its expanded leaf link point to each other. As such, a query can be found between a pair of connected leaf nodes.

The \textbf{source-direction} refers to the direction leading to the start node along the path formed by the links and nodes, while the \textbf{target-direction} refers to the direction leading to the goal node. 
The terms are applied to objects when describing their location in relation to other objects.
For example, a link indirectly connects two nodes. The node leading to the start node is called the \textbf{source node} of the link, while the node leading to the goal node is called the \textbf{target node}.
Let $\mtdir \in \{S,T\}$ be the tree-direction where $S=-1$ and $T=1$ denote the source and target-direction respectively.
Nodes in the source tree (\textbf{source-tree nodes}) lie in the source direction of a query, while nodes in the the target tree (\textbf{target-tree nodes}) lie in the target direction of a query.
A \textbf{parent} node or link of an object lies closer to the root node of a tree than the object. 
For example, a parent node or link of a source-tree node lies in the source-direction of the source-tree node.
A \textbf{child} node lies further away from the root node.
For example, a child node or link of a target-tree node lies in the source-direction of the target-tree node.

When a cast collides with an obstacle, a \textbf{left} and \textbf{right} trace occurs from the left and right side of the collision point respectively. 
The traces trace the obstacle's contour, placing nodes along the contours. 
A node has a side that is the same as the trace that found it -- a left-sided node is placed by a left trace, etc.
Let $\mside \in \{L,R\}$ be the side where $L=-1$ and $R=1$ denote the left and right side respectively.

% Fig. \ref{suppfig:tree} illustrates the trees, nodes, links, and queries. 

%%%%%%%%%%%%%%%%%%%%%%%%%%%% DATA STRUCTURES AND METHODS %%%%%%%%%%%%%%%%%%%%%%%%%%%%%%%%%%%%%%%%%%%%%

\section{Data Structures and Methods}
This section describes the objects, or data structures, used in \rtwop{}.
The subsections contain functions which are sorted into three categories.
The first category contains helper functions, denoted by a single letter, to access members or owners of an object.
The second category contains a constructor or initializer, which is prefixed with \textit{Get} or \textit{Create}.
The third category contains functions that interface with other objects. 
The material includes only noteworthy functions, and the functions are non-exhaustive.

%%%%%%%%%%%%%%%%%%%%%%%%%%%% POSITION & BEST %%%%%%%%%%%%%%%%%%%%%%%%%%%%%%%%%%%%%%%%%%%%%
\subsection{Position and Best Objects}
A position object $\mpos$ describes the best cost and direction to reach a position, and owns any nodes placed at the location.
It is described by
\begin{equation}
\mpos = (\mx, \mnodes, \mbest_S, \mbest_T),
\end{equation}
where $\mx$ is the pair of coordinates, and $\mnodes$ is a set containing nodes.
$\mbest_S$ and $\mbest_T$ contain information about the best cost and direction to $\mx$ from the source and target direction respectively, and are described by
\begin{equation}
    \mbest = (\mcost, \mnodebest, \mxpar).
\end{equation}
$\mcost$ is the smallest cost known so far to reach the current position $\mpos$. 
For $\mbest_S$, the cost is cost-to-come, and for $\mbest_T$, the cost is cost-to-go.
$\mnodebest$ indicates the node $\mnode$ at the current location which has the shortest, unobstructed path to the root node. The root node is the start or goal node for $\mbest_S$ or $\mbest_T$ respectively.
For $\mbest_S$, $\mxbest$ is the location of the source node along the path. 
For $\mbest_T$, $\mxbest$ is the location of the target node along the path.

%%%%%%%%%%%%%%%%%%%%%%%%%%%%%%%% FBEST %%%%%%%%%%%%%%%%%%%%%%%%%%%%%%%%%%
\subsubsection{\textproc{$\fbest$}: \textit{Gets a Best Member Object of a Position Object}}
The helper function $\fbest$ (Alg. \ref{suppalg:fbest}) returns $\mbest_S$ or $\mbest_T$ depending on the tree-direction $\mtdir\in\{S,T\}$.
\begin{algorithm}[!ht]
\begin{algorithmic}[1]
\caption{Get source Best or target Best property of Position object}
\label{suppalg:fbest}
\Function{$\fbest$}{$\mtdir, \mpos$}
    \State \Return $\mpos\mdot\mbest_S$ \textbf{if} $\mtdir = S$ \textbf{else} $\mpos\mdot\mbest_T$
\EndFunction
\end{algorithmic}
\end{algorithm}

%%%%%%%%%%%%%%%%%%%%%%%%%%%%%%%% GETPOS %%%%%%%%%%%%%%%%%%%%%%%%%%%%%%%%%%
\subsubsection{\textproc{GetPos}: \textit{Retrieves or Constructs a Position Object}}
The function \textproc{GetPos} (Alg. \ref{suppalg:getpos}) returns $\mpos$ if it exists at a corner at $\mx$. Otherwise, it creates a new $\mpos$ at $\mx$ and returns it.

\begin{algorithm}[!ht]
\begin{algorithmic}[1]
\caption{Retrieve existing Position object or create it if it does not exist.}
\label{suppalg:getpos}
\Function{GetPos}{$\mx$}
    \State Try getting $\mpos = (\mx, \cdots )$ from a corner at $\mx$. 
    \If {$\mpos$ does not exist}
        \State Create $ \mpos = (\mx, \{\}, \infty, \varnothing, \infty, \varnothing)$ \Comment{$\mx_g$ and $\mx_h$ can be initialised to any value.}
    \EndIf
    \State \Return $\mpos$
\EndFunction
\end{algorithmic}
\end{algorithm}

%%%%%%%%%%%%%%%%%%%%%%%%%%%% NODE %%%%%%%%%%%%%%%%%%%%%%%%%%%%%%%%%%%%%%%%%%%%%
\subsection{Node and Node Types}
A node $\mnode$ is described by
\begin{equation}
\mnode = ( \mntype, \msidenode, \mtdirnode, \mlinks_\mnode ),
\end{equation}
which has a side $\msidenode \in \{L, R\}$ that has the same side as the trace that placed it. 
The node is a \textbf{source-tree node} or a \textbf{target-tree node} if $\mtdirnode = S$ or $\mtdirnode = T$ respectively. 
$\mlinks_\mnode$ refers to an unordered set of links anchored at the node $\mnode$. 

The node types are $\mntype \in \{ \mnvy, \mnvu, \mney, \mneu, \mnun, \mnoc, \mntm \}$. 
% The node types are described in Table \ref{supptab:nodetypes1}. 
Table \ref{supptab:nodetypes2} indicate if two nodes can be indirectly connected with a link. 
The connections are derived from the algorithm's logic, and \rtwop{} does not follow any explicit rule to connect nodes.

Any node with \textbf{cumulative visibility} to another node has an unobstructed path between them. As only the cumulative visibility to the start node or goal node is important, the definition of cumulative visibility is overloaded in the text for brevity --
if a source-tree node or target-tree node has \textbf{cumulative visibility}, then the node has cumulative visibility to the start node or goal node respectively.

% \begin{table*}[!ht]
% \centering
% \caption{Node}
% \begin{tabular}{ c | c c p{12.5cm} }
% $\mntype$ & Tree & Corner & Description \\
% \hline
% $\mnvy$ & $S,T$ & + & An $\mnvy$ node has cumulative visibility to the start node or goal node if the $\mnvy$ node lies in the source or target tree respectively. Cost-to-come or cost-to-go can be determined respectively. \\
% \hline
% $\mnvu$ & $S,T$ & $+$ & Cumulative visibility to any node is unknown. \\
% \hline
% $\mney$ & $S,T$ & $+$ & Same as $\mnvy$, but is more expensive than the minimum cost-to-come or cost-to-go at the corner. \\
% \hline
% $\mneu$ & $S$ & $+$ & Same as $\mnvu$, but has at least one $\mney$ ancestor.\\
% \hline
% $\mnun$ & $T$ & $\pm$ & An unreachable target node placed at the beginning of third-traces or recursive angular-sector traces. \\
% \hline
% $\mnoc$ & $T$ & $\pm$ & A target node that is placed by recursive occupied-sector rule for target nodes. \\
% \hline
% $\mnph$ & $T$ & $-$ & A phantom point. \\
% \hline
% $\mntm$ & $T$ & $\pm$ & A target node that is placed whenever a trace is interrupted for queuing or recursive traces. \\
% \multicolumn{4}{p{16cm}}{
% \footnotesize
% The \textit{Tree} column indicates which tree a node of type $\mntype$ can lie in. The node can lie in the source or target tree if \textit{tree} is $S$ or $T$ respectively. The \textit{Corner} column indicates if a $\mntype$ node can lie on a convex ($+$), non-convex ($-$) or both ($\pm$) corners.
% }
% \end{tabular}
% \label{supptab:nodetypes1}
% \end{table*}

\begin{table*}[!ht]
\centering
\caption{Node Types and Possible Parent Node Types}
\begin{tabular}{ c | c c c c c c c c | p{7cm} }
\diagbox{$\mntype$}{$\mntype_\mtdir$} & $\mnvy$ & $\mnvu$ & $\mney$ & $\mneu$ & $\mnun$ & $\mnoc$ & $\mntm$ & \multirow{9}{7cm}{
\footnotesize
Alphabet(s) $S$ and/or $T$ are placed if a node of type $\mntype$ and a parent node of type $\mntype_\mtdir$ can be connected to each other with a link. $S$ is placed if both nodes are in the source tree and the nodes can be connected, $T$ if both nodes are in the target tree and the nodes can be connected. For example, a $\mnvy$ node and a parent $\mnvy$ node can connect regardless of the tree they lie in, and a $\mnvu$ node and a parent $\mnun$ node can be connected only if they are in the target tree. The connections are derived from the algorithm, and \rtwop{} does not follow any explicit connection rule.
}
\\
\cline{1-8}
$\mnvy$ & $S,T$ & -- & -- & -- & -- & -- & -- & -- & \\
\cline{1-8}
$\mnvu$ & $S,T$ & $S,T$ & -- & -- & $T$ & $T$ & $T$ & \\
\cline{1-8}
$\mney$ & $S,T$ & -- & $S,T$ & -- & -- & -- & -- &\\
\cline{1-8}
$\mneu$ & -- & -- & $S$ & $S$ & -- & -- & -- & \\
\cline{1-8}
$\mnun$ & $T$ & $T$ & -- & -- & -- & -- & $T$  &\\
\cline{1-8}
$\mnoc$ & $T$ & $T$ & $T$ & -- & $T$ & -- & $T$  &\\
\cline{1-8}
% $\mnph$ & $T$ & $T$ & $T$ & -- & $T$ & $T$ & $T$ & $T$ & \\
\cline{1-8}
$\mntm$ & $T$ & $T$ & $T$ & -- & $T$ & $T$ & $T$ & \\

\end{tabular}
\label{supptab:nodetypes2}
\end{table*}

%%%%%%%%%%%%%%%%%%%%%%%%%%%%%%%%% ALG: FPOS %%%%%%%%%%%%%%%%%%%%%%%%%%%%%%%%%%%%%
\subsubsection{\textproc{$\fpos$}: \textit{Gets Position Object that Owns a Node}}
\textproc{$\fpos$} (Alg. \ref{suppalg:fpos}) returns the position object owning the node $\mnode$ .
\begin{algorithm}[!ht]
\begin{algorithmic}[1]
\caption{Get Position object}
\label{suppalg:fpos}
\Function{$\fpos$}{$\mnode$}
    \State \Return $\mpos$ where $\mnode \in \mpos\mdot\mnodes$
\EndFunction
\end{algorithmic}
\end{algorithm}

%%%%%%%%%%%%%%%%%%%%%%%%%%%%%%%%% ALG: FX %%%%%%%%%%%%%%%%%%%%%%%%%%%%%%%%%%%%%
\subsubsection{\textproc{$\fx$}: \textit{Gets Coordinates of a Node}}
\textproc{$\fx$} (Alg. \ref{suppalg:fx}) returns the coordinates of $\mnode$.
\begin{algorithm}[!ht]
\begin{algorithmic}[1]
\caption{Get coordinates of a node}
\label{suppalg:fx}
\Function{$\fx$}{$\mnode$}
    \State \Return \Call{$\fpos$}{$\mnode$}$\mdot\mx$
\EndFunction
\end{algorithmic}
\end{algorithm}

%%%%%%%%%%%%%%%%%%%%%%%%%%%%%%%%% ALG: GETNODE %%%%%%%%%%%%%%%%%%%%%%%%%%%%%%%%%%%%%
\subsubsection{\textproc{GetNode}: \textit{Retrieves or Constructs a Node}}
The function \textproc{GetNode} (Alg. \ref{suppalg:getnode}) gets a matching node with type $\mntype$, side $\mside_\mnode$ and tree $\mtdir_\mnode$ from a corner at $\mx$ if it exists. 
Otherwise, a new node is created.
\begin{algorithm}[!ht]
\begin{algorithmic}[1]
\caption{Retrieve existing node or create it if it does not exist}
\label{suppalg:getnode}
\Function{GetNode}{$\mx$, $\mntype$, $\mside_\mnode$, $\mtdir_\mnode$}
    \State $\mpos \gets $ \Call{GetPos}{$\mx$}
    \State Try finding $\mnode = (\mntype, \mside_\mnode, \mtdir_\mnode, \cdots)$ from $\mpos \mdot \mnodes$.
    \If{ $\mnode$ does not exist}
        \State Create $\mnode = (\mntype, \mside_\mnode, \mtdir_\mnode, \{\})$
    \EndIf
    \State \Return $\mnode$
\EndFunction
\end{algorithmic}
\end{algorithm}

%%%%%%%%%%%%%%%%%%%%%%%%%%%% LINK %%%%%%%%%%%%%%%%%%%%%%%%%%%%%%%%%%%%%%%%%%%%%
\subsection{Link}
A link $\mlink$ is described by
\begin{equation}
\mlink = ( \mlinks_S, \mlinks_T, \mcost, \mray_L, \mray_R ),
\end{equation}
where $\mlinks_S$ and $\mlinks_T$ are unordered sets containing links in the source and target-direction of the link respectively. 
$\mcost$ is the cost-to-come if the link's anchored node is a source-tree node, or cost-to-go if the anchored node is a target-tree node. 
$\mray_L$ is the left sector-ray, and $\mray_R$ is the right sector-ray.

%%%%%%%%%%%%%%%%%%%%%%%%%%%%%%%%% ALG: FX %%%%%%%%%%%%%%%%%%%%%%%%%%%%%%%%%%%%%
\newpage
\subsubsection{\textproc{$\fnode$}: \textit{Gets Anchored Node of a Link}}
\textproc{$\fnode$} (Alg. \ref{suppalg:fnode}) returns the anchored node of a link $\mlink$.
\begin{algorithm}[!ht]
\begin{algorithmic}[1]
\caption{Get the anchored node of a link.}
\label{suppalg:fnode}
\Function{$\fnode$}{$\mlink$}
    \State \Return $\mnode$ where $\mlink \in \mnode\mdot\mlinks_\mnode$
\EndFunction
\end{algorithmic}
\end{algorithm}

%%%%%%%%%%%%%%%%%%%%%%%%%%%%%%%%% ALG: FRAY %%%%%%%%%%%%%%%%%%%%%%%%%%%%%%%%%%%%%
\subsubsection{\textproc{$\fray$}: \textit{Gets Sector-ray of a Link}}
\textproc{$\fray$} (Alg. \ref{suppalg:fray}) returns the $\mside$-sided ($\mside\in\{L, R\}$) sector-ray of a link $\mlink$.
\begin{algorithm}[!ht]
\begin{algorithmic}[1]
\caption{Get a sector-ray of a link.}
\label{suppalg:fray}
\Function{$\fray$}{$\mside, \mlink$}
    \State \Return $\mlink\mdot\mray_L$ \textbf{if} $\mside = L$ \textbf{else} $\mlink\mdot\mray_R$
\EndFunction
\end{algorithmic}
\end{algorithm}

%%%%%%%%%%%%%%%%%%%%%%%%%%%%%%%%% ALG: FLINKS %%%%%%%%%%%%%%%%%%%%%%%%%%%%%%%%%%%%%
\subsubsection{\textproc{$\flinks$}: \textit{Gets a Set of Connected Links}}
\textproc{$\flinks$} (Alg. \ref{suppalg:flinks}) returns the set of $\mtdir$-direction links ($\mtdir \in \{S,T\}$) of a link $\mlink$.
\begin{algorithm}[!ht]
\begin{algorithmic}[1]
\caption{Get the set of source or target links of a link.}
\label{suppalg:flinks}
\Function{$\flinks$}{$\mtdir, \mlink$}
    \State \Return $\mlink\mdot\mlinks_L$ \textbf{if} $\mside = L$ \textbf{else} $\mlink\mdot\mlinks_R$
\EndFunction
\end{algorithmic}
\end{algorithm}

%%%%%%%%%%%%%%%%%%%%%%%%%%%%%%%%% ALG: FLINK %%%%%%%%%%%%%%%%%%%%%%%%%%%%%%%%%%%%%
\subsubsection{\textproc{$\flink$}: \textit{Gets a Connected Link}}
\textproc{$\flink$} (Alg. \ref{suppalg:flink}) returns \textit{any one} link in the $\mtdir$-direction. \textproc{$\flink$}$(\cdot)$ is commonly used to access the parent node of a link. 
As most links only point to one parent link, the function is also commonly used.
\begin{algorithm}[!ht]
\begin{algorithmic}[1]
\caption{Get any one source or target link of a link.}
\label{suppalg:flink}
\Function{$\flink$}{$\mtdir, \mlink$}
    \State Any one link in \Call{$\flinks$}{$\mtdir, \mlink$}
\EndFunction
\end{algorithmic}
\end{algorithm}

%%%%%%%%%%%%%%%%%%%%%%%%%%%%%%%%% ALG: FQUERY %%%%%%%%%%%%%%%%%%%%%%%%%%%%%%%%%%%%%
\newpage
\subsubsection{\textproc{$\fquery$}: \textit{Gets an associated Queued Query}}
\textproc{$\fquery$} returns a queued query (see Sec. \ref{suppsec:queuedquery}) that is associated with a link $\mlink$, if any.
\begin{algorithm}[!ht]
\begin{algorithmic}[1]
\caption{Get queued query of a link}
\label{suppalg:fquery}
\Function{$\fquery$}{$\mlink$}
    \State \Return $\mquery$ \textbf{if} there exists a $\mquery$ such that $\mlink = \mquery\mdot\mlink_\mquery$ \textbf{else} $\varnothing$
\EndFunction
\end{algorithmic}
\end{algorithm}

%%%%%%%%%%%%%%%%%%%%%%%%%%%%%%%%% ALG: CREATELINK %%%%%%%%%%%%%%%%%%%%%%%%%%%%%%%%%%%%%
\subsubsection{\textproc{CreateLink}: \textit{Creates a Link}}
The function \textproc{CreateLink} (Alg. \ref{suppalg:createlink}) creates a new link object, anchors it at a node $\mnode$, and returns it.
\begin{algorithm}[!ht]
\begin{algorithmic}[1]
\caption{Create a new link.}
\label{suppalg:createlink}
\Function{CreateLink}{$\mnode$}
    \State Create $\mlink = (\{\}, \{\}, \infty, \varnothing, \varnothing)$
    \State \Call{Anchor}{$\mlink, \mnode$}
    \State \Return $\mlink$
\EndFunction
\end{algorithmic}
\end{algorithm}

%%%%%%%%%%%%%%%%%%%%%%%%%%%%%%%%% ALG: ANCHOR  %%%%%%%%%%%%%%%%%%%%%%%%%%%%%%%%%%%%%
\subsubsection{\textproc{Anchor}: \textit{Anchors a link at a Node}}
The \textproc{Anchor} function (Alg. \ref{suppalg:anchor}) re-anchors a link $\mlink$ from its anchor node to a new node $\mnode_\mnew$. 
\vspace{-0.2cm}
\begin{algorithm}[!ht]
\begin{algorithmic}[1]
\caption{Anchor link to a node}
\label{suppalg:anchor}
\Function{Anchor}{$\mlink$, $\mnode_\mnew$} 
    \IfThen{$\fnode(\mlink) \ne \varnothing$} {Remove $\mlink$ from $\fnode(\mlink) \mdot \mlinks_\mnode$}
    \State Add $\mlink$ to $\mnode_\mnew \mdot \mlinks_\mnode$
\EndFunction
\end{algorithmic}
\end{algorithm}

%%%%%%%%%%%%%%%%%%%%%%%%%%%%%%%%% ALG: CONNECT %%%%%%%%%%%%%%%%%%%%%%%%%%%%%%%%%%%%%
\subsubsection{\textproc{Connect}: \textit{Connects Two Links}}
The \textproc{Connect} function  (Alg. \ref{suppalg:connect}) connects two links by adding pointers to each other.
\vspace{-0.2cm}
\begin{algorithm}[!ht]
\begin{algorithmic}[1]
\caption{Connect two links.}
\label{suppalg:connect}
\Function{Connect}{$\mtdir$, $\mlink$, $\mlink_\mtdir$}
    \State Add $\mlink_\mtdir$ to $\flinks_\mtdir(\mlink)$ 
    \State Add $\mlink$ to $\flinks_{-\mtdir}(\mlink_\mtdir)$ 
\EndFunction
\end{algorithmic}
\end{algorithm}

%%%%%%%%%%%%%%%%%%%%%%%%%%%%%%%%% ALG: DISCONNECT %%%%%%%%%%%%%%%%%%%%%%%%%%%%%%%%%%%%%
\newpage
\subsubsection{\textproc{Disconnect}: \textit{Disconnects Two Links}}
The \textproc{Disconnect} function  (Alg. \ref{suppalg:disconnect}) disconnects two links by removing their pointers from each other.
\begin{algorithm}[!ht]
\begin{algorithmic}[1]
\caption{Disconnect two links.}
\label{suppalg:disconnect}
\Function{Disconnect}{$\mtdir$, $\mlink$, $\mlink_\mtdir$}
    \State Remove $\mlink_\mtdir$ from $\flinks_\mtdir(\mlink)$ 
    \State Remove $\mlink$ from $\flinks_{-\mtdir}(\mlink_\mtdir)$ 
\EndFunction
\end{algorithmic}
\end{algorithm}

%%%%%%%%%%%%%%%%%%%%%%%%%%%%%%%%% ALG: MERGERAY  %%%%%%%%%%%%%%%%%%%%%%%%%%%%%%%%%%%%%
\subsubsection{\textproc{MergeRay}: \textit{Merges a Sector-Ray into a Link}}
The \textproc{MergeRay} function (Alg. \ref{suppalg:mergeray}) replaces the $\mside$-side ray of a link $\mlink$ with $\mray_\mnew$ if the resulting angular-sector becomes smaller.
\begin{algorithm}[!ht]
\begin{algorithmic}[1]
\caption{Merge Ray if resulting angular-sector becomes smaller}
\label{suppalg:mergeray}
\Function{MergeRay}{$\mside, \mlink, \mray_\mnew$} 
    \State $\mray_\mathrm{old} \gets \fray(\mside, \mlink)$
    \If {$\mray_\mathrm{old} = \varnothing$}
        \State $\fray(\mside, \mlink) \gets \mray_\mnew$ \Comment{Replace the ray}
    \Else
        \State $\mv_\mathrm{rayOld} \gets \mray_\mathrm{old}\mdot\mx_T - \mray_\mathrm{old}\mdot\mx_S$
        \State $\mv_\mathrm{rayNew} \gets \mray_\mnew\mdot\mx_T - \mray_\mnew\mdot\mx_S$
        \If {$\mside(\mv_\mathrm{rayOld}, \mv_\mathrm{rayNew}) \ge 0$}        
            \State $\fray(\mside, \mlink) \gets \mray_\mnew$ \Comment{Replace the ray}
        \EndIf
    \EndIf
\EndFunction
\end{algorithmic}
\end{algorithm}

%%%%%%%%%%%%%%%%%%%%%%%%%%%%%%%%% ALG: MINCOST %%%%%%%%%%%%%%%%%%%%%%%%%%%%%%%%%%%%%
\subsubsection{\textproc{MinCost}: \textit{Returns the Minimum Cost of Connect Links}}
The function \textproc{MinCost} (Alg. \ref{suppalg:mincost}) finds and returns the minimum cost of a link's ($\mlink$) source or target ($\mtdir \in \{S, T\}$) links. 
The minimum cost does not include the length of the link.
\rtwop{} guarantees that all $\mtdir$ links lie on the same tree when \textproc{MinCost} is used.
\begin{algorithm}[!ht]
\begin{algorithmic}[1]
\caption{Find minimum cost of source or target links.}
\label{suppalg:mincost}
\Function{MinCost}{$\mtdir, \mlink$}
    \State $c_{\min} \gets \infty$
    \For {$\mlink_\mtdir \in \flinks(\mtdir, \mlink)$}
        \If{$c_{\min} > \mlink_\mtdir \mdot c$}
            \State $c_{\min} \gets \mlink_\mtdir \mdot c$
        \EndIf
    \EndFor
    \State \Return $c_{\min}$
\EndFunction
\end{algorithmic}
\end{algorithm}

%%%%%%%%%%%%%%%%%%%%%%%%%%%%%%%%% ALG: COST  %%%%%%%%%%%%%%%%%%%%%%%%%%%%%%%%%%%%%
\subsubsection{\textproc{Cost}: \textit{Returns Cost of a Link}}
The function \textproc{Cost}  (Alg. \ref{suppalg:cost}) finds and reutrns the cost of a link $\mlink$ by adding the minimum cost of the parent links to the length of the link, which is between the link's anchored node and its parent node.
\rtwop{} guarantees that all links lie on the same tree when \textproc{Cost} is used. If $\mlink$ is anchored on an $S$-tree node or $T$-tree node, the cost is cost-to-come or cost-to-go respectively.
\begin{algorithm}[!ht]
\begin{algorithmic}[1]
\caption{Calculate cost of a link.}
\label{suppalg:cost}
\Function{Cost}{$\mlink$}
    \State $\mnode \gets \fnode(\mlink)$
    \State $\mnodepar \gets \fnode(\flink(\mnode\mdot\mtdirnode, \mlink))$
    \State \Return $\lVert \fx(\mnode) - \fx(\mnodepar) \rVert$ + \Call{MinCost}{$\fnode(\mlink), \mlink$} \Comment{$\lVert \cdot \rVert$ denotes the Euclidean L2-norm}
\EndFunction
\end{algorithmic}
\end{algorithm}

%%%%%%%%%%%%%%%%%%%%%%%%%%%%%%%%% ALG: COPYLINK  %%%%%%%%%%%%%%%%%%%%%%%%%%%%%%%%%%%%%
\subsubsection{\textproc{CopyLink}: \textit{Duplicates a Link and its Connections}}
The \textproc{CopyLink} function (Alg.\ref{suppalg:copylink}) duplicates a link $\mlink$ and anchors the duplicated link at a new node $\mnode_\mnew$ if $\mnode_\mnew \ne \varnothing$. 
If $\mnode_\mnew = \varnothing$, $\mnode_\mnew$ defaults to $\fnode(\mlink)$.
The duplicated link is connected to the link's source or target links depending on $\mtdirs$. If $\mtdirs = \{S,T\}$ all links are connected. If $\mtdirs=\{S\}$ or $\mtdirs=\{T\}$ only the source or target links are connected respectively.
\begin{algorithm}[!ht]
\begin{algorithmic}[1]
\caption{Copy a link.}
\label{suppalg:copylink}
\Function{CopyLink}{$\mlink, \mnode_\mnew, \mtdirs$}
    \Comment{$\mtdirs \in \{\{S,T\}, \{S\}, \{T\}\}$}
    \If {$\mnode_\mnew = \varnothing$}
        \State $\mnode_\mnew \gets \fnode(\mlink)$
    \EndIf
    \State $\mlink_\mnew \gets $ \Call{CreateLink}{$\mnode_\mnew$}
    \State $\mlink_\mnew\mdot c \gets \mlink\mdot c$
    \State $\mlink_\mnew\mdot\mray_L \gets \mlink\mdot\mray_L$
    \State $\mlink_\mnew\mdot\mray_R \gets \mlink\mdot\mray_R$
    \For {$\mtdir \in \mtdirs$}
        \For {$\mlink_\mtdir \in \flinks(\mtdir, \mlink)$}
            \State \Call{Connect}{$\mtdir, \mlink_\mnew, \mlink_\mtdir$}
        \EndFor
    \EndFor
    \State \Return $\mlink_\mnew$
\EndFunction
\end{algorithmic}
\end{algorithm}

%%%%%%%%%%%%%%%%%%%%%%%%%%%%%%%%% ALG: ISOLATE  %%%%%%%%%%%%%%%%%%%%%%%%%%%%%%%%%%%%%
\subsubsection{\textproc{Isolate}: \textit{Isolates a Link Connection}}
The \textproc{Isolate} function  (Alg. \ref{suppalg:isolate}) tries to isolate the connection between a link $\mlink$ and a $\mtdir$-direction ($\mtdir\in\{S,T\}$) link $\mlink_\mtdir$.
If $\mlink$ is connected to only $\mlink_\mtdir$ in the $\mtdir$-direction, $\mlink$ is returned.
If $\mlink$ is connected to multiple links in the $\mtdir$-direction including $\mlink_\mtdir$, a new link is created that connects only to $\mlink_\mtdir$, and the connection between $\mlink$ and $\mlink_\mtdir$ is removed.
The function helps \rtwop{} to avoid data races between different queries.

$\mnode_\mnew$ is the new anchor node after the isolation. If $\mlink$ is not copied, $\mlink$ is re-anchored at $\mnode_\mnew$. If $\mlink$ is copied, the new link is anchored at $\mnode_\mnew$ instead.
If $\mnode_\mnew$ is not set ($\mnode_\mnew = \varnothing$ ), $\mnode_\mnew$ is defaulted to $\fnode(\mlink)$.
\begin{algorithm}[!ht]
\begin{algorithmic}[1]
\caption{Isolate a connection between two links.}
\label{suppalg:isolate}
\Function{Isolate}{$\mtdir$, $\mlink$, $\mlink_\mtdir$, $\mnode_\mnew$} \Comment{Note, if $\mlink_\mtdir \ne \varnothing$ then $\mlink_\mtdir \in \flinks_{\mtdir}(\mlink)$}
    \State $numLinks \gets \lvert \flinks_{\mtdir}(\mlink) \rvert$ \Comment{$numLinks$ is the number of links that are in $\mtdir$ direction of $\mlink$}
    \If {($numLinks = 0$ \An $\mlink_\mtdir = \varnothing$) 
        \Or ($numLinks = 1$ \An $\mlink_\mtdir \ne \varnothing$)} 
        \State $\mlink_\mnew \gets \mlink$ \Comment{Nothing to isolate}
    \Else    \Comment{Copy and isolate $(\mlink$, $\mlink_\mtdir)$ connection}
        % \State $\mlink_\mnew \gets (\{\}, \{\}, \mlink \mdot c, \mlink \mdot \mray_L, \mlink \mdot \mray_R)$
        % \State Add $\mlink_\mnew$ to $\fnode(\mlink) \mdot \mlinks_n$
        
        \State $\mlink_\mnew \gets $ \Call{CreateLink}{$\fnode(\mlink)$}
        \State $\mlink_\mnew\mdot c \gets \mlink\mdot c$
        \State $\mlink_\mnew\mdot\mray_L \gets \mlink\mdot\mray_L$
        \State $\mlink_\mnew\mdot\mray_R \gets \mlink\mdot\mray_R$
        
        \For {$\mlink_{-\mtdir} \in \flinks(-\mtdir, \mlink)$}
            \State \Call{Connect}{$-\mtdir$, $\mlink_\mnew$, $\mlink_{-\mtdir}$}
        \EndFor
        \If {$\mlink \ne \varnothing$}
            \State \Call{Connect}{$\mtdir$, $\mlink_\mnew$, $\mlink_\mtdir$}
            \State \Call{Disconnect}{$\mtdir$, $\mlink$, $\mlink_\mtdir$}
        \EndIf
    \EndIf
    \If {$\mnode_\mnew \ne \varnothing$}
        \State \Call{Anchor}{$\mlink_\mnew, \mnode_\mnew$}
    \EndIf
    \State \Return $\mlink_\mnew$
\EndFunction
\end{algorithmic}
\end{algorithm}

%%%%%%%%%%%%%%%%%%%%%%%%%%%%%%%%% ALG: ERASETREE  %%%%%%%%%%%%%%%%%%%%%%%%%%%%%%%%%%%%%
\subsubsection{\textproc{EraseTree}: \textit{Deletes Dangling Links}}
The \textproc{EraseTree} function (Alg. \ref{suppalg:erasetree}) erases a link $\mlink$ is dangling and does not connect to links in the $(-\mtdir)$-direction. 
The removal of $\mlink$ can cause the $\mtdir$-direction links to dangle, which the function subsequently erases.
If any dangling link is pointing to a queued query, the query is removed from the open-list and deleted.
\begin{algorithm}[!ht]
\begin{algorithmic}[1]
\caption{Erase Tree}
\label{suppalg:erasetree}
\Function{EraseTree}{$\mtdir$, $\mlink$} 
    \If{$\flinks(-\mtdir, \mlink) \ne \{\}$}
        \State \Return
    \ElsIf{$\fquery(\mlink) \ne \varnothing$}
        \State Unqueue $\fquery(\mlink)$ from open-list
    \EndIf
    \For {$\mlink_\mtdir \in \flinks(\mtdir, \mlink)$}
        \State \Call{Disconnect}{$\mtdir$, $\mlink$, $\mlink_\mtdir$}
        \Call{EraseTree}{$\mtdir$, $\mlink$}
    \EndFor
\EndFunction
\end{algorithmic}
\end{algorithm}

%%%%%%%%%%%%%%%%%%%%%%%%%%%% RAY %%%%%%%%%%%%%%%%%%%%%%%%%%%%%%%%%%%%%%%%%%%%%
\newpage
\subsection{Sector-ray}
A sector-ray has the following form
\begin{equation}
    \mray = (\mrtype, \mx_S, \mx_T, \mx_L, \mx_R),
\end{equation}
which records a cast from $\mx_S$ to $\mx_T$ and any collision information $\mx_L$ and $\mx_R$.
The ray type $\mrtype \in \{\mrvy, \mrvu, \mrvn\}$ indicates visibility between $\mx_S$ and $\mx_T$.
If $\mx_S$ and $\mx_T$ are visible to each other, $\mrtype = \mrvy$. If they are not visible to each other, $\mrtype= \mrvn$. If the ray is not yet cast and visibility is unknown, $\mrtype = \mrvu$.

When a ray collides with an obstacle's edge and $\mrtype=\mrvn$, $\mx_L$ and $\mx_R$ record the first corner encountered after tracing left and right from the collision point respectively.  
If the ray is $\mrvy$-typed, it can be \textbf{projected} in the direction $\mora{\mx_S,\mx_T}$ from $\mx_T$. 
The projected ray may collide and $\mx_L$ and $\mx_R$ records the collision of the projected ray. 
Note that, in \rtwop{}, projections always collide and never goes out of the map.

%%%%%%%%%%%%%%%%%%%%%%%%%%%%%%%%% ALG: FXCOL %%%%%%%%%%%%%%%%%%%%%%%%%%%%%%%%%%%%%
\subsubsection{\textproc{$\fxcol$}: \textit{Returns a Corner Beside the Collision Point}}
\textproc{$\fxcol$} (Alg. \ref{suppalg:fxcol}) returns first corner on the $\mside$-side ($\mside\in\{L,R\}$) of the collision point of a ray $\mray$. 
\begin{algorithm}[!ht]
\begin{algorithmic}[1]
\caption{Get collision point's left or right corner.}
\label{suppalg:fxcol}
\Function{$\fxcol$}{$\mside, \mray$}
    \State \Return $\mray\mdot\mx_L$ \textbf{if} $\mside = L$ \textbf{else} $\mray\mdot\mx_R$
\EndFunction
\end{algorithmic}
\end{algorithm}

%%%%%%%%%%%%%%%%%%%%%%%%%%%%%%%%% ALG: GETRAY  %%%%%%%%%%%%%%%%%%%%%%%%%%%%%%%%%%%%%
\subsubsection{\textproc{GetRay}: \textit{Retrieves or Constructs a Sector-ray}} 
The function \textproc{GetRay} (Alg. \ref{suppalg:getray}) finds a matching ray from $\mx_S$ to $\mx_T$ ifit exists and returns it.
If it does not exist, a new ray is created and returned.
\begin{algorithm}[!ht]
\begin{algorithmic}[1]
\caption{Retrieve existing ray or create it if it does not exist.}
\label{suppalg:getray}
\Function{GetRay}{$\mx_S$, $\mx_T$}
    \State Try finding $\mray = (\cdots, \mx_S, \mx_T, \cdots)$.
    \If{ $\mray$ does not exist}
        \State Create $\mray = (\mrvu, \mx_S, \mx_T, \varnothing, \varnothing)$
    \EndIf
    \State \Return $\mray$
\EndFunction
\end{algorithmic}
\end{algorithm}

%%%%%%%%%%%%%%%%%%%%%%%%%%%%%%%% QUEUED QUERY %%%%%%%%%%%%%%%%%%%%%%%%%%%%%%%%%%%%%%%%%%%
\subsection{Queued Query and Open-list} \label{suppsec:queuedquery}
A query, when queued to the open-list, has the form
\begin{equation}
    \mquery = (\mqtype, f, \mlink_\mquery),
\end{equation}
where $\mqtype \in \{ \mqcast, \mqtrace \}$ is used to denote if the query is a casting or tracing query. 
$f$ is the sum of the minimum cost-to-go and cost-to-come of a link $\mlink_\mquery$, which is anchored on a leaf node \textproc{$\fnode$}($\mlink_\mquery$). 

% When a casting query is \textit{polled} from the open-list, line-of-sight is checked  between the two nodes indirectly connected by the link. When line-of-sight is checked, a ray is cast from the node that is in the source direction to the node that is in the target direction.

% When a tracing query is \textit{polled} from the open-list, a trace begins from $\fx(\mlink)$. The anchored node $\fnode(\mlink)$ is a source leaf node. Before the trace begins, $\mtrace$ is created. $\mlink$ is transferred to $\mtrace\mdot\mnode_S$ and its target links transferred to $\mtrace\mdot\mnode_T$.

The \textbf{open-list} is a priority queue that sorts queries based on a query's cost $f$,
which serves the same purpose as the A*'s open-list.
The default sorting algorithm of the open-list is insert-sort which is fast if the number of queries are small. The number of queries are small if the path is expected to have few turning points. 
A more efficient sorting algorithm may have higher overheads, but can improve solving times if the path is expected to have many turning points, especially in maps with highly non-convex obstacles or many disjoint obstacles.

%%%%%%%%%%%%%%%%%%%%%%%%%%%%%%%%% ALG: QUEUE  %%%%%%%%%%%%%%%%%%%%%%%%%%%%%%%%%%%%%
\subsubsection{\textproc{Queue}: \textit{Queues a Query into the Open-list}}
The function \textproc{Queue} (Alg. \ref{suppalg:queue}) creates a query $\mquery$, and
sorts it into the open-list based on its cost $\mquery\mdot f$.
\begin{algorithm}[!ht]
\begin{algorithmic}[1]
\caption{Queue to Open-list}
\label{suppalg:queue}
\Function{Queue}{$\mqtype, f, \mlink_\mquery$}
    \State $\mquery \gets (\mqtype, f, \mlink_\mquery)$
    \State Insert $\mquery$ into open-list and sort $\mquery$ based on $\mquery\mdot f$
\EndFunction
\end{algorithmic}
\end{algorithm}

%%%%%%%%%%%%%%%%%%%%%%%%%%%%%%%%% ALG: UNQUEUE  %%%%%%%%%%%%%%%%%%%%%%%%%%%%%%%%%%%%%
\subsubsection{\textproc{Unqueue}: \textit{Removes a Queued Query from the Open-list}}
The function \textproc{Unqueue} (Alg. \ref{suppalg:unqueue}) removes a query $\mquery$ from the open-list.
\begin{algorithm}[!ht]
\begin{algorithmic}[1]
\caption{Unqueue from Open-list}
\label{suppalg:unqueue}
\Function{Unqueue}{$\mquery$}
    \State Remove $\mquery$ from the open-list
\EndFunction
\end{algorithmic}
\end{algorithm}

%%%%%%%%%%%%%%%%%%%%%%%%%%%%%%%%% ALG: POLL  %%%%%%%%%%%%%%%%%%%%%%%%%%%%%%%%%%%%%
\subsubsection{\textproc{Poll}: \textit{Removes Cheapest Query from the Open-list}}
The function \textproc{Poll} (Alg. \ref{suppalg:poll}) finds the cheapest queued query and returns it after removing it from the open-list.
\begin{algorithm}[!ht]
\begin{algorithmic}[1]
\caption{Poll Cheapest Queued Query from Open-list}
\label{suppalg:poll}
\Function{Poll}{\null}
    \State $\mquery \gets $ query that has smallest cost $\mquery\mdot f$ in open-list
    \State \Call{Unqueue}{$\mquery$}
    \State \Return $\mquery$
\EndFunction
\end{algorithmic}
\end{algorithm}

%%%%%%%%%%%%%%%%%%%%%%%%%%%%%%%% TRACING QUERY %%%%%%%%%%%%%%%%%%%%%%%%%%%%%%%%%%%%%%%%%%%
\subsection{Trace Object}
During a trace, any information or states related to the trace are captured in the Trace object
\begin{equation}
    \mtrace = ( \mxtrace, \msidetrace, \mtnode_S, \mtnode_T, \mnlets_S, \mnlets_T, \mnumcrns, \moverlap).
\end{equation}
$\mxtrace$ is the pair of coordinates for the corner being examined by the trace. $\msidetrace$ is the side of the trace. 
$\mtnode_S$ and $\mtnode_T$ are trace-nodes described in Sec. \ref{suppsec:tnodesandtlinks}.
$\mnlets_S$ and $\mnlets_T$ are ordered sets containing nodelets, described in Sec. \ref{suppsec:nodelets}.
$\mnumcrns$ is a counter for the number of corners traced, and $\moverlap$ is a Boolean flag indicating if the trace has found nodes that overlap with other queries.

%%%%%%%%%%%%%%%%%%%%%%%%%%%%%%%% ALG: FTNODE %%%%%%%%%%%%%%%%%%%%%%%%%%%%%%%%%%%%%%%%%%%
\newpage
\subsubsection{\textproc{$\ftnode$}: \textit{Gets a Trace-node from a Trace Object}}
\textproc{$\ftnode$} (Alg. \ref{suppalg:ftnode}) returns the $\mtdir$-direction ($\mtdir\in\{S,T\}$) trace-node $\mtnode_S$ or $\mtnode_T$ of a Trace object $\mtrace$.
\begin{algorithm}[!ht]
\begin{algorithmic}[1]
\caption{Get the source or target trace-node in Trace object}
\label{suppalg:ftnode}
\Function{$\ftnode$}{$\mtdir, \mtrace$}
    \State \Return $\mtrace\mdot\mtnode_S$ \textbf{if} $\mtdir = S$ \textbf{else} $\mtrace\mdot\mtnode_T$
\EndFunction
\end{algorithmic}
\end{algorithm}

%%%%%%%%%%%%%%%%%%%%%%%%%%%%%%%% ALG: FNLETS %%%%%%%%%%%%%%%%%%%%%%%%%%%%%%%%%%%%%%%%%%%
\subsubsection{\textproc{$\fnlets$}: \textit{Gets a Set of Nodelets from a Trace Object}}
\textproc{$\fnlets$} (Alg. \ref{suppalg:fnlets}) returns the $\mtdir$-direction ($\mtdir\in\{S,T\}$) set of nodelets $\mnlets_S$ or $\mnlets_T$ of a Trace object $\mtrace$.
\begin{algorithm}[!ht]
\begin{algorithmic}[1]
\caption{Gets the source or target nodelet set in Trace object}
\label{suppalg:fnlets}
\Function{$\fnlets$}{$\mtdir, \mtrace$}
    \State \Return $\mtrace\mdot\mnlets_S$ \textbf{if} $\mtdir = S$ \textbf{else} $\mtrace\mdot\mnlets_T$
\EndFunction
\end{algorithmic}
\end{algorithm}

%%%%%%%%%%%%%%%%%%%%%%%%%%%%%%%% ALG: FNLETS %%%%%%%%%%%%%%%%%%%%%%%%%%%%%%%%%%%%%%%%%%%
\subsubsection{\textproc{$\fnlet$}: \textit{Gets a Nodelet from a Trace Object}}
\textproc{$\fnlet$} (Alg. \ref{suppalg:fnlet}) returns the first nodelet from the $\mtdir$-direction ($\mtdir\in\{S,T\}$) set of nodelets $\mnlets_S$ or $\mnlets_T$. 
The nodelets belong to a Trace object $\mtrace$.
The function is useful since traces contain only one source nodelet at all times, and there may occasionally be only one target nodelet.
\begin{algorithm}[!ht]
\begin{algorithmic}[1]
\caption{Gets the first source or target nodelet}
\label{suppalg:fnlet}
\Function{$\fnlet$}{$\mtdir, \mtrace$}
    \State \Return $\mtrace\mdot\mnlets_S$[1] \textbf{if} $\mtdir = S$ \textbf{else} $\mtrace\mdot\mnlets_T$[1]
\EndFunction
\end{algorithmic}
\end{algorithm}

%%%%%%%%%%%%%%%%%%%%%%%%%%%%%%%%% ALG: CREATETRACE %%%%%%%%%%%%%%%%%%%%%%%%%%%%%%%%%%%%%
\subsubsection{\textproc{CreateTrace}: \textit{Creates a Trace Object}}
The function \textproc{CreateTrace} (Alg. \ref{suppalg:createtrace}) creates a new trace object and returns it.
\begin{algorithm}[!ht]
\begin{algorithmic}[1]
\caption{Create a Trace object.}
\label{suppalg:createtrace}
\Function{CreateTrace}{$\mxtrace, \msidetrace$}
    \State $\mtrace \gets (\mxtrace, \msidetrace, \varnothing, \varnothing, \{\}, \{\}, 0, \mfalse)$
    \State $\mtrace \mdot \mnode_S \gets (\mntm, \msidetrace, S, \{\})$
    \State $\mtrace \mdot \mnode_T \gets (\mntm, \msidetrace, T, \{\})$
    \State \Return $\mtrace$
\EndFunction
\end{algorithmic}
\end{algorithm}

\subsection{Trace-nodes and Trace-links} \label{suppsec:tnodesandtlinks}
Trace objects contain temporary nodes (\textbf{trace-nodes}) $\mtnode_S$ and $\mtnode_T$, which follow the trace and are always located at $\mxtrace$, such that $\fx(\mtnode_S) = \fx(\mtnode_T) = \mxtrace$.
The trace-nodes are not part of the source-tree or target-tree. The trace-nodes anchor temporary links called \textbf{trace-links}. 
The parent node of a trace-link is located on the source-tree or target-tree if the trace-link is anchored at $\mtnode_S$ or $\mtnode_T$ respectively.
At the start of the trace, trace-links may have been re-anchored to the trace-nodes from nodes in the trees, or may have been duplicated from existing links.

Trace-links do not have child links, and are intentionally left \textbf{dangling} (not connected to any source link and/or target link) to facilitate pruning and placement.
When a trace stops, trace-links are re-anchored to the source-tree or target-tree nodes, and re-connected to other links if they are not discarded.
If the trace-links are discarded, they would be1 left dangling.

\subsection{Nodelets} \label{suppsec:nodelets}
Trace objects contain \textbf{nodelets} in ordered sets $\mnlets_S$ and $\mnlets_T$.
A nodelet describes a trace-link and its associated progression ray and winding counter:
\begin{equation}
\mnlet = ( \mtlink, \mvprog, \mcprog).
\end{equation}
where $\mtlink$ is a trace-link anchored on a trace-node. 
$\mvprog$ is the progression ray with respect to the parent node. $\mcprog \ge 0$ is the winding counter for the progression ray.

During a trace, the nodelets are examined in sequence, and may be removed or added.
At any one time during the trace, only \textit{one} source nodelet is examined, while at least one target nodelet is examined. 
The trace ends when there are no source or target nodelets left.

%%%%%%%%%%%%%%%%%%%%%%%%%%%%%%%%% ALG: CREATENODELET  %%%%%%%%%%%%%%%%%%%%%%%%%%%%%%%%%%%%%
\subsubsection{\textproc{CreateNodelet}: \textit{Creates a Nodelet}}
The function \textproc{CreateNodelet} (Alg. \ref{suppalg:createnodelet}) creates a new nodelet object, pushes it to the front or back ($pos = \{\mfront, \mback\}$) of $\mnlets$ and returns the nodelet object.
\begin{algorithm}[!ht]
\begin{algorithmic}[1]
\caption{Create a new nodelet.}
\label{suppalg:createnodelet}
\Function{CreateTraceNode}{$\mtlink, \mvprog, pos, \mnlets$}
    \State $\mnlet \gets (\mtlink, \mvprog, 0)$
    \State Push $\mnlet$ to $pos$ of $\mnlets$
        \Comment{$pos \in \{\mfront, \mback\}$}
    \State \Return $\mnlet$
\EndFunction
\end{algorithmic}
\end{algorithm}

%%%%%%%%%%%%%%%%%%%%%%%%%%%%%%%%% OVERLAP-BUFFER  %%%%%%%%%%%%%%%%%%%%%%%%%%%%%%%%%%%%%
\subsection{Overlap-buffer}
The \textbf{overlap-buffer} is an unordered set containing Position objects. Each position object contains source-tree $\mnvu$ and/or $\mneu$ nodes. The nodes anchor links which are part of queries that overlap. 

During an iteration between two polls from the open-list, a few traces may occur.
The overlap-buffer is filled during the traces, if overlaps are identified.
The overlap-buffer is subsequently emptied and the links processed by the overlap rule after all traces in the iteration have stopped. 
% The source-tree $\mnvu$ and $\mneu$ nodes are identified at the position, and the overlap rule searches the links anchored on the nodes.
% Queries in the target-direction of the links are identified and removed from the open-list.
% The first source-tree $\mney$ or $\mnvy$ node in the source-direction of the links are subsequently identified.

%%%%%%%%%%%%%%%%%%%%%%%%%%%%%%%%% ALG: PUSHOVERLAP  %%%%%%%%%%%%%%%%%%%%%%%%%%%%%%%%%%%%%
\subsubsection{\textproc{PushOverlap}: \textit{Push a Position Object into Overlap-buffer}}
The function \textproc{PushOverlap} (Alg. \ref{suppalg:pushoverlap}) pushes a Position object $\mpos$ into the overlap-buffer. The overlap-buffer is a set of $\mpos$, and is filled at the end of a trace when overlapping paths are identified during the trace.
The reader may choose to avoid adding duplicate $\mpos$ into the overlap-buffer.

\begin{algorithm}[!ht]
\begin{algorithmic}[1]
\caption{Push to Overlap-buffer}
\label{suppalg:pushoverlap}
\Function{PushOverlap}{$\mpos$}
    \State Insert $\mpos$ to overlap-buffer (an unordered set of $\mpos$).
\EndFunction
\end{algorithmic}
\end{algorithm}

%%%%%%%%%%%%%%%%%%%%%%%%%%%%%%%%% OCCUPANCY GRID %%%%%%%%%%%%%%%%%%%%%%%%%%%%%%%%%%%%%
\subsection{Occupancy Grid}
The current implementation of \rtwop{} operates on a binary occupancy grid.
Each cell is either occupied or free, and \rtwop{} finds an any-angle path on vertices, which are the corners of the cells.
The occupancy grid is implemented as a Boolean array. 
Hash tables or equivalent containers consume less memory but are slower to access for two-dimensional occupancy grids.

%%%%%%%%%%%%%%%%%%%%%%%%%%%%%%%%% ALG: BISECTOCSEC  %%%%%%%%%%%%%%%%%%%%%%%%%%%%%%%%%%%%%
\subsubsection{\textproc{Bisect}: \textit{Gets Bisecting Directional Vector of a Corner}} 
The \textproc{Bisect} function (Alg. \ref{suppalg:bisect}) returns the directional vector that bisects the occupied-sector of a corner at $\mx$. 
For an occupancy grid, the bisecting vector points in the ordinal (northeast, northwest, etc.) directions.

\begin{algorithm}[!ht]
\begin{algorithmic}[1]
\caption{Get Bisecting Directional Vector of a Corner's Occupied-sector}
\label{suppalg:bisect}
\Function{Bisect}{$\mx$} 
    \State \Return vector parallel to a directional vector that points into and bisects the occupied-sector at the corner at $\mx$.
\EndFunction
\end{algorithmic}
\end{algorithm}

%%%%%%%%%%%%%%%%%%%%%%%%%%%%%%%%% ALG: GETEDGE  %%%%%%%%%%%%%%%%%%%%%%%%%%%%%%%%%%%%%
\subsubsection{\textproc{GetEdge}: \textit{Gets Directional Vector of an Obstacle's Edge}} 
The \textproc{GetEdge} function (Alg. \ref{suppalg:getedge}) returns the directional vector of an edge on the $\mside$-side of a corner at $\mx$. 
The directional vector should be \textit{parallel} to $\mx_\mside$ - $\mx$, where $\mx_\mside$ is the corner at the other end of the $\mside$-side edge.

\begin{algorithm}[!ht]
\begin{algorithmic}[1]
\caption{Get Edge Vector}
\label{suppalg:getedge}
\Function{GetEdge}{$\mx$, $\mside$} 
    \State \Return vector parallel to $\mside$-side edge adjacent to corner at $\mx$. \Comment{Vector points from $\mx$ to corner on the $\mside$-side.}
\EndFunction
\end{algorithmic}
\end{algorithm}

%%%%%%%%%%%%%%%%%%%%%%%%%%%%%%%%% ALG: TRACE %%%%%%%%%%%%%%%%%%%%%%%%%%%%%%%%%%%%%
\subsubsection{\textproc{Trace}: \textit{Traces an Obstacle's Edge}} 
The function \textproc{Trace} (Alg. \ref{suppalg:trace} traces to any corner from a position $\mx$ along the $\mside$-side edge of $\mx$. 
$\mx$ must be located on an obstacle's edge.
Corners encountered by the algorithm should be cached as a graph of connected corners, to allow repeatedly traced corners to be identified in constant time by the function.

The intersection of an obstacle edge with the map boundary is considered a corner.
A trace that continues from the boundary corner will cause the function to return $\varnothing$ as the trace has gone out of map.
\begin{algorithm}[!ht]
\begin{algorithmic}[1]
\caption{Trace To Corner}
\label{suppalg:trace}
\Function{Trace}{$\mx$, $\mside_d$}
    \State $\mx_\mnext \gets$ which is position of the first corner at the $\mside$-side of $\mx$
    \State \Return $\mx_\mnext$ if in map, or $\varnothing$ if out of map.
\EndFunction
\end{algorithmic}
\end{algorithm}

%%%%%%%%%%%%%%%%%%%%%%%%%%%%%%%%% ALG: LOS  %%%%%%%%%%%%%%%%%%%%%%%%%%%%%%%%%%%%%
\subsubsection{\textproc{LOS}: \textit{Collision Line Algorithm}}
The function \textproc{LOS} (Alg. \ref{suppalg:los}) casts or projects a ray using a line algorithm that can detect collisions.
If the occupancy grid is used, the Bresenham line algorithm can be implemented but should be modified to allow all cells intersected by the ray to be identified.
If a collision occurs, the function finds the first left or right corner from the collision point.

In \rtwop{}, a projected ray can never go out of a rectangular map, or a map with a convex shape. 
For completeness, the reader may choose to implement an out-of-map check when a projection occurs.

\rtwop{} depends on $\mx_L$ and $\mx_R$ to check if a trace has crossed a collision point of a sector-ray. Careful positioning of $\mx_L$ and $\mx_R$ is required to break ties in discrete, special cases. 
The cases can include a trace being parallel to a sector ray, or the start node lying on a corner or an edge, etc.
\begin{algorithm}[!ht]
\begin{algorithmic}[1]
\caption{Line-of-sight and Collision Finder}
\label{suppalg:los}
\Function{LOS}{$cast, \mray = (\mrtype, \mx_S, \mx_T, \mx_L, \mx_R)$}
    \State $\mv_\mathrm{ray} \gets \mx_T - \mx_S$
    \If {$cast = \mtrue$}
        \State Do line algorithm in direction $\mv_\mathrm{ray}$ from $\mx_S$ until collision or $\mx_T$ is reached.
    \Else
    \Comment{Project}
        \State Do line algorithm in direction $\mv_\mathrm{ray}$ from $\mx_T$ until collision. 
    \EndIf
    \If {a collision occurs at some $\mx_\mathrm{col}$}
        \State $\mray\mdot\mrtype \gets \mrvn$
        \If {$\mx_\mathrm{col}$ is at a corner}
            \State $\mray\mdot\mx_L \gets \mx_\mcol$
            \State $\mray\mdot\mx_R \gets \mx_\mcol$
            \State $\mv_\mathrm{crn} \gets $ \Call{Bisect}{$\mx_\mcol$}
            \State $u = \mv_\mathrm{crn} \times \mv_\mathrm{ray}$
            \If {$u < 0$}
            \Comment{Ray points to right of $\mv_\mathrm{crn}$}
                \State $\mray\mdot\mx_R \gets$ \Call{Trace}{$\mx_\mathrm{col}$, $R$}
            \ElsIf {$u > 0$}
                \Comment{Ray points to left of $\mv_\mathrm{crn}$}
                \State $\mray\mdot\mx_L \gets$ \Call{Trace}{$\mx_\mathrm{col}$, $L$}
            \EndIf
        \Else 
        \Comment{$\mx_\mathrm{col}$ is on an edge}
            \State $\mray\mdot\mx_L \gets $ \Call{Trace}{$\mx_\mathrm{col}, L$}
            \State $\mray\mdot\mx_R \gets $ \Call{Trace}{$\mx_\mathrm{col}, R$}
        \EndIf
    \Else \Comment{No collision, occurs only when $cast = \mtrue$}
        \State $\mray\mdot\mrtype \gets \mrvy$ 
    \EndIf
\EndFunction
\end{algorithmic}
\end{algorithm}

%%%%%%%%%%%%%%%%%%%%%%%%%%%%%%%%% ALG: CAST  %%%%%%%%%%%%%%%%%%%%%%%%%%%%%%%%%%%%%
\subsubsection{\textproc{Cast}: \textit{Casts a Ray}}
The function \textproc{Cast} (Alg. \ref{suppalg:cast}) wraps the function \textproc{LOS}, and returns immediately if the ray has been cast. 
\vspace{-0.3cm}
\begin{algorithm}[!ht]
\begin{algorithmic}[1]
\caption{Cast}
\label{suppalg:cast}
\Function{Cast}{$\mray = (\mrtype, \mx_S, \mx_T, \mx_L, \mx_R)$}
    \IfThen {$\mrtype = \mrvu$} { \Call{LOS}{$\mtrue, \mray$} }
\EndFunction
\end{algorithmic}
\end{algorithm}

%%%%%%%%%%%%%%%%%%%%%%%%%%%%%%%%% ALG: PROJECT  %%%%%%%%%%%%%%%%%%%%%%%%%%%%%%%%%%%%%
\subsubsection{\textproc{Project}: \textit{Projects a Ray}}
The function \textproc{Project} (Alg. \ref{suppalg:project}) wraps the function \textproc{LOS}, and returns immediately if the ray has been cast.
\vspace{-0.3cm}
\begin{algorithm}[!ht]
\begin{algorithmic}[1]
\caption{Project}
\label{suppalg:project}
\Function{Project}{$\mray = (\mrtype, \mx_S, \mx_T, \mx_L, \mx_R)$}
    \IfThen {$\mx_L = \varnothing$}{ \Call{LOS}{$\mfalse, \mray$}}
\EndFunction
\end{algorithmic}
\end{algorithm}

\clearpage
%%%%%%%%%%%%%%%%%%%%%%%%%%%%%%%%% ALG: RUN  %%%%%%%%%%%%%%%%%%%%%%%%%%%%%%%%%%%%%
\section{The \rtwop{} Algorithm}
This section details the \rtwop{} algorithm, and is constructed from the objects and methods listed in the prior sections.

\renewcommand{\thealgorithm}{\thesubsection} %.\arabic{algorithm}}
\subsection{\textproc{Run}: Main Algorithm}
\begin{algorithm}[!ht]
\begin{algorithmic}[1]
\caption{Main \rtwop{} algorithm}
\label{suppalg:run}
\Function{Run}{$\mxstart, \mxgoal$}
    \State $path \gets \{\}$
    \Comment{open-list, overlap-buffer, $path$, $\mxstart$ and $\mxgoal$ are accessible to all functions in R2*.}
    \If {\Call{InitialCaster}{\null} $ = \mfalse$} 
        \Comment{No direct path found.}
        \While{open-list is not empty}
            \State $\mquery \gets $ \Call{Poll}{\null}
            \If {$\mquery\mdot\mqtype = \mqtrace$}
                \State \Call{TracerFromLink}{$\mquery\mdot\mlink$}
            \ElsIf {\Call{Caster}{$\mquery\mdot\mlink$}}
                \State \Break \Comment{Path found.}
            \EndIf
            \If {overlap-buffer is not empty}
                \State \Call{ShrinkSourceTree}{\null} \Comment{Overlap rule for traces that overlap in this iteration.}
            \EndIf
        \EndWhile
    \EndIf

    \State \Return $path$

\EndFunction
\end{algorithmic}
\end{algorithm}
\renewcommand{\thealgorithm}{\thesubsubsection} %.\arabic{algorithm}}

\input{tex_supp_caster}
\input{tex_supp_tracer}
\input{tex_supp_overlap}

%% file: tex_supp_caster.tex
\subsection{Initial Casting and Tracing}
This section details initial functions used to initialize \rtwop{} and bring the algorithm into the main iteration between open-list polls.

%%%%%%%%%%%%%%%%%%%%%%%%%%%%%%%%% ALG: INITIALCASTER  %%%%%%%%%%%%%%%%%%%%%%%%%%%%%%%%%%%%%
\subsubsection{\textproc{InitialCaster}: \textit{First Cast}} \label{suppsec:initialcaster}
The first caster function \textproc{InitialCaster} (Alg. \ref{suppalg:initialcaster}) is a special casting-query function between the start and goal points. The function is used only once.

If there is line-of-sight between the start and goal points, \rtwop{} returns immediately with the path.
If the ray from the start to goal collides, \rtwop{} begins initialization, and two traces and a reversed ray are generated.

The reversed ray begins at the goal point and ends at the start point, which is opposite to the direction of the cast.
The reversed ray ensures that calculations with a start node's angular-sector are correct, primarily by dividing the start node into two nodes, each with a $180^\circ$ angular-sector.
The $180^\circ$ sectors are bounded by the forward ray (start to goal) and the reversed ray (goal to start), and ensures that all calculations with the cross-product are correct.

The goal node is initialized before being passed to the tracing queries.
\begin{algorithm}[!ht]
\begin{algorithmic}[1]
\caption{InitialCaster}
\label{suppalg:initialcaster}
\Function{InitialCaster}{\null}
    \State $\mray \gets $ \Call{GetRay}{$\mxstart, \mxgoal$}
    \State \Call{Cast}{$\mray$}
    \If {$\mray\mdot\mrtype = \mrvy$}
        \Comment{Start and goal points have line-of-sight}
        \State $path \gets \{ \mxgoal, \mxstart \}$
        \State \Return $\mtrue$
    \EndIf
    
    \Comment{Cast collided, create nodes and begin tracing queries}
    \State $\mray_\mathrm{rev} \gets $ \Call{GetRay}{$\mxgoal, \mxstart$}
    \Comment{A special, reversed ray to ensure that the ang-sec. for start nodes are convex.}
    \State $\mnode_{T} \gets $ \Call{GetNode}{$\mxgoal, \mnvy, L, T$} 
    \Comment{Goal node can have any side.}

    \State \Call{InitialTrace}{$L, \mray, \mray_\mathrm{rev}, \mnode_T$}
    \State \Call{InitialTrace}{$R, \mray, \mray_\mathrm{rev}, \mnode_T$}

    \State \Return $\mfalse$
\EndFunction
\end{algorithmic}
\end{algorithm}

%%%%%%%%%%%%%%%%%%%%%%%%%%%%%%%%% ALG: INITIALTRACER  %%%%%%%%%%%%%%%%%%%%%%%%%%%%%%%%%%%%%
\subsubsection{\textproc{InitialTrace}: \textit{First Trace}}
When a cast from the start point to the goal point collides, the \textproc{InitialTrace} function (Alg. \ref{suppalg:initialcaster}) initializes links and two start nodes.
After initialization, a trace begins from one-side of the collision point. 
The function is used only twice, each for one side of the collision.

Two start nodes $\mnode_S$ and $\mnode_{SS}$ with sides $\msidetrace$ are created by the function, each with a $180^\circ$ angular-sector. 
For brevity, a \textit{start node} in the text refers to any one of the two, unless a distinction needs to be made.
Zero-length links $\mlink_S$ and $\mlink_{SS}$ are anchored at $\mnode_S$ and $\mnode_{SS}$ respectively. 
$\mlink_{SS}$ is the source link of $\mlink_S$, and $\mlink_S$ is the source link of the trace-link $\mtlink_S$ which is anchored on the source trace-node.
$\mtlink_S$ and $\mlink_S$ stores the $180^\circ$ angular-sectors of $\mnode_S$ and $\mnode_{SS}$ respectively.
When viewed from the start point, the angular-sector of $\mnode_S$ is the \textit{first} $180^\circ$ angular sector, which rotates in the $\mside$-direction from the forward ray to the reverse ray.
The angular-sector of $\mnode_{SS}$ is the \textit{second} $180^\circ$ angular sector, which rotates from the reverse ray to the forward ray.

% Fig. \ref{suppfig:initialtracer} illustrates the sectors and nodes.

The goal link is initialized in the tracing query instead of in \textproc{InitialCaster} like the goal node.
If the goal link is initialized in \textproc{InitialCaster}, it may be deleted by the first call to \textproc{InitialTracer} and before the second call is made, as the first trace can go out-of-map.

Note that as more collisions occur, more links will be connected to the goal link, and the set of source link pointers $\mlinks_S$ in the goal link can become anomalously large, particularly in maps with highly non-convex obstacles or many disjoint obstacles.
For example, while the majority of links may have connections numbering less than ten, the goal link may have over a few thousand connections.
The reader may choose to use a container with constant time insertion or deletion (note that $\mlinks_S$ is unordered), but as the number of links are usually very small, a contiguous data structure with linear time insertion or deletion may be more efficient.
If \rtwop{} is expected to perform on maps with highly non-convex obstacles or with many disjoint obstacles, a container with constant time insertion and deletion is recommended. 
Otherwise, a simple, contiguous data structure like an array will suffice, which is the default implementation of \rtwop{}.

\begin{algorithm}[!ht]
\begin{algorithmic}[1]
\caption{InitialTrace}
\label{suppalg:initialtrace}
\Function{InitialTrace}{$\msidetrace, \mray, \mray_\mathrm{rev}, \mnode_T$}
    \State $\mtrace \gets $ \Call{CreateTrace}{$\fxcol(\msidetrace, \mray), \msidetrace$}

    \Comment{Create source trace-link}
    \State $\mtlink_S \gets $ \Call{CreateLink}{$\mtrace\mdot\mtnode_S$}
    \State $\fray(\msidetrace, \mtlink_S) \gets \mray_\mathrm{rev}$
    \State $\fray(-\msidetrace, \mtlink_S) \gets \mray$
    \State $\mvray \gets \mray\mdot\mx_T - \mray\mdot\mx_S$

    \Comment{Create $\msidetrace$-side start node and link for first $180^\circ$ ang-sec.}
    \State $\mnode_S \gets $ \Call{GetNode}{$\mxstart, \mnvy, \msidetrace, S$}
    \State $\mlink_S \gets $ \Call{CreateLink}{$\mnode_S$}
    \State $\fray(\msidetrace, \mlink_S) \gets \mray$
    \State $\fray(-\msidetrace, \mlink_S) \gets \mray_\mathrm{rev}$
    \State $\mlink_S\mdot\mcost \gets 0$
    \State \Call{Connect}{$T, \mlink_S, \mtlink_S$}

    \Comment{Create link for second $180^\circ$ ang-sec of start node.}
    \State $\mlink_{SS} \gets $ \Call{CreateLink}{$\mnode_S$}
    \State $\mlink_{SS}\mdot\mcost \gets 0$
    \State \Call{Connect}{$T, \mlink_{SS}, \mlink_S$}

    \Comment{Create a link anchored at goal node and a target trace-link}
    \State $\mtlink_T \gets $ \Call{CreateLink}{$\mtrace\mdot\mtnode_T$}
    \State $\mlink_T \gets $ \Call{CreateLink}{$\mtnode_T$}
    \State $\mlink_T\mdot\mcost \gets 0$
    \State \Call{Connect}{$T, \mtlink_T, \mlink_T$}

    \State \Call{CreateNodelet}{$\mtlink_S, \mvray, \mback, \mtrace\mdot\mnlets_S$}
    \State \Call{CreateNodelet}{$\mtlink_T, -\mvray, \mback, \mtrace\mdot\mnlets_T$}
    \State \Call{Tracer}{$\mtrace$}

    \EndFunction
\end{algorithmic}
\end{algorithm}

\clearpage
\subsection{\textproc{Caster}: \textit{Main Casting Query}}
The \textproc{Caster} function (Alg. \ref{suppalg:caster}) implements a casting query for a link $\mlink$.
A ray is cast from the source node of the link to the target node of the link.
The link can be anchored on either node, and is connected to one source link and at least one target link.
%%%%%%%%%%%%%%%%%%%%%%%%%%%%%%%%% ALG: CASTER  %%%%%%%%%%%%%%%%%%%%%%%%%%%%%%%%%%%%%
\renewcommand{\thealgorithm}{\thesubsection} %.\arabic{algorithm}}
\begin{algorithm}[!ht]
\begin{algorithmic}[1]
\caption{Caster}
\label{suppalg:caster}
\Function{Caster}{$\mlink$}
    \State $\mnode_S \gets \fnode(\flink(S, \mlink))$
    \State $\mnode_T \gets \fnode(\flink(T, \mlink))$

    \State $\mray \gets $ \Call{GetRay}{$\fx(\mnode_S), \fx(\mnode_T)$}
    \State \Call{Cast}{$\mray$}

    \If {$\mray\mdot\mrtype = \mrvy$}
        \State \Return \Call{CastReached}{$\mray, \mlink$}
    \Else
        \State \Call{CastCollided}{$\mray, \mlink$}
        \State \Return $\mfalse$
    \EndIf
\EndFunction
\end{algorithmic}
\end{algorithm}
\renewcommand{\thealgorithm}{\thesubsubsection} %.\arabic{algorithm}}

%%%%%%%%%%%%%%%%%%%%%%%%%%%%%%%%% ALG: CASTREACHED  %%%%%%%%%%%%%%%%%%%%%%%%%%%%%%%%%%%%%
\subsubsection{\textproc{CastReached}: \textit{Line-of-sight between Nodes of Cast}} 
The function \textproc{CastReached} (Alg. \ref{suppalg:castreached}) handles the case where a cast ray has line-of-sight.

In general, five cases occur when a cast reaches the target node from the source node.
\begin{enumerate}
    \item The shortest path is found when both nodes have cumulative visibility. This case is handled by the function \textproc{PathFound} (Alg. \ref{suppalg:pathfound}).
    \item The target node is an unreachable, $\mnun$-node. The caster query is discarded.
    \item The target node is a $\mntm$-node that is generated by an interrupted trace. The node may or may not be a valid turning point, and this case is handled by the function \textproc{ReachedTm} (Alg. \ref{suppalg:reachedtm}).
    \item Both nodes have unknown cumulative visibility. A new casting query for each target link of the link $\mlink$ is queued unless there are overlapping queries. This case is handled by the function \textproc{NoCumulativeVisibility} (Alg. \ref{suppalg:nocumulativevisibility}).
    \item Only the source node or target node has cumulative visibility. 
    If the source or target node has cumulative visibility, \rtwop{} proceeds to queue casting queries for the target links or source link respectively. This case is handled by the function \textproc{SingleCumulativeVisibility} (Alg. \ref{suppalg:singlecumulativevisibility}).
\end{enumerate}

% Fig. \ref{suppfig:castreached}.
\begin{algorithm}[!ht]
\begin{algorithmic}[1]
\caption{Cast reached and ray has line-of-sight}
\label{suppalg:castreached}
\Function{CastReached}{$\mray, \mlink$}
    \State $\mnode_S \gets \fnode(\flink(S, \mlink))$
    \State $\mnode_T \gets \fnode(\flink(T, \mlink))$
    
    \If{$\mnode_S\mdot\mntype = \mnvy $ \Or $\mnode_T\mdot\mntype = \mnvy$} 
        \Comment{Shortest path found.}
        \State \Call{PathFound}{$\mlink$}
        \State \Return $\mtrue$
    \ElsIf {$\mnode_T\mdot\mntype = \mnun$}
        \Comment{Discard if target node has type $\mnun$}
        \State \Call{DiscardReachedCast}{$\mlink$}   
    \ElsIf{$\mnode_T\mdot\mntype = \mntm$ }
        \Comment{Reached an interrupted trace.}
        \State \Call{ReachedTm}{$\mray, \mlink$}
    \ElsIf{$\mnode_S\mdot\mntype \notin \{\mnvy, \mney\}$ \An $\mnode_T\mdot\mntype \notin \{ \mnvy, \mney \}$}
        \Comment{$\mtrue$ if $\mnvu$ source node reached $\mnvu$ or $\mnoc$ target-tree node.}
        \State $(\mtdir_\mnext, \sim) \gets$ \Call{NoCumulativeVisibility}{$\mlink$}
        \State \Call{QueueReachedCast}{$\mtdir_\mnext, \mlink$}
    \Else
        \Comment{Either source or target node has cumulative visibility.}
        \State $(\mtdir_\mnext, \mnode_\mnext) \gets$ \Call{SingleCumulativeVisbility}{$\mray, \mlink$}
        \If {$\mnode_\mnext \ne \varnothing$}
            \State \Call{QueueReachedCast}{$\mtdir_\mnext, \mlink$}
        \EndIf
    \EndIf
    \State \Return $\mfalse$
\EndFunction
\end{algorithmic}
\end{algorithm}

%%%%%%%%%%%%%%%%%%%%%%%%%%%%%%%%% ALG: REACHEDTM  %%%%%%%%%%%%%%%%%%%%%%%%%%%%%%%%%%%%%
\newpage
\paragraph{\textproc{ReachedTm}: \textit{Cast Reached for Target $\mntm$-node}}
The helper function \textproc{ReachedTm} handles cases when a cast reached a $\mntm$-node.
A $\mntm$-node is generated by an interrupted trace, which occurs when a number of corners is traced (Alg. \ref{suppalg:interrupt}), a recursive occupied-sector trace is called from the source node (Alg. \ref{suppalg:ocsecrule}), or a recursive angular-sector trace is called (Alg. \ref{suppalg:recurangsectrace}).
The $\mntm$-node has the same side as the interrupted trace that generated it.

As the $\mntm$-node can lie on any corner, \rtwop{} first checks if a turning point can be placed at the corner where the $\mntm$-node lies.
If a turning point cannot be placed at the corner, a trace resumes from the corner using the function \textproc{TracerFromLink} (Alg. \ref{suppalg:tracerfromlink}). 

If a turning point can be placed, the query proceeds by treating the target $\mntm$-node as a turning point.
If the source node has no cumulative visibility, 
the function \textproc{NoCumulativeVisibility} (Alg. \ref{suppalg:nocumulativevisibility}) is called.
If the source node has cumulative visibility, 
the function \textproc{SingleCumulativeVisibility} (Alg. \ref{suppalg:singlecumulativevisibility}) is called.

Even if a turning point can be placed, the target links of $\mlink$ may point into the occupied-sector of the target $\mntm$-node and are not castable.
A trace is generated for all non-castable links in \textproc{TraceFromTm} (Alg. \ref{suppalg:tracefromtm}), while a casting query is queued for each castable link in \textproc{CastFromTm} (Alg. \ref{suppalg:castfromtm}).

% Fig. \ref{suppfig:reachedtm}.

\renewcommand{\thealgorithm}{\theparagraph} %.\arabic{algorithm}}
\begin{algorithm}[!ht]
\begin{algorithmic}[1]
\caption{Cast reached an interrupted trace.}
\label{suppalg:reachedtm}
\Function{ReachedTm}{$\mray, \mlink$}
    \State $\mlink_S \gets \flink(S, \mlink)$
    \State $\mnode_S \gets \fnode(\flink(S, \mlink))$
    \State $\mnode_T \gets \fnode(\flink(T, \mlink))$
    \State $\mvray \gets \mray\mdot\mv_T - \mray\mdot\mv_S$
    \State $\mxtrace \gets \fx(\mnode_T)$
    \State $\msidetrace \gets \mnode_T\mdot\msidenode$
    \State $\mvnext \gets $ \Call{GetEdge}{$\msidetrace, \fx(\mnode_T)$}

    \Comment{Trace immediately if no turning point can be placed at $\mntm$ node.}
    \If {corner at $\mxtrace$ is non-convex \Or \Call{IsRev}{$\msidetrace, \mvray, \mvnext$} = $\mfalse$}
        \State \Call{TracerFromLink}{$\mlink$}
        \State \Return
    \EndIf

    \Comment{Try placing a turning point at $\mntm$ node.}
    \State $\mnode_\mnext \gets \varnothing$
    \If {$\mnode_S \notin \{\mnvy, \mney\}$}
        \State $(\sim, \mnode_\mnext) \gets $ \Call{NoCumulativeVisibility}{$\mlink$}
    \Else
        \State $(\sim, \mnode_\mnext) \gets $ \Call{SingleCumulativeVisibility}{$\mlink$}
        \If {$\mnode_\mnext = \varnothing$}  
            \State \Return
        \EndIf
    \EndIf

    \Comment{Queue castable links from new turning point.}
    \State $\hat{\mlinks}_\mathrm{newT} \gets $ \Call{CastFromTm}{$\mlink,  \mxtrace, \msidetrace, \mvnext$}
    
    \Comment{Continue tracing for non-castable links at new turning point.}
    \If {$\hat{\mlinks}_\mathrm{newT} \ne \{\}$}
        \State \Call{TraceFromTm}{$\mlink,  \mxtrace, \msidetrace, \mvnext, \hat{\mlinks}_\mathrm{newT}$}
    \EndIf
\EndFunction
\end{algorithmic}
\end{algorithm}

%%%%%%%%%%%%%%%%%%%%%%%%%%%%%%%% ALG: CASTROMTM %%%%%%%%%%%%%%%%%%%%%%%%%%%%%%%%%
\newpage
\paragraph{\textproc{CastFromTm}: \textit{Try Casting From a Reached $\mntm$-node}}
The helper function \textproc{CastFromTm} (Alg. \ref{suppalg:castfromtm}) identifies castable target links of $\mlink$ and queues a casting query for each castable link.
Links that are not castable are pushed into an unordered set of links $\hat{\mlinks}_\mathrm{newT}$ and returned.
\begin{algorithm}[!ht]
\begin{algorithmic}[1]
\caption{Try casting from a new turning point at reached $\mntm$ node.}
\label{suppalg:castfromtm}
\Function{CastFromTm}{$\mlink, \mxtrace, \msidetrace, \mvnext$}
    \State $\hat{\mlinks}_\mathrm{newT} \gets \{\}$
    \For {$\mlink_T \in \mlink\mdot\mlinks_T$}
        \State $\mnodepar \gets \fnode(\flink(T, \mlink_T))$
        \State $\mvpar \gets \mxtrace - \fx(\mnodepar)$
        % \If {$\mnodepar\mdot\mntype = \mnph$ \Or \Call{IsVis}{$\msidetrace, \mvpar, \mvnext$}$ = \mfalse$}
        %     \State Push $\mlink_T$ to back of $\hat{\mlinks}_\mathrm{newT}$
        %     \Comment{Trace from target node if next node is phantom point or not castable.}
        % \Else
        %     \Comment{Next node is castable and not phantom point.}
        %     \State $\mnode_\mnew \gets $ \Call{GetNode}{$\mxtrace, \mnvu, \msidetrace, T$}
        %     \State \Call{Anchor}{$\mlink_T, \mnode_\mnew$}
        %     \State \Call{Queue}{$\mqcast, \mlink\mdot\mcost + \mlink_T\mdot\mcost, \mlink_T$}
        % \EndIf
        \If {\Call{IsVis}{$\msidetrace, \mvpar, \mvnext$}}
            \Comment{Next node is castable and not phantom point.}
            \State $\mnode_\mnew \gets $ \Call{GetNode}{$\mxtrace, \mnvu, \msidetrace, T$}
            \State \Call{Anchor}{$\mlink_T, \mnode_\mnew$}
            \State \Call{Queue}{$\mqcast, \mlink\mdot\mcost + \mlink_T\mdot\mcost, \mlink_T$}
        \EndIf
    \EndFor
    \State \Return $\hat{\mlinks}_\mathrm{newT}$
\EndFunction
\end{algorithmic}
\end{algorithm}

%%%%%%%%%%%%%%%%%%%%%%%%%%%%%%%% ALG: TRACEFROMTM %%%%%%%%%%%%%%%%%%%%%%%%%%%%%%%%%
\newpage
\paragraph{\textproc{TraceFromTm}: \textit{Try Tracing From a Reached $\mntm$-node}}
The helper function \textproc{TraceFromTm} (Alg. \ref{suppalg:tracefromtm}) generates a trace for all non-castable target links anchored on the target $\mntm$-node of a reached cast.
The links are contained in an unordered set $\hat{\mlinks}_\mathrm{newT}$, and are found by the function \textproc{CastFromTm} (Alg. \ref{suppalg:castfromtm}).
\begin{algorithm}[!ht]
\begin{algorithmic}[1]
\caption{Try tracing from a new turning point at reached $\mntm$ node.}
\label{suppalg:tracefromtm}
\Function{TraceFromTm}{$\mlink, \mxtrace, \msidetrace, \mvnext, \hat{\mlinks}_\mathrm{newT}$}
    \State $\mtrace \gets $ \Call{CreateTrace}{$\mxtrace, \msidetrace$}
    \State $\mtlink_\mathrm{newS} \gets $ \Call{CreateLink}{$\mtrace\mdot\mnode_S$}
    \State \Call{Connect}{$T, \mlink, \mtlink_\mathrm{newS}$}
    \State \Call{CreateNodelet}{$\mtlink_\mathrm{newS}, \mvnext, \mback, \mtrace\mdot\mnlets_S$}

    \For {$\mtlink_\mathrm{newT} \in \hat{\mlinks}_\mathrm{newT}$}
        \State $\mnode_\mathrm{newT} \gets \fnode(\flink(T, \mtlink_\mathrm{newT}))$
        \State \Call{Anchor}{$\mtlink_\mathrm{newT}, \mtrace\mdot\mnlets_T$}
        \State \Call{Disconnect}{$T, \mlink, \mtlink_\mathrm{newT}$}
        \State \Call{CreateNodelet}{$\mtlink_\mathrm{newT}, \mxtrace - \fx(\mnode_\mathrm{newT}), \mback, \mtrace\mdot\mnlets_T$}
    \EndFor

    \State $\mtrace\mdot\mxtrace \gets $ \Call{Trace}{$\mxtrace, \msidetrace$}
    \State \Call{Tracer}{$\mtrace$}
\EndFunction
\end{algorithmic}
\end{algorithm}

%%%%%%%%%%%%%%%%%%%%%%%%%%%%%%%% ALG: PATHFOUND %%%%%%%%%%%%%%%%%%%%%%%%%%%%%%%%%
\paragraph{\textproc{PathFound}: \textit{Cumulative Visibility for Both Nodes of Reached Cast }}
The helper function \textproc{PathFound} (Alg. \ref{suppalg:pathfound}) returns the shortest path when the source and target nodes of a cast have cumulative visibility.

If the source node and target node have cumulative visibility, the nodes have to be $\mnvy$-type, as it is impossible for one node to be an $\mney$-node. 
An expensive query will have at least one source-tree or target-tree $\mney$-node along its examined path.
Subsequent casts by expensive queries will only occur if the source or target node of a cast is $\mney$-type.
By ensuring that either the source node or target node is $\mney$-type, the cost-to-go and cost-to-come of  nodes in subsequent casting queries can be verified, respectively.
As such, it is impossible for an expensive path to be unobstructed. 
An expensive path that is unobstructed implies that no cheaper, unobstructed path exists in the open-list, which is impossible since \rtwop{} is complete \cite{bib:r2}.
\begin{algorithm}[!ht]
\begin{algorithmic}[1]
\caption{Cast reached between nodes with cumulative visibilities: found an optimal path}
\label{suppalg:pathfound}
\Function{PathFound}{$\mlink$}
    % \State $\mnode_S \gets \fnode(\flink(T, \mlink))$
    % \State $\mnode_T \gets \fnode(\flink(S, \mlink))$
    \Comment{Found shortest path because link connects a $\mnvy$ source-tree node and a $\mnvy$ target-tree node.}
    % \State $path$ $\gets \{\fx(\mnode_T), \fx(\mnode_S)\}$
    \Comment{$path$ is accessible to all functions.}
    \State $\mlink_i \gets \flink(T, \mlink)$ 
    \State $path \gets \{\fx(\fnode(\mlink_i))\}$
    \Comment{There is only one target link.}
    \While {front of $path \ne \mxgoal$}
        \State $\mlink_i \gets \flink(T, \mlink_i)$
        \State Push $\fx(\fnode(\mlink_i))$ to front of $path$
    \EndWhile
    \State $\mlink_i \gets \flink(S, \mlink)$
    \State $path \gets $ Push $\fx(fnode(\mlink_i))$ to back of $path$
    \While {back of $path \ne \mxstart$}
        \State $\mlink_i \gets \flink(S, \mlink_i)$
        \State Push $\fx(\fnode(\mlink_i))$ to back of $path$
    \EndWhile
    \State \Return $\mtrue$
\EndFunction
\end{algorithmic}
\end{algorithm}

%%%%%%%%%%%%%%%%%%%%%%%%%%%%%%%% ALG: NOCUMULATIVEVISIBILITY %%%%%%%%%%%%%%%%%%%%%%%%%%%%%%%%%
\paragraph{\textproc{NoCumulativeVisibility}: \textit{Unknown Cumulative Visibility for Both Nodes of Reached Cast}}
The helper function \textproc{NoCumulativeVisibility} (Alg. \ref{suppalg:nocumulativevisibility}) handles the case when the source and target node of a reached cast do not have verified cumulative visibility to the start or goal node respectively.
Such a case occurs if the source node is $\mnvu$-type, and if the target node is $\mnvu$, $\mnoc$ or $\mntm$-type.
\rtwop{} proceeds by queuing a casting query for each target link of the expanded link $\mlink$.
\begin{algorithm}[!ht]
\begin{algorithmic}[1]
\caption{Cast succeeds between nodes with no cumulative visibility to start and goal nodes}
\label{suppalg:nocumulativevisibility}
\Function{NoCumulativeVisibility}{$\mlink$}
    \State $\mnode_S \gets \fnode(\flink(S, \mlink))$
    \State $\mnode_T \gets \fnode(\flink(T, \mlink))$
    \State $\mnode_\mnext \gets $ \Call{GetNode}{$\fx(\mnode_T), \mnvu, \mnode_T\mdot\mside, S)$}
    \State \Call{FinishReachedCast}{$\mtdir_\mnext, \mnode_\mnext, \mlink, \mray$}
    \State \Return ($T, \mnode_\mnext$)
\EndFunction
\end{algorithmic}
\end{algorithm}

%%%%%%%%%%%%%%%%%%%%%%%%%%%%%%%% ALG: SINGLECUMULATIVEIVISIBILITY %%%%%%%%%%%%%%%%%%%%%%%%%%%%%%%%%
\paragraph{\textproc{SingleCumulativeVisibility}: \textit{Cumulative Visibility for Either Node of Reached Cast}} \label{suppsec:singlecumulativevisibility}
The helper function \textproc{SingleCumulativeVisibility} (Alg. \ref{suppalg:singlecumulativevisibility}) 
when either the source or target node of a trace has cumulative visibility to the start or goal node respectively.

The algorithm progresses by checking nodes with no cumulative visibility, and determining the \textbf{next} direction is important.
If the source node has cumulative visibility, the next node is the target node, and the next direction is in the target-direction.
If the target node has cumulative visibility, the next node is the source node, and the next direction is in the source-direction.
The \textbf{previous} node is the source or target node which is not the next node.

The function updates the minimum cost at the next node's corner if the current query is the cheapest when reaching the next node. If the next node is the source node, the cost-to-go at the source node's corner is checked; if the next node is the target node, the cost-to-come at the target node's corner is checked.
The overlap rule subsequently discards or re-anchors expensive links to expensive nodes in \textproc{ConvToExBranch} (Alg. \ref{suppalg:convtoexbranch}). 

In \rtwo{}, more expensive links are re-anchored at $\mney$-nodes, and an expensive query does not generate additional queries \cite{bib:r2}.
\rtwop{} introduces a new rule to reduce the number of expensive queries further, by discarding the current query if it is expensive and is guaranteed to cross a cheaper path (see Sec. \ref{suppsec:overlap}).
The guarantee and the cheapest cost can be obtained by examining the Best object $\mbest_\mnext$ at the corner.
If the current query is the cheaper or equal to the minimum cost, the query updates $\mbest_\mnext$ instead.

\begin{algorithm}[!ht]
\begin{algorithmic}[1]
\caption{Cast between a node with cumulative visibility and a node with no cumulative visibility}
\label{suppalg:singlecumulativevisibility}
\Function{SingleCumulativeVisibility}{$\mray, \mlink$}
        \Comment{Either $\mnode_S$ or $\mnode_T$ has a type that is $\mney$ or $\mnvy$.}
    \State $\mnode_S \gets \fnode(\flink(S, \mlink))$
    \State $\mnode_T \gets \fnode(\flink(T, \mlink))$
    \State $\mtdir_\mnext \gets T$ if $\mnode_S\mdot\mntype \notin \{ \mnvy, \mney \}$ else $S$ 
        \Comment{$\mnode_S\mdot\mntype \notin \{\mnvy, \mney\} \implies \mnode_S\mdot\mntype = \mnvu$} 
    \State $\mnode_\mnext \gets \mnode_S$ if $\mtdir_\mnext = S$ else $\mnode_T$
    \State $\mnode_\mprev \gets \mnode_T$ if $\mtdir_\mnext = S$ else $\mnode_S$

    \If{$\mnode_\mprev\mdot\mntype = \mney$ \An $\mnode_\mprev\mdot\msidenode \ne \mnode_\mnext\mdot\msidenode$}
        \Comment{Discard if a node is expensive and both have different sides.}
        \State \Call{DiscardReachedCast}{$\mlink$}
        \State \Return ($\mtdir_\mnext, \varnothing$)
    \EndIf
    
    \State $\mcost_\mnext \gets $ \Call{MinCost}{$-\mtdir_\mnext, \mlink$} + $\lVert \mray\mdot\mx_T - \mray\mdot\mx_S \rVert$ 
    \State $\mpos_\mnext \gets \fpos(\mnode_\mnext)$
    \State $\mbest_\mnext \gets \fbest(-\mtdir_\mnext, \mpos_\mnext)$
    \State $\mside_\mbest \gets \mbest_\mnext\mdot\mnodebest\mdot\msidenode$
    \State $\mv_\mbest \gets \fx(\mnode_\mnext) - \mbest_\mnext\mdot\mxbest$
    \State $\mv_\mathrm{test} \gets \fx(\mnode_\mnext) - \fx(\mnode_\mprev))$
    \If {$\mbest_\mnext\mdot\mcost < c_\mnext$}
        \Comment{Current path to $\mnode_\mnext$ has larger ($-\mtdir_\mnext$)-cost than minimum at $\fpos(\mnode_\mnext)$.}
       
        \If {$\mnode_\mnext \mdot \msidenode = \mside_\mbest$ \An $\mtdir_\mnext \mside_\mbest(\mv_\mbest \times \mv_\mathrm{test}) > 0$}
            \State \Call{DiscardReachedCast}{$\mlink$}
            \State \Return ($\mtdir_\mnext, \varnothing$)
                \Comment{Future queries is always costlier as their paths cross the cheapest path to $\mx$.}
        \EndIf
        \State $\mnode_\mnext \gets $ \Call{GetNode}{$\fx(\mnode_\mnext), \mney, -\mtdir_\mnext$}
    \ElsIf{$\mbest_\mnext\mdot\mcost > c_\mnext$}
        \Comment{Current path to $\mnode_\mnext$ has smaller ($-\mtdir_\mnext$)-cost than minimum at $\fpos(\mnode_\mnext)$.}
        \State $\mnode_\mnext \gets $ \Call{GetNode}{$\fx(\mnode_\mnext), \mnvy, -\mtdir_\mnext$}
        \State $\mbest_\mnext\mdot\mcost \gets \mcost_\mnext$
        \State $\mbest_\mnext\mdot\mnodebest \gets \mnode_\mnew$ 
        \State $\mbest_\mnext\mdot\mxbest \gets \fx(\mnode_\mprev)$
        \State \Call{ConvToExBranch}{$-\mtdir_\mnext, \mpos_\mnext$}
    \Else
        \Comment{Current path to $\mnode_\mnext$ has identical ($-\mtdir_\mnext$)-cost than minimum at $\fpos(\mnode_\mnext)$.}
        \If {$\mnode_\mnext\mdot\msidenode = \mside_\mbest$ \An $\mtdir_\mnext \mside_\mbest(\mv_\mbest \times \mv_\mathrm{test})  \le 0$}
            \State $\mbest_\mnext\mdot\mxbest \gets \fx(\mnode_\mprev)$
            \Comment{Does not matter if previous node is $\mney$.}
        \EndIf
        \State $\mnode_\mnext \gets $ \Call{GetNode}{$\fx(\mnode_\mnext), \mnode_\mprev\mdot\mntype, -\mtdir_\mnext$}
    \EndIf
    \State \Call{FinishReachedCast}{$\mtdir_\mnext, \mnode_\mnext, \mlink, \mray$}
    \State \Return ($\mtdir_\mnext, \mnode_\mnext$)
\EndFunction
\end{algorithmic}
\end{algorithm}

%%%%%%%%%%%%%%%%%%%%%%%%%%%%%%%% ALG: DiscardReachedCast %%%%%%%%%%%%%%%%%%%%%%%%%%%%%%%%%
\paragraph{\textproc{DiscardReachedCast}: \textit{Cumulative Visibility for Either Node of Reached Cast}}
The helper function \textproc{DiscardReachedCast} (Alg. \ref{suppalg:discardreachedcast}) discards a reached casting query and removes branches of links that only pass through the expanded link $\mlink$.
\begin{algorithm}[!ht]
\begin{algorithmic}[1]
\caption{Discard the reached casting query}
\label{suppalg:discardreachedcast}
\Function{DiscardReachedCast}{$\mlink$}
    \State $\mlink_S \gets \flink(S, \mlink)$
    \State \Call{Disconnect}{$T, \mlink_S, \mlink$}
    \State \Call{EraseTree}{$S, \mlink_S$}
    \State \Call{EraseTree}{$T, \mlink$}
\EndFunction
\end{algorithmic}
\end{algorithm}

%%%%%%%%%%%%%%%%%%%%%%%%%%%%%%%% ALG: FinishReachedCast %%%%%%%%%%%%%%%%%%%%%%%%%%%%%%%%%
\paragraph{\textproc{FinishReachedCast}: \textit{Final Steps for a Reached Cast}}
The helper function \textproc{FinishReachedCast} performs the final steps of a reached cast.
The next node of the cast is converted into a new node, and the expanded link is re-anchored on the new node.
The ray representing the cast is merged into the angular sector of the source node (the rays are stored in the expanded link) if the source node has cumulative visibility.
Overlapping queries at the next node are identified and pushed into the overlapping buffer.
\begin{algorithm}[!ht]
\begin{algorithmic}[1]
\caption{Final steps for a reached cast.}
\label{suppalg:finishreachedcast}
\Function{FinishReachedCast}{$\mtdir_\mnext, \mnode_\mnext, \mlink, \mray$}
    \State \Call{Anchor}{$\mlink, \mnode_\mnew$}
    \State $\mlink\mdot\mcost \gets $ \Call{Cost}{$\mlink$}
    \For {$\mlink_\mnext \in \flinks(\mtdir_\mnext, \mlink)$}
        \State \Call{Isolate}{$-\mtdir_\mnext, \mlink_\mnext, \mlink$}
    \EndFor
    \If {$\fnode(\flink(S, \mlink))\mdot\mntype \in \{\mnvy, \mney\}$} 
        \Comment{Merge sector-rays if source node is $\mnvy$ or $\mney$. $\mtdir_\mnext$ is $T$.}
        \State \Call{MergeRay}{$-\mnode_\mnew\mdot\msidenode, \mlink, \mray$}
        \For {$\mlink_\mnext \in \flinks(T, \mlink)$}
            \State \Call{MergeRay}{$\mnode_\mnew\mdot\msidenode, \mlink_\mnext, \mray$}
        \EndFor
    \EndIf
    
    \State $numLinksAll \gets $ sum of $ \lvert \mnode\mdot\mlinks  \rvert $ for all $\mnode \in \fpos(\mnode_\mnext)\mdot\mnodes$ \Comment{Number of links anchored at new node's corner.}
    \State $numLinks \gets 1 + \lvert \flinks(\mtdir_\mnext, \mlink) \rvert$
    \Comment{Number of links anchored at corner by the current casting query.}
    \If {$numLinksAll > numLinks$} 
        \Comment{Trigger overlap rule later if reached a node with other links.}
        \State \Call{PushOverlap}{$\fpos(\mnode_\mnext)$}
    \EndIf
\EndFunction
\end{algorithmic}
\end{algorithm}

%%%%%%%%%%%%%%%%%%%%%%%%%%%%%%%% ALG: QUEUEREACHEDCAST %%%%%%%%%%%%%%%%%%%%%%%%%%%%%%%%%
\paragraph{\textproc{QueueReachedCast}: \textit{Queue New Queries After a Reached Cast}}
The helper function \textproc{QueueReachedCast} (Alg. \ref{suppalg:queuereachedcast}) queues casting queries for links in the next direction (see Sec. \ref{suppsec:singlecumulativevisibility}).
\begin{algorithm}[!ht]
\begin{algorithmic}[1]
\caption{Queue subsequent castable queries}
\label{suppalg:queuereachedcast}
\Function{QueueReachedCast}{$\mtdir_\mnext, \mlink$}
    \For {$\mlink_\mnext \in \flinks(\mtdir_\mnext, \mlink)$}
        \Comment{Queue next casting queries.}
        \State \Call{Queue}{$\mqcast, \mlink\mdot c + \mlink_\mnext\mdot c, \mlink_\mnext$}
    \EndFor
\EndFunction
\end{algorithmic}
\end{algorithm}
\renewcommand{\thealgorithm}{\thesubsubsection} %.\arabic{algorithm}}

%%%%%%%%%%%%%%%%%%%%%%%%%%%%%%%%% ALG: CASTERCOLLIDED  %%%%%%%%%%%%%%%%%%%%%%%%%%%%%%%%%%%%%
\subsubsection{\textproc{CastCollided}: \textit{No Line-of-sight Between Nodes of Cast}}
The function \textproc{CastCollided} (Alg. \ref{suppalg:castcollided}) generates traces when a cast collides.
The \textbf{minor-trace} is generated from the collision point and has a side that is different from the source node of the cast.
If the target node is the goal node, and the source node is not a start node, the \textbf{third-trace} is generated from the source node.
The third-trace has a side that is the same as the source node.
The \textbf{major-trace} is generated from the collision point and has a side that is the same as the source node of the cast.

% Fig. \ref{suppfig:castcollided}.
\begin{algorithm}[!ht]
\begin{algorithmic}[1]
\caption{Cast collided and ray has no line-of-sight}
\label{suppalg:castcollided}
\Function{CastCollided}{$\mray, \mlink$}
    \State \Call{MinorTrace}{$\mray, \mlink$}
    \State \Call{ThirdTrace}{$\mray, \mlink$}
    \State \Call{MajorTrace}{$\mray, \mlink$}
\EndFunction
\end{algorithmic}
\end{algorithm}

%%%%%%%%%%%%%%%%%%%%%%%%%%%%%%%%% ALG: MINORTRACE  %%%%%%%%%%%%%%%%%%%%%%%%%%%%%%%%%%%%%
\paragraph{\textproc{MinorTrace}: \textit{Generates a Trace With Different Side From Source Node}}
When a cast collides, the helper function \textproc{MinorTrace} (Alg. \ref{suppalg:minortrace}) generates a trace from the collision point. The trace, called the \textit{minor-trace}, has a side that is different from the source node.
The ray of the collided cast is merged into the source node of the cast, via a trace-link.

The reader may choose to ignore the third-trace (Alg. \ref{suppalg:thirdtrace}) if the minor-trace traces back  and re-encounters the source node of the cast.
The default implementation of \rtwop{} introduces an additional variable to the Trace object $\mtrace$ to monitor a re-encounter. However, if the minor-trace is interrupted, the re-encounter can no longer be monitored.
The monitoring is not shown in the pseudocode, as the additional step is not necessary for \rtwop{} to work.
\renewcommand{\thealgorithm}{\theparagraph} %.\arabic{algorithm}}
\begin{algorithm}[!ht]
\begin{algorithmic}[1]
\caption{Minor-trace for collided cast}
\label{suppalg:minortrace}
\Function{MinorTrace}{$\mray, \mlink$}
    \State $\mnode_S \gets \fnode(\flink(S, \mlink))$
    \State $\mside_\mmnr \gets -\mnode_S\mdot\msidenode$
    \State $\mvray \gets \mray\mdot\mx_T - \mray\mdot\mx_S$
    \If {$\mnode_S\mdot\mntype = \mney$}
        \State \Return \Comment{Minor Trace cannot be generated for $\mney$ source node }
    \EndIf \Comment{The reader may choose to return if $\fxcol(\mside_\mmnr, \mray)$ is at the map boundary.}
    \State $\mtrace_\mmnr \gets $ \Call{CreateTrace}{$\fxcol(\mside_\mmnr, \mray), \mside_\mmnr$}
    \Comment{Create source nodelet and trace-link.}
    \State $\mtlink_\mathrm{newS} \gets $ \Call{CopyLink}{$\mlink, \mtrace_\mmnr\mdot\mtnode_S, \{S\}$}
    \State \Call{MergeRay}{$-\mside_\mmnr, \mtlink_\mathrm{newS}, \mray$}
    \State \Call{CreateNodelet}{$\mtlink_\mathrm{newS}, \mvray, \mback, \mtrace_\mmnr\mdot\mnlets_S$}
    \Comment{Create target nodelet and trace-link.}
    \State $\mtlink_\mathrm{newT} \gets $ \Call{CreateLink}{$\mtrace_\mmnr\mdot\mtnode_T$}
    \For {$\mlink_T \in \mlink\mdot\mlinks_T$}
        \State \Call{Connect}{$T, \mtlink_\mathrm{newT}, \mlink_T$}
    \EndFor
    \State \Call{CreateNodelet}{$\mtlink_\mathrm{newT}, -\mvray, \mback, \mtrace_\mmnr\mdot\mnlets_T$}
    \State \Call{Tracer}{$\mtrace_\mmnr$}
\EndFunction
\end{algorithmic}
\end{algorithm}

%%%%%%%%%%%%%%%%%%%%%%%%%%%%%%%%% ALG: THIRDTRACE  %%%%%%%%%%%%%%%%%%%%%%%%%%%%%%%%%%%%%
\paragraph{\textproc{ThirdTrace}: \textit{Generates a Third Trace}}
When a cast from a non-start point to the goal point collides, the helper function \textproc{ThirdTrace} (Alg. \ref{suppalg:thirdtrace}) generates a trace from the source node.
The trace is called the \textit{third-trace}, which has the same side as the source node.
The ray of the collided cast is merged into the source node of the cast, via a trace-link.
\begin{algorithm}[!ht]
\begin{algorithmic}[1]
\caption{Third-trace for collided cast}
\label{suppalg:thirdtrace}
\Function{ThirdTrace}{$\mray, \mlink$}
    \If {$\mray\mdot\mx_T \ne \mxgoal$ \Or $\mnode_S\mdot\mntype = \mney$}
        \State \Return 
        \Comment{Can ignore third-trace if minor-trace refinds source node...}
    \EndIf  
    \Comment{... or if $\mnode_S$ lies on an edge touching the map boundary.}
    \State $\mnode_S \gets \fnode(\flink(S, \mlink))$
    \State $\mside_\mthird \gets \mnode_S\mdot\msidenode$
    \State $\mvray \gets \mray\mdot\mx_T - \mray\mdot\mx_S$
    \State $\mtrace_\mthird \gets $ \Call{CreateTrace}{$\fx(\mnode_S), \mside_\mthird)$}
    \State $\mtrace_\mthird\mdot\mxtrace \gets $ \Call{Trace}{$\mtrace_\mthird\mdot\mxtrace,  \mtrace_\mthird\mdot\msidetrace$}
    \Comment{Third-trace begins at corner after source node.}
    
    \State $\mtlink_\mathrm{newS} \gets $ \Call{CopyLink}{$\mlink, \mtrace_\mthird\mdot\mnode_S, \{S\}$} \Comment{Create source nodelet and trace-link.}
    \State \Call{MergeRay}{$\mside_\mthird, \mtlink_\mathrm{newS}, \mray$}
    \State \Call{CreateNodelet}{$\mtlink_\mathrm{newS}, \mtrace_\mthird\mdot\mxtrace - \fx(\mnode_S), \mback, \mtrace_\mthird\mdot\mnlets_S$}

    \State $\mnode_\mathrm{newUn} \gets $ \Call{GetNode}{$\fx(\mnode_S), \mnun, \mside_\mthird, T$} 
    \Comment{Create a target-tree $\mnun$ node and link at casting point.}
    \State $\mlink_\mathrm{newUn} \gets $ \Call{CreateLink}{$\mnode_\mathrm{newUn}$}
    \For {$\mlink_T \in \mlink\mdot\mlinks_T$}
        \State \Call{Connect}{$T, \mlink_\mathrm{newUn}, \mlink_T$}
    \EndFor
    \State $\mlink_\mathrm{newUn}\mdot c \gets $ \Call{Cost}{$\mlink_\mathrm{newOc}$}

    \State $\mx_\mathrm{newOc} \gets $ \Call{Trace}{$\fx(\mnode_S), -\mside_\mthird$} 
    \Comment{Create a target-tree $\mnoc$ node and link at the corner before source node.}
    \State $\mnode_\mathrm{newOc} \gets $ \Call{GetNode}{$\mx_\mathrm{newOc}, \mnoc, \mside_\mthird, T$}
    \State $\mlink_\mathrm{newOc} \gets $ \Call{CreateLink}{$\mnode_\mathrm{newOc}$}
    \State \Call{Connect}{$T, \mlink_\mathrm{newOc}, \mlink_\mathrm{newUn}$}
    \State $\mlink_\mathrm{newOc}\mdot c \gets $ \Call{Cost}{$\mlink_\mathrm{newOc}$}

    \State $\mtlink_\mathrm{newT} \gets $ \Call{CreateLink}{$\mtrace_\mthird\mdot\mtnode_T$} 
    \Comment{Create target nodelet and trace-link.}
    \State \Call{Connect}{$T, \mtlink_\mathrm{newT}, \mlink_\mathrm{newOc}$}
    \State \Call{CreateNodelet}{$\mtlink_\mathrm{newT}, \mtrace_\mthird\mdot\mxtrace - \mx_\mathrm{newOc}, \mback, \mtrace_\mthird\mdot\mnlets_T$}

    \State \Call{Tracer}{$\mtrace_\mthird$}
\EndFunction
\end{algorithmic}
\end{algorithm}

%%%%%%%%%%%%%%%%%%%%%%%%%%%%%%%%% ALG: MAJORTRACE  %%%%%%%%%%%%%%%%%%%%%%%%%%%%%%%%%%%%%
\paragraph{\textproc{MajorTrace}: \textit{Generates a Trace with Same Side as Source Node}}
When a cast collides, the helper function \textproc{MajorTrace} generates a trace from the collision.
The trace, called the \textit{major-trace}, has the same side as the source node.
The ray of the collided cast is merged into the source node of the cast, via a trace-link.
\begin{algorithm}[!ht]
\begin{algorithmic}[1]
\caption{Major-trace for collided cast}
\label{suppalg:majortrace}
\Function{MajorTrace}{$\mray, \mlink$}
    \State $\mnode_S \gets \fnode(\flink(S, \mlink))$
    \State $\mside_\mmjr \gets \mnode_S\mdot\msidenode$
    \State $\mvray \gets \mray\mdot\mx_T - \mray\mdot\mx_S$
    \State $\mtrace_\mmjr \gets $ \Call{CreateTrace}{$\fxcol(\mside_\mmjr, \mray)$}
    
    \State \Call{Anchor}{$\mlink, \mtrace_\mmjr\mdot\mtnode_S$}
    \Comment{$\mlink$ becomes source trace-link; Create source nodelet.}
    \State \Call{MergeRay}{$-\mside_\mmjr, \mlink, \mray$}
    \State \Call{CreateNodelet}{$\mlink, \mvray, \mback, \mtrace_\mmjr\mdot\mnlets_S$}

    \State $\mtlink_\mathrm{newT} \gets $ \Call{CreateLink}{$\mtrace_\mmjr\mdot\mnode_T$}
    \Comment{Create target nodelet and trace-links.}
    \For{$\mlink_T \in \mlink\mdot\mlinks_T$}
        \State \Call{Disconnect}{$T, \mlink, \mlink_T$} 
        \Comment{Connections in $\mlink$ prevents its removal if the mnr. or thd. traces fail.}
        \State \Call{Connect}{$T, \mtlink_\mathrm{newT}, \mlink_T$}
    \EndFor
    \State \Call{CreateNodelet}{$\mtlink_\mathrm{newT}, -\mvray, \mback, \mtrace_\mmjr\mdot\mnlets_T$}

    \State \Call{Tracer}{$\mtrace_\mmjr$}
\EndFunction
\end{algorithmic}
\end{algorithm}
\renewcommand{\thealgorithm}{\thesubsubsection} %.\arabic{algorithm}}

%% file: tex_supp_tracer.tex
%%%%%%%%%%%%%%%%%%%%%%%%%%%%%%%%% ALG: TRACER  %%%%%%%%%%%%%%%%%%%%%%%%%%%%%%%%%%%%%
\newpage
\subsection{\textproc{Tracer}: \textit{Tracing Query Functions}} 
The main tracing function \textproc{Tracer} (Alg. \ref{suppalg:tracer}) implements the tracing query.
While 
There is only \textit{one} source nodelet during a trace, at at least one target nodelet.
The tracing query makes use of ordered set of nodelets $\mnlets_S$ and $\mnlets_T$ to prune, place and examine nodes.
Each nodelet contains a trace-link which connects a trace-node to a parent node, allowing the algorithm to examine a parent node with respect to the trace.
The source nodelet examines a source node, while a target nodelet examines a target node.
The source node lies on the source-tree, while the target node lies on the target-tree.

% Fig. \ref{suppfig:tracer}.

Each iteration in \textproc{Tracer} evaluates a corner at $\mxtrace$.
If the trace traces to the source node (\textproc{RefoundSrc}, Alg. \ref{suppalg:refoundsrc}), the query is discarded.
If the trace does not progress with respect to the source node, \textproc{SrcProgCast} (Alg. \ref{suppalg:srcprogcast}) is called and a casting query is generated from the source node if the angular progression decreased by more than $180^\circ$ with respect to the source node. 
Otherwise, if the progression decreases by less than $180^\circ$, the trace moves to the next corner without evaluating the nodes.

If the trace has progressed with respect to the source node, the source and target nodes are examined further.
\textproc{TracerProc} (Alg. \ref{suppalg:tracerproc}) evaluates the target nodes with the progression, occupied-sector, and pruning rules.
\textproc{TracerProc} subsequently evaluates the source node with the angular-sector, occupied-sector, and pruning rules, which generates a recursive trace and discards the current trace if necessary.
\textproc{TracerProc} for the target nodes is called before the source node to ensure that the path is taut when a recursive trace is generated.

The interrupt rule interrupts the trace in \textproc{InterruptRule} (Alg. \ref{suppalg:interrupt}) if several corners are traced.
The placement rule in \textproc{PlaceRule} (Alg. \ref{suppalg:place}) tries to place a turning point or phantom point at the current position, and casts to any potentially visible target node from a placed turning point.
The placement rule engages the overlap rule instead if the query crossed paths with other queries. 

% If the angular progression of the trace with respect to the source node has reversed by more than $180^\circ$.
% The trace stops and a casting query is generated by \textproc{SrcProgCast} (Alg. \ref{suppalg:srcprogcast}).
% If the trace is not progressed with respect to the source node, the trace proceeds to the next corner.
% The target nodes can be ignored because all target nodes placed within the convex hull of an obstacle are pruned when the trace exits the hull \cite{bib:r2}.

% If the winding counter in the source nodelet increases to two, the angular progression of the trace with respect to the source node has reversed by more than $180^\circ$. 
% The function \textproc{SrcProgCast} (Alg. \ref{suppalg:srcprogcast}) is subsequently triggered, which generates a casting query between the source node and the non-convex corner where the angular progression is maximum.
% By ensuring that traces can never reverse by more than $180^\circ$ with respect to the source node, the angular progressions with respect to the source and target nodes are at a maximum when the trace resumes angular progression with respect to the source node.

% If the trace has progressed with respect to the source node, the trace subsequently engages the progression rule, occupied-sector rule, and pruning rule \texit{for each target node} in the function \textproc{TracerProc} (Alg. \ref{suppalg:tracerproc}).
% The same function is su

\renewcommand{\thealgorithm}{\thesubsection} %.\arabic{algorithm}}
\begin{algorithm}[!ht]
\begin{algorithmic}[1]
\caption{Tracer}
\label{suppalg:tracer}
\Function{Tracer}{$\mtrace$}
    \DoWhile
    \Comment{There is only one source link, and $\mtrace\mdot\mnlets_S = \{\mnlet_S\}$}
        \State $\mtrace\mdot\mnumcrns \gets \mtrace\mdot\mnumcrns + 1$
        % \State $\mnlet_S \gets $ the only trace-node in $\mtrace\mdot\mnlets_S$
        \If {\Call{RefoundSrc}{$\mtrace$}}
            \Comment{Refound source node}
            \State \Break
        \ElsIf {\Call{ProgRule}{$\mtrace, \fnlet(S, \mtrace)$}}
            \Comment{Not prog. w.r.t. source node.}
            \If {\Call{SrcProgCast}{$\mtrace$}}
                \Comment{Prog. reversed by $>180^\circ$ w.r.t. source node.}
                \State \Break
            \EndIf
        \Else
            \Comment{Prog. w.r.t. source node.}
            \If {\Call{TracerProc}{$T, \mtrace$}} \Comment{Order of execution is important.}
                \State \Break
            \ElsIf {\Call{TracerProc}{$S, \mtrace$}}
                \State \Break
            \ElsIf {\Call{InterruptRule}{$\mtrace$}}
                \State \Break
            \ElsIf {\Call{PlaceRule}{$\mtrace$}}
                \State \Break
            \EndIf
        \EndIf
        \State $\mtrace\mdot\mxtrace \gets $ \Call{Trace}{$\mtrace\mdot\mxtrace, \mtrace\mdot\msidetrace$}
    \EndDoWhile{$\mtrace\mdot\mxtrace \ne \varnothing$}
    \Comment{While trace is in map.}

    \For {$\mtdirpar \in \{S, T\}$} \Comment{Delete all dangling links.}
        \For {$\mnlet \in \fnlets(\mtdirpar, \mtrace)$}
            \State \Call{EraseTree}{$\mtdirpar, \mnlet\mdot\mtlink$}
        \EndFor
    \EndFor

\EndFunction
\end{algorithmic}
\end{algorithm}
\renewcommand{\thealgorithm}{\thesubsubsection} %.\arabic{algorithm}}

%%%%%%%%%%%%%%%%%%%%%%%%%%%%%%%%% ALG: TRACERFROMLINK %%%%%%%%%%%%%%%%%%%%%%%%%%%%%%%%%%%%%
\subsubsection{\textproc{TracerFromLink}: \textit{Wrapper for} \textproc{Tracer}} \label{suppsec:tracerfromlink}
The tracing function \textproc{TracerFromLink} (Alg. \ref{suppalg:tracerfromlink}) wraps \textproc{Tracer} (Alg. \ref{suppalg:tracer}). 
\textproc{TracerFromLink} prepares a tracing query from a link $\mlink$ and calls \textproc{Tracer}.
The function is called when a tracing query is polled from the open-list (Alg. \ref{suppalg:run}), 
or when a $\mntm$-node is reached by a cast (Alg. \ref{suppalg:reachedtm}). 
\begin{algorithm}[!ht]
\begin{algorithmic}[1]
\caption{Prepare a tracing query from a link.}
\label{suppalg:tracerfromlink}
\Function{TracerFromLink}{$\mlink$}
    \State $\mnode \gets \fnode(\mlink)$
    \State $\mnode_S \gets \fnode(\flink(S, \mlink))$
    \State $\mtrace \gets $ \Call{Tracer}{$\fx(\mnode), \mnode\mdot\msidenode$}
    \State \Call{Anchor}{$\mlink, \mtrace\mdot\mnode_S$}
    \State \Call{CreateNodelet}{$\mlink, \fx(\mnode) - \fx(\mnode_S), \mback, \mtrace\mdot\mnlets_S$}

    \For {$\mlink_T \in \mlink\mdot\mlinks_T$}
        \State $\mnode_T \gets \fnode(\flink(T, \mlink))$
        \State \Call{Disconnect}{$T, \mlink, \mlink_T$}
        \State \Call{Anchor}{$\mlink_T, \mtrace\mdot\mnode_T$}
        \State \Call{CreateNodelet}{$\mlink_T, \fx(\mnode) - \fx(\mnode_T), \mback, \mtrace\mdot\mnlets_T$}
    \EndFor

    \State \Call{Tracer}{$\mtrace$}
\EndFunction
\end{algorithmic}
\end{algorithm}

%%%%%%%%%%%%%%%%%%%%%%%%%%%%%%%%% ALG: TRACERPROC  %%%%%%%%%%%%%%%%%%%%%%%%%%%%%%%%%%%%%
\subsubsection{\textproc{TracerProc}: \textit{Process Source or Target-tree Nodes in Trace}}
Function \textproc{TracerProc} (\ref{suppalg:tracerproc}) examines the source node or target nodes. 
The function relies on the ordered set of nodelets $\mnlets_S$ and $\mnlets_T$ stored in the Trace object $\mtrace$ to examine the nodes.
If the trace is discarded or interrupted, either set becomes empty and the function returns.
While there can be multiple target nodelets at all times during a trace, there is always only \textit{one} source nodelet.

When \textproc{TracerProc} is called for the source node ($\mtdirpar=S$), the source node is first examined by the angular-sector rule. 
If the trace is no longer within the angular-sector of the node and the trace is unable to continue, the trace is interrupted or discarded.
Otherwise, the pruning rule or occupied-sector rule is checked depending on the side of the source node. If the trace has the same side as the source node, the pruning rule examines the source node. Otherwise, the occupied-sector rule examines the source node. 
The occupied-sector rule may interrupt the current trace with a recursive occupied-sector trace.

When \textproc{TracerProc} is called for the target nodes ($\mtdirpar=T$), the function checks the angular progression with respect to each target node before proceeding to the occupied-sector rule or pruning rule.
As the recursive occupied-sector trace for a target node ends quickly, the current trace will not be interrupted when the target nodes are examined.

\begin{algorithm}[!ht]
\begin{algorithmic}[1]
\caption{Process Source or Target-tree Nodes}
\label{suppalg:tracerproc}
\Function{TracerProc}{$\mtdirpar, \mtrace=(\mxtrace, \msidetrace, \cdots)$}
    \For {$\mnlet \in \fnlets(\mtdirpar, \mtrace)$}
        \State $\mnodepar \gets \fnode(\mnlet\mdot\mtlink)$
        \State $\msidepar \gets \mnodepar\mdot\msidenode$
        \If {$\mtdirpar = T$ \An \Call{ProgRule}{$\mtrace, \mnlet$}}
            \State \Continue
        \ElsIf {$\mtdirpar = S$ \An \Call{AngSecRule}{$\mtrace, \mnlet$}        }
            \State \Continue 
        \ElsIf {$\fx(\mnodepar) \in \{\mx_\mstart, \mx_\mgoal \}$} 
            \Comment{Parent node is start or goal node.}
            \State \Continue
        \ElsIf {$\msidepar = \msidetrace$}
            \State \Call{PruneRule}{$\mtrace, \mnlet$}
        \Else 
            \State \Call{OcSecRule}{$\mtrace, \mnlet$}
        \EndIf
    \EndFor
    \State \Return $\fnlets(\mtdirpar, \mtrace) = \{\}$
    \Comment{Returns true if the $\mtdirpar$ set of nodelets is empty.}
\EndFunction
\end{algorithmic}
\end{algorithm}

%%%%%%%%%%%%%%%%%%%%%%%%%%%%%%%%% ALG: SRCPROGCAST  %%%%%%%%%%%%%%%%%%%%%%%%%%%%%%%%%%%%%
\subsubsection{\textproc{SrcProgCast}: \textit{Cast From Highly Winded Source Node}}
The \textproc{SrcProgCast} function (Alg. \ref{suppalg:srcprogcast}) interrupts the trace when the winding counter increases to two for the source node. When the counter increases to two, the progression has decreased by more than $180^\circ$ from the maximum angular progression.
The maximum angular progression points to a phantom point, which is replaced with an unreachable $\mnun$-node.
A casting query is subsequently queued from the source node to the $\mnun$-node.

The phantom point can be found directly from the only target trace-node when the trace is not progressed with respect to the source node.
The point can be found as no target nodes are placed when the angular progression is reversing for the source node, and the trace is in the convex hull of a non-convex obstacle.
No target nodes need to be placed in a convex hull as they would be pruned once the trace exits the convex hull and resumes angular progression with respect to the source node.
% Fig. \ref{suppfig:srcprogcast}
\begin{algorithm}[!ht]
\begin{algorithmic}[1]
\caption{Cast when progression w.r.t. source node is reversed more than $180^\circ$}
\label{suppalg:srcprogcast}
\Function{SrcProgCast}{$\mtrace = (\mxtrace, \cdots)$}
    \State $\mnlet_S \gets \fnlet(S, \mtrace)$
    \If {$\mnlet_S\mdot\mcprog > 1$}
    
        \Comment{(A) Create new target-tree $\mnun$-node and anchor target link there.}
        \State $\mnlet_T \gets \fnlet(T, \mtrace)$ 
        \Comment{The only target nodelet.}
        \State $\mlink_T \gets \flink(T, \mnlet_T\mdot\mtlink)$
        \State $\mnode_T \gets \fnode(\mlink_T)$ 
        \Comment{Target node is a phantom point.}
        \State $\mnode_\mathrm{newUn} \gets $ \Call{GetNode}{$\fx(\mnode_T), \mnun, \mnode_T\mdot\msidenode, T$}
        \State \Call{Anchor}{$\mlink_T, \mnode_\mathrm{newUn}$}

        \Comment{(B) Create new target-tree $\mnvu$-node and re-anchor source trace-link there.}
        \State $\mlink_\mathrm{newVu} \gets \mnlet_S\mdot\mtlink$
        \State $\mlink_S \gets \flink(S, \mlink_\mathrm{newVu})$
        \State $\mnode_S \gets \fnode(\mlink_S)$
        \State $\mnode_\mathrm{newVu} \gets $ \Call{GetNode}{$\fx(\mnode_S), \mnvu, \mnode_S\mdot\msidenode, T$}
        \State \Call{Anchor}{$\mlink_\mathrm{newVu}, \mnode_\mathrm{newVu}$}
        \State \Call{Connect}{$T, \mlink_\mathrm{newVu}, \mlink_T$}
        \State $\mlink_\mathrm{newVu}\mdot\mcost \gets $ \Call{Cost}{$\mlink_\mathrm{newVu}$}

        \Comment{(C) Queue casting query for source trace-link.}
        \State \Call{Queue}{$\mqcast, \mlink_S\mdot\mcost + \mlink_\mathrm{newVu}\mdot\mcost, \mlink_\mathrm{newVu}$}
        \State \Return $\mtrue$
    \EndIf
    \State \Return $\mfalse$
\EndFunction
\end{algorithmic}
\end{algorithm}

%%%%%%%%%%%%%%%%%%%%%%%%%%%%%%%%% ALG: REFOUNDSRC  %%%%%%%%%%%%%%%%%%%%%%%%%%%%%%%%%%%%%
\subsubsection{\textproc{RefoundSrc}: \textit{Checks if Trace Traces to Source Node}}
Function \textproc{RefoundSrc} (Alg. \ref{suppalg:refoundsrc}) returns $\mtrue$ if a trace traces to its examined source node.
A trace can never trace to a target node in \rtwop.
\begin{algorithm}[!ht]
\begin{algorithmic}[1]
\caption{Refound Source Node}
\label{suppalg:refoundsrc}

\Function{RefoundSrc}{$\mtrace = (\mxtrace, \cdots)$}
    \State $\mnlet_S \gets \fnlet(S, \mtrace)$
    \State $\mnode_S \gets \fnode(\mnlet_S\mdot\mtlink)$
    \State \Return $\mxtrace = \fx(\mnode_S)$
\EndFunction
\end{algorithmic}
\end{algorithm}

%%%%%%%%%%%%%%%%%%%%%%%%%%%%%%%%% ALG: PROGRULE  %%%%%%%%%%%%%%%%%%%%%%%%%%%%%%%%%%%%%
\clearpage
\subsubsection{\textproc{ProgRule}: \textit{Checks Angular Progression w.r.t. a Node}}
Function \textproc{ProgRule}  (Alg. \ref{suppalg:progrule}) returns $\mfalse$ if the trace at corner $\mxtrace$ has progressed with respect to a parent node or $\mtrue$ if the trace has not progressed. 

If a trace crosses the progression ray, it will intersect a line (unbounded on both sides) coincident to the ray (bounded on one side). 
The progression ray is drawn from a parent node in the direction $\mvprog$ and stretches to infinity. If a trace crossed the front of the ray, the trace intersected the ray.
Let $\mvprog'$ be a ray that is drawn from a parent node in the direction $-\mvprog$. If a trace crossed behind the ray, then the trace crossed $\mvprog'$.

As only the direction of the intersection with respect to the parent node needs to be known, the intersection need not be calculated explicitly. 
Let $itx =  \fsgn(\mvprev \times \mv_\mpar) \fsgn(\mvprev \times \mvprog ) \in {-1, 0, 1}$. 
$\mvprev$ is the previous trace direction that reached $\mx$. $\mv_\mpar$ is the vector pointing from the parent node to $\mx$. $\mvprog$ is the progression ray.
$-1$ and $1$ indicate that the trace has crossed behind and in front of the ray respectively. $0$ indicates that $\mvprev$ is parallel to $\mv_\mpar$ and/or $\mvprev$ is parallel to $\mvprog$. 
For both parallel cases, the reader may verify that the trace is always progressed with respect to the parent node.
Note that $\mvprog \ne 0$, $\mvprev \ne 0$ and $\mvprog \ne 0$ when $itx$ is evaluated.
\begin{algorithm}[!ht]
\begin{algorithmic}[1]
\caption{Progression Rule}
\label{suppalg:progrule}
\Function{ProgRule}{$\mtrace = (\mxtrace, \msidetrace, \cdots)$, $\mnlet = (\mtlink, \mvprog, \mcprog) $}
% \State $(\mtdirpar, \msidepar, \mnodepar, \mvpar, \mvprev, \sim) \gets $\Call{TraceCalc}{$\mtrace, \mnlet$}
    \State $\mnodepar \gets \fnode(\mtlink)$
    \State $\mtdirpar \gets \mnodepar\mdot\mtdirnode$ 
    \State $\mvpar \gets \mxtrace - \fx(\mnodepar)$
    \State $\mvprev \gets -$\Call{GetEdge}{$\mxtrace, -\msidetrace$} 
\If{ $\mvpar = 0$} 
    \State $\mvpar \gets $\Call{Bisect}{$\mxtrace$} \Comment{Occurs for start nodes or checkerboard corners}
\EndIf
\State $u \gets \mtdirpar \msidetrace (\mvpar \times \mvprog)$
\State $isProg \gets u < 0 $ \Or ($u = 0$ \An $\mvpar \sdot \mvprog > 0$) \Comment{$\mtrue$ if $\mvpar$ lies to $(\mtdirpar \msidetrace)$-side of $\mvprog$, or if both are pointing in same direction }
\State $wasProg \gets \mcprog = 0$

\State $itx \gets \fsgn(\mvprev \times \mvpar) \fsgn(\mvprev \times \mvprog )$ \Comment{$itx = 1$:  crossed in front of ray; $itx =-1$: crossed behind ray.}
\If {$isProg = \mtrue$ \An $wasProg = \mtrue$}
    \State $\mnlet\mdot\mvprog \gets \mvpar$ \Comment{Update ray}
\ElsIf {$isProg = \mtrue$ \An $wasProg = \mfalse$}
    \If {$itx \ge 0$}
        \State $\mnlet\mdot\mcprog \gets \mcprog - 1$ \Comment{Unwind}
        \If {$\mcprog > 0$}  \Comment{Not completely unwound, flip ray}
            \State $\mnlet\mdot\mvprog \gets -\mvprog$ 
            \State $isProg \gets \mfalse$
        \Else \Comment{Unwound completely, update ray}
            \State $\mnlet\mdot\mvprog \gets \mvpar$ 
        \EndIf
    \Else  \Comment{Wind and flip ray}
        \State $\mnlet\mdot\mcprog \gets \mcprog + 1$
        \State $\mnlet\mdot\mvprog \gets -\mvprog$ 
        \State $isProg \gets \mfalse$
    \EndIf
\ElsIf{$isProg = \mfalse$ \An $wasProg = \mtrue$} 
    \If {$itx \le 0$} \Comment{No winding: case occurs if start node lies on an obstacle's edge}
        \State $\mnlet\mdot\mvprog \gets -\mvprog$
        \State $isProg \gets \mtrue$
    \Else  \Comment {Wind}
        \State $\mnlet\mdot\mcprog \gets \mcprog + 1$ 
    \EndIf
\EndIf
\State \Return \Not $isProg$
\EndFunction
\end{algorithmic}
\end{algorithm}

%%%%%%%%%%%%%%%%%%%%%%%%%%%%%%%%% ALG: PRUNERULE  %%%%%%%%%%%%%%%%%%%%%%%%%%%%%%%%%%%%%
\clearpage
\subsubsection{\textproc{PruneRule}: \textit{Checks if a Node can be Pruned}}
Function \textproc{PruneRule} (Alg. \ref{suppalg:prune}) prunes a node if the resulting path at the examined corner at $\mxtrace$ is not taut. A source-tree node that is $\mnvy$ cannot be pruned by the rule, as it is handled by the angular-sector rule.

Pruning of source-tree $\mnvy$-nodes are handled by the angular-sector. 
As the angular-sector of a source-tree $\mnvy$-node can only be fully evaluated at the next corner, any pruning at the current corner can cause \rtwop{} to be incomplete.

\begin{algorithm}[!ht]
\begin{algorithmic}[1]
\caption{Pruning Rule}
\label{suppalg:prune}
\Function{PruneRule}{$\mtrace = (\mxtrace, \msidetrace, \cdots)$, $\mnlet = (\mtlink, \cdots) $}
    \State $\mnodepar \gets \fnode(\mtlink)$
    \State $\mtdirpar \gets \mnodepar\mdot\mtdirnode$ 
    \State $\msidepar \gets \mnodepar\mdot\msidenode$
    \State $\mvpar \gets \mxtrace - \fx(\mnodepar)$

    \Comment{\textbf{(A) $\mnvy$ source-tree nodes are handled by the angular sector rule.}}
\If{$\mtdirpar = S$ \An $\mnodepar \mdot \mntype = \mnvy$}
    \State \Return 
\EndIf 

    \Comment{(B) Try pruning for each parent node and link.}
\For{$\mlinkpar \in \flinks(\mtdirpar, \mtlink)$}

    \Comment{(B.1) Taut: go to next parent node and link.}
    \State $\mnodegpar \gets \fnode(\flink(\mtdirpar, \mlinkpar))$
    \State $\mvgpar \gets \fx(\mnodepar) - \fx(\mnodegpar)$  
    \If {\Call{IsTaut}{$\mtdirpar$, $\msidepar$, $\mvpar$, $\mvgpar$}}
        \State \Continue
        \Comment{$(\mxtrace, \mnodepar), \mnodegpar)$ is taut and nothing is pruned.}
    \EndIf

    \Comment{(B.2) Not taut: prune parent node.}
    \State \Call{Disconnect}{$-\mtdirpar$, $\mlinkpar$, $\mtlink$} 
    \State $\mlink_\mnew \gets $ \Call{Isolate}{$-\mtdirpar$, $\mlinkpar$, $\varnothing$, $\ftnode(\mtdirpar, \mtrace)$}
    \If{$\mtdirpar = T$} 
        \Comment{Remove rays if prune in target direction.}
        \State $\mlink_\mnew \mdot \mray_L \gets \varnothing$
        \State $\mlink_\mnew \mdot \mray_R \gets \varnothing$
    \EndIf

    % \Comment{(B.3) Not taut: If new parent node is target-tree $\mnph$ or $\mnun$, re-anchor new parent links to $\mntm$ parent node.}
    % \If{$\mnode_\mpar \mdot \mntype \in \{\mntm,\mnun\}$ \An $\mnode_\mgpar \mdot \mntype = \mnph$} 
    % % \Comment{Move links from old gpar. $\mnph$ nodes to $\mntm$ if old par. node is $\mntm$.} 
    %     \State $\mnode_\mathrm{newT} \gets $ \Call{GetNode}{$\fx(\mnode_\mgpar)$,  $\mnode_\mpar\mdot\mntype$, ($\mnode_\mgpar \mdot \msidenode$), $\mtdirpar$}
    %     \For {$\mlink_\mnph \in \flinks_\mpar(\mlink_\mpar)$} 
    %         \State $\mlink_\mathrm{newT} \gets $ \Call{Isolate}{$-\mtdirpar$, $\mlink_\mnph$, $\mlink_\mnew$, $\mnode_\mathrm{newT}$} 
    %     \EndFor
    % \EndIf

    \Comment{(B.4) Create new nodelet for new parent node.}
    \State \Call{CreateNodelet}{$\mlink_\mnew, \mx_\mtrace - \fx(\mnode_\mgpar), \mback, \fnlets(\mtdirpar, \mtrace)$} 
    \Comment{Push to back to try pruning later.}
        
\EndFor

        \Comment{(C) Remove current nodelet and trace-link if trace-link has no more parent links.}
\If {$\flinks(\mtdirpar, \mtlink)  = \{\}$ } 
    \State Remove $\mtlink$ from $\ftnode(\mtdirpar, \mtrace)\mdot\mlinks_\mnode$ 
    \State Remove $\mnlet$ from $\fnlets(\mtdirpar, \mtrace)$
\EndIf
\EndFunction
\end{algorithmic}
\end{algorithm}

%%%%%%%%%%%%%%%%%%%%%%%%%%%%%%%%% ALG: PRUNERULE  %%%%%%%%%%%%%%%%%%%%%%%%%%%%%%%%%%%%%
\clearpage
\paragraph{\textproc{IsTaut}: \textit{Checks if a Node is Taut}}
The helper function \textproc{IsTaut} (Alg. \ref{suppalg:istaut}) evaluates the path segments around a node and returns $\mtrue$ if the node can be pruned.
\renewcommand{\thealgorithm}{\theparagraph} %.\arabic{algorithm}}
\begin{algorithm}[!ht]
\begin{algorithmic}[1]
\caption{Is Taut}
\label{suppalg:istaut}
\Function{IsTaut}{$\mtdirpar$, $\msidepar$, $\mvpar$, $\mvgpar$}
    \State $u \gets \mtdirpar \msidepar (\mvpar \times \mvgpar)$ 
    \If {$u = 0$} \Comment{Return $\mtrue$ if angle between $\mvpar$ and $\mvgpar$ is $180^\circ$, $\mfalse$ if $0^\circ$}
        \State \Return $\mvpar \sdot \mvgpar \ge 0$ 
    \Else \Comment{Returns $\mtrue$ if $\mvgpar$ lies to the $\mtdirpar \msidepar$-side of $\mvpar$ }
        \State \Return $u > 0$
    \EndIf
\EndFunction
\end{algorithmic}
\end{algorithm}
\renewcommand{\thealgorithm}{\thesubsubsection} %.\arabic{algorithm}}

%%%%%%%%%%%%%%%%%%%%%%%%%%%%%%%%% ALG: OCSECRULE  %%%%%%%%%%%%%%%%%%%%%%%%%%%%%%%%%%%%%
\clearpage
\subsubsection{\textproc{OcSecRule}: \textit{Checks if a Node can be Pruned}}
The function \textproc{OcSecRule} (Alg. \ref{suppalg:ocsecrule}) implements the occupied-sector rule.
The function returns immediately if the trace is outside the occupied-sector of parent node.
If the trace enters the occupied-sector for the source node, a recursive occupied-sector trace is generated, and the current trace is interrupted.

If the trace enters the occupied-sector of a target node that is $\mnoc$ type, the trace is discarded because the path has looped.
Otherwise, if the target is a non-$\mnoc$ node, a $\mnoc$ target-tree node is placed in the source-direction of the target node.
The placement can be considered part of a short recursive occupied-sector trace.

% Fig. \ref{suppfig:ocsecrule}
\begin{algorithm}[!ht]
\begin{algorithmic}[1]
\caption{Occupied-Sector Rule}
\label{suppalg:ocsecrule}
\Function{OcSecRule}{$\mtrace = (\mx_\mtrace, \mside_\mtrace, \cdots)$, $\mnlet = (\mtlink, \cdots) $}
    % \State $(\mtdirpar, \msidepar, \mnodepar, \mvpar, \mvprev, \sim) \gets $\Call{TraceCalc}{$\mtrace, \mnlet$}
    \State $\mnodepar \gets \fnode(\mtlink)$
    \State $\mtdirpar \gets \mnodepar\mdot\mtdirnode$ 
    \State $\msidepar \gets \mnodepar\mdot\msidenode$
    \State $\mvpar \gets \mxtrace - \fx(\mnodepar)$
    \State $\mvocpar \gets $ \Call{GetEdge}{$\fx(\mnodepar), -\mtdirpar\msidepar$}

    \Comment{(A) Parent is target $\mnoc$-node. Discard trace if it enters oc.sec., return otherwise.}
    \If {$\mnodepar \mdot \mntype = \mnoc$} 
        \Comment{Previously entered occupied-sector of a target-tree node (gpar. node)}
        \State $\mnodegpar \gets \fnode(\flink(T, \flink(T, \mtlink)))$ 
        \State $\mvgpar \gets \fx(\mnodepar) - \fx(\mnodegpar)$  
        \If {$\mside_\mpar (\mv_\mpar \times \mv_\mgpar) \le 0$}
            \Comment{Discard as path has looped around gpar. node and par. $\mnoc$ node} 
            \State \Call{EraseTree}{$T, \mtlink$} 
            \State Remove $\mnlet$ from  $\mtrace\mdot\mnlets_T$
        \EndIf
        \State \Return % $\mfalse$
    
    \Comment{(B) Return if not in occupied-sector.}
    \ElsIf{$\mtdirpar \mside_\mpar (\mv_\mpar \times \mvocpar) \le 0$} 
        \State \Return % $\mfalse$
        \Comment{Not in occupied-sector of $\mnode_\mpar$}
    \EndIf

    \Comment{(C) Create recursive oc-sec. trace if entered oc-sec. of source node.} 
    \If{$\mtdirpar = S$}
        \State $\mtrace_\mnew \gets $ \Call{CreateTrace}{$\fx(\mnodepar), \msidepar$}
        \State \Call{Anchor}{$\mtlink$, $\mtrace_\mnew\mdot\mnode_S$} 
        \State Move $\mnlet$ to $\mtrace_\mnew\mdot\mnlets_S$ \Comment{s.t. $\mtrace_\mnew\mdot\mnlets_S = \{\mnlet\}$ and $\mtrace\mdot\mnlets_S = \{\}$}
        \State $\mnlet\mdot\mvprog \gets \mvocpar$

        \State $\mnode_\mathrm{newTm} \gets$ \Call{GetNode}{$\mxtrace, \mntm, \msidetrace, T$}
        \Comment{Create a $\mntm$ target-tree node and re-anchor target trace-links there.}
        \State $\mtlink_\mathrm{newT} \gets $ \Call{CreateLink}{$\mtrace_\mnew\mdot\mnode_T$}
        \State \Call{CreateNodelet}{$\mtlink_\mathrm{newT}, -\mvpar, \mback, \mtrace_\mnew \mdot \mnlets_T$}

        \For{$\mnlet_T \in \mtrace\mdot\mnlets_T$} 
            \State $\mlink_\mathrm{newT} \gets \mnlet_T\mdot\mtlink$
            \State \Call{Anchor}{$\mlink_\mathrm{newT}$, $\mnode_\mathrm{newTm}$}
            \State $\mlink_\mathrm{newT}\mdot\mcost \gets$ \Call{Cost}{$\mlink_T$}
            \State \Call{Connect}{$T$, $\mtlink_\mathrm{newT}$, $\mlink_\mathrm{newT}$}
        \EndFor

        \State $\mtrace_\mnew\mdot\mxtrace \gets $ \Call{Trace}{$\mtrace_\mnew\mdot\mxtrace$, $\mtrace_\mnew\mdot\msidetrace$}
        \Comment{Start recursive trace from the $\msidepar$-side corner of the source node}
        \State \Call{Tracer}{$\mtrace_\mnew$} 

    \Comment{(D) Place target-tree $\mnoc$-node if entered oc-sec. of target node.} 
    \Else 
        \State $\mx_\mathrm{newOc} \gets $ \Call{Trace}{$\fx(\mnodepar)$, $-\msidepar$} 
        % \Comment{Get the $-\msidepar$ corner of the target node.}
        \State $\mnode_\mathrm{newOc} \gets$ \Call{GetNode}{$\mx_\mathrm{newOc}, \mnoc, \msidepar, T$} 
        % \Comment{Get $\mnoc$ target-tree node at $-\msidepar$ corner.}
        \State \Call{Anchor}{$\mtlink$, $\mnode_\mathrm{newOc}$} 
        % \Comment{Transfer expanded link to $\mnoc$ node.}
        \State $\mtlink\mdot\mcost \gets $ \Call{Cost}{$\mtlink$}
        \State $\mtlink_\mathrm{newT} \gets $ \Call{CreateLink}{$\mtrace\mdot\mtnode_T$}
        \State \Call{Connect}{$T, \mtlink_\mathrm{newT}, \mtlink$}
        \State $\mnlet\mdot\mtlink \gets \mtlink_\mathrm{newT}$
        \State $\mnlet\mdot\mvprog \gets \mvpar$
    \EndIf
\EndFunction
\end{algorithmic}
\end{algorithm}

%%%%%%%%%%%%%%%%%%%%%%%%%%%%%%%%% ALG: ANGSEC  %%%%%%%%%%%%%%%%%%%%%%%%%%%%%%%%%%%%%
\subsubsection{\textproc{AngSecRule}: \textit{Implements the Angular-Sector Rule}}
The function \textproc{AngSecRule} (Alg. \ref{suppalg:angsecrule}) implements the angular-sector rule for the source node.

Four helper functions breaks the implementation down into smaller parts. They are \textproc{RayNotCrossed} (Alg. \ref{suppalg:raynotcrossed}), \textproc{AngSecPrune} (Alg. \ref{suppalg:angsecprune}), \textproc{Project} (Alg. \ref{suppalg:project}), and \textproc{RecurAngSecTrace} (Alg. \ref{suppalg:recurangsectrace}).
\textproc{AngSecRule} examines the $\msidetrace$-sided ray $\mray$ of the source node's angular sector, and determines if the trace is within the angular-sector.
If the trace is within the angular-sector, \textproc{AngSecRule} returns without calling the other helper functions.

The other helper functions are called when the trace exits the source node's angular sector.
\textproc{AngSecPrune} prunes the source node if the resulting path is not taut after exiting the source node's angular sector.
% The pruned source node is guaranteed to be $\mnvy$ type. 
% A prune by the angular-sector rule occurs only when the $\mray$ ends at the source node, indicating that the node is a source-tree node with cumulative visibility ($\mnvy$ or $\mney$ type). $\mney$-nodes are pruned by the pruning rule, as no recursive angular-sector trace can be generated for an $\mney$ node.
% $\mney$ node. A recursive angular-sector trace cannot occur for an $\mney$ node as the $\mney$ node cannot no longer be pruned by future queries.
If a prune by the angular-sector occurs, the current trace continues regardless of any subsequent recursive angular-sector trace.
\textproc{Project} (Alg. \ref{suppalg:project}) projects the ray $\mray$ and finds the projected ray's collision point.
\textproc{RecurAngSecTrace} generates a recursive angular-sector trace if the current trace did not pass through the projected ray's collision point.
% Fig. \ref{suppfig:angsecrule}.
\begin{algorithm}[!ht]
\begin{algorithmic}[1]
\caption{Angular-Sector Rule}
\label{suppalg:angsecrule}
\Function{AngSecRule}{$\mtrace = (\mxtrace, \msidetrace, \cdots)$, $\mnlet = (\mtlink, \cdots) $}
    \State $\mray \gets \fray(\msidetrace, \mtlink)$

    \If {\Call{RayNotCrossed}{$\mray, \mtrace, \mnlet$}}
        \State \Return $\mfalse$
    \EndIf
    \State \Call{AngSecPrune}{$\mray, \mtrace, \mnlet$}
    \State \Call{Project}{$\mray$}
    \State \Call{RecurAngSecTrace}{$\mray, \mtrace, \mnlet$}
    \State \Return $\mtrue$
    
\EndFunction
\end{algorithmic}
\end{algorithm}

%%%%%%%%%%%%%%%%%%%%%%%%%%%%% ALG:RAYNOTCROSSED %%%%%%%%%%%%%%%%%%%%%%%%%%%%%%%%%%%%
\paragraph{\textproc{RayNotCrossed}: \textit{Checks if Trace is Within Angular-Sector}}
The function \textproc{RayNotCrossed} examines the $\msidetrace$-sided ray $\mray$ of the source node and compares the ray with the current position of the trace.
If the trace crosses the ray, the function returns $\mtrue$.

If the current position lies on the ray, it is not clear if the ray has been crossed. 
The function breaks ties by considering how the line-of-sight function \textproc{LOS} (Alg. \ref{suppalg:los}) returns the first corners after a collision. 
\textproc{LOS} relies on the directional vector bisecting the corner at the current position to place the first corners. 
The consideration is required because the first corners are subsequently checked in \textproc{RecurAngSecTrace} -- the ray can only be considered crossed if the $\msidetrace$-sided first corner lies at the current position.
However, the ray may not have been projected, and the algorithm can only check against the bisecting vector.

To eliminate errors resulting from floating point calculations, the algorithm is designed to handle only discrete, integer calculations. This helps the authors to verify proofs, and provides a foundation from which \rtwop{} can be improved.
The reader may choose to implement a much simpler intersection check that uses floating point calculations instead.

\renewcommand{\thealgorithm}{\theparagraph} %.\arabic{algorithm}}
\begin{algorithm}[!ht]
\begin{algorithmic}[1]
\caption{Angular-Sector Rule: Check if sector-ray is crossed.}
\label{suppalg:raynotcrossed}
\Function{RayNotCrossed}{$\mray, \mtrace = (\mxtrace, \msidetrace, \cdots)$, $\mnlet = (\mtlink, \cdots) $}
    \If {$\mray = \varnothing$} 
        \State \Return $\mtrue$
    \EndIf
    \State $\mvray \gets \mray\mdot\mx_T - \mray\mdot\mx_S$
    \State $\mlink_S \gets \flink(S, \mtlink)$ \Comment{$\mlink$ has only one source link.}
    \State $\mvpar \gets \mxtrace - \fx(\fnode(\mlink_S))$
    
    \State $u_\mathrm{ray,par} \gets \msidetrace (\mvray \times \mvpar)$
    \If{$u_\mathrm{ray,par} > 0$}
        \State \Return $\mtrue$
        \Comment{Return as ray is not crossed.}
    \ElsIf{$u_\mathrm{ray,par} = 0$} 
        \Comment{$\mvray$ and $\mvpar$ are parallel. Use bisecting oc-sec. vector to break ties.}
        \State $\mv_\mathrm{crn} \gets $ \Call{Bisect}{$\mxtrace$} 
        \State $u_\mathrm{ray,crn} \gets \msidetrace(\mvray \times \mv_\mathrm{crn})$ 
        \If {$u_\mathrm{ray,crn} > 0$ \Or $u_\mathrm{ray,crn} = 0$ \An $\mvray \sdot \mv_\mathrm{crn} > 0$} 
            \State \Return $\mtrue$
            \Comment{Ray is not crossed: $\mvray$ lies strictly to the $\msidetrace$-side of $\mv_\mathrm{crn}$, or...}
        \EndIf
        \Comment{...ray points away from corner: $\mvray$ opposite to $\mv_\mathrm{crn}$.}
    \EndIf 
    \State \Return $\mfalse$.
\EndFunction
\end{algorithmic}
\end{algorithm}

%%%%%%%%%%%%%%%%%%%%%%%%%%%%% ALG:ANGSECPRUNE %%%%%%%%%%%%%%%%%%%%%%%%%%%%%%%%%%%%
\paragraph{\textproc{AngSecPrune}: \textit{Tries to Prune Source Node}}
The function \textproc{AngSecPrune} (Alg. \ref{suppalg:angsecprune}) tries to prune the source node if the resulting path is not taut after exiting the node's angular sector.
If the node cannot be pruned, nothing is done.
When a prune occurs, the current trace continues regardless of a subsequent recursive angular-sector trace.

The pruned source node is guaranteed to be $\mnvy$ type. 
A prune by the angular-sector rule occurs only when the $\mray$ ends at the source node, indicating that the node is a source-tree node with cumulative visibility ($\mnvy$ or $\mney$ type). $\mney$-nodes are pruned by the pruning rule, as no recursive angular-sector trace can be generated for an $\mney$ node.
$\mney$ node. A recursive angular-sector trace cannot occur for an $\mney$ node as the $\mney$ node can no longer be pruned by future queries.
\begin{algorithm}[!ht]
\begin{algorithmic}[1]
\caption{Angular-Sector Rule: Tries to prune a $\mnvy$ node after exiting its angular sector.}
\label{suppalg:angsecprune}
\Function{AngSecPrune}{$\mray, \mtrace = (\mxtrace, \msidetrace, \cdots)$, $\mnlet = (\mtlink, \cdots) $}
    \State $\mlink_S \gets \flink(S, \mtlink)$
    \State $\mnode_S \gets \fnode(\mlink_S)$
    
    \State $prunable \gets \mray\mdot\mx_T = \fx(\mnode_S)$ \Comment{Prunable if sector-ray ends at parent source-tree node, or...}
    \State $prunable \gets prunable $ \Or ($\mray\mdot\mx_T = \mxstart$ \An $\mray\mdot\mx_S = \mxgoal$ \An $\msidetrace = \mnode_S\mdot\msidenode$) \Comment{...prunable if ray is special sector-ray for start node and the side traced is same as start node's side.}
    \If {$prunable$} 
        \State $\mtlink_\mathrm{newS} \gets $ \Call{Isolate}{$T, \mlink_S, \varnothing, \mtrace, \mtrace\mdot\mtnode_S$}
        \State $\mnode_\mathrm{newS} \gets \fnode(\flink(S, \mtlink_\mathrm{newS}))$
        \State \Call{CreateNodelet}{$\mtlink_\mathrm{newS}, \mxtrace - \fx(\mnode_\mathrm{newS}), \mback, \mtrace\mdot\mnlets_S$}
    \EndIf
\EndFunction
\end{algorithmic}
\end{algorithm}

%%%%%%%%%%%%%%%%%%%%%%%%%%%%% ALG:ANGSECPRUNE %%%%%%%%%%%%%%%%%%%%%%%%%%%%%%%%%%%%
\paragraph{\textproc{RecurAngSecTrace}: \textit{Tries to Generate a Recursive Angular-sector Trace}}
The function \textproc{RecurAngSecTrace} generates a recursive angular-sector trace if the trace crosses the collision point of the projected ray $\mray$.
The trace crosses the collision point if the current position $\mxtrace$ lies on the first $\msidetrace$-sided  corner from the collision point.

A $(-\msidetrace)$-sided $\mnun$ node is placed at the first $\msidetrace$-side corner from the collision point. The node ensures that the $(-\msidetrace)$-sided recursive trace progresses with respect to a target node at the initial corner traced.
The $\mnun$ node will be pruned at the initial corner if the resulting path is not taut.
A $\mntm$ node is placed at the current trace location, and is the target node of the $\mnun$ node.
Additionally the $\mnun$ node is unreachable because if a query from the recursive trace reaches the node, a cheaper path has to exist that can reach the $\mntm$ node.

% Fig. \ref{suppfig:recurangsectrace}
\begin{algorithm}[!ht]
\begin{algorithmic}[1]
\caption{Angular-Sector Rule: Check if a recursive trace is required}
\label{suppalg:recurangsectrace}
\Function{RecurAngSecTrace}{$\mray, \mtrace = (\mxtrace, \msidetrace, \cdots)$, $\mnlet = (\mtlink, \cdots) $}
    \State $\mnode_S \gets \fnode(\flink(S, \mtlink))$
    \If {$\fxcol(\msidetrace, \mray) = \mxtrace$ \Or $\mnodepar\mdot\mntype = \mney$}
    \Comment{Discard if trace crossed the collision point of the sector-ray, or...}
        \State \Call{EraseTree}{$S, \mtlink$}
        \Comment{...parent source-tree node is $\mney$ (recur. trace has different side from $\mney$ node).}
        \State Remove $\mnlet$ from $\mtrace\mdot\mnlets_S$
    \Else   
        \Comment{Trace did not cross sector-ray's collision point. Recursive ang-sec. trace.}
        \State $\mtrace_\mnew \gets $ \Call{CreateTrace}{$\fxcol(-\msidetrace, \mray), -\msidetrace$}
        \State $\mnlet\mdot\mvprog \gets \mtrace_\mnew\mdot\mxtrace - \fx(\mnode_S)$
        
        \State $\mx_\mathrm{newUn} \gets \fxcol(\msidetrace, \mray)$ 
        \Comment{Create a $\mnun$ target-tree node at $\msidetrace$ corner of collision.}
        \If {$\mx_\mathrm{newUn} = \fxcol(-\msidetrace, \mray)$} 
            \Comment{Collision at corner instead of edge.}
            \State $\mx_\mathrm{newUn} \gets $ \Call{Trace}{$\mx_\mathrm{newUn}, \msidetrace$}
        \EndIf
        \State $\mnode_\mathrm{newUn} \gets $ \Call{GetNode}{$\mx_\mathrm{newUn}, \mnun, -\msidetrace, T$}
        \State $\mlink_\mathrm{newT} \gets $ \Call{CreateLink}{$\mnode_\mathrm{newUn}$}
        
        \State $\mnode_\mathrm{newTm} \gets$ \Call{GetNode}{$\mxtrace, \mntm, \msidetrace, T$} 
        \Comment{Create a $\mntm$ target-tree node at current corner at $\mxtrace$}
        \For {$\mnlet_T \in \mtrace\mdot\mnlets_T$}
            \State $\mlink_\mathrm{newTT} \gets $ \Call{CopyLink}{$\mnlet_T\mdot\mtlink, \mnode_\mathrm{newTm}, \{T\}$} 
             \Comment{Re-anchor target trace-links from current trace at $\mntm$ node.}
            \State $\mlink_\mathrm{newTT}\mdot\mcost \gets$ \Call{Cost}{$\mlink_\mathrm{newTT}$}
            \State \Call{Connect}{$T$, $\mlink_\mathrm{newT}, \mlink_\mathrm{newTT}$}
        \EndFor
        \State $\mlink_\mathrm{newT}\mdot\mcost \gets$ \Call{Cost}{$\mlink_\mathrm{newT}$}

        \State $\mtlink_\mathrm{newT} \gets $ \Call{CreateLink}{$\mtrace_\mnew\mdot\mtnode_T$} 
        \State \Call{Connect}{$T, \mtlink_\mathrm{newT}, \mlink_\mathrm{newT}$}
        \State \Call{CreateNodelet}{$\mtlink_\mathrm{newT}, \mtrace_\mnew\mdot\mxtrace - \mx_\mathrm{newUn}, \mfront, \mtrace_\mnew\mdot\mnlets_T$}

        \State \Call{Tracer}{$\mtrace_\mnew$}
    \EndIf
\EndFunction
\end{algorithmic}
\end{algorithm}
\renewcommand{\thealgorithm}{\thesubsection} %.\arabic{algorithm}}

%%%%%%%%%%%%%%%%%%%%%%%%%%%%%%%%% ALG: INTERRUPT RULE  %%%%%%%%%%%%%%%%%%%%%%%%%%%%%%%%%%%%%
\clearpage
\subsubsection{Interrupt Rule} \label{suppsec:interrupt}
Function \textproc{InterruptRule} (Alg. \ref{suppalg:interrupt}) implements the interrupt rule, which interrupts and queues a trace when a number of corners are traced and if the trace is progressed with respect to all source-tree and target-tree parent nodes.

The number of corners to interrupt $numInterrupt$ is arbitrary. 
The larger $numInterrupt$ is, the slower \rtwop{} may get for simple queries with few turning points as the tracing queries spend more time on non-convex contours. 
The smaller $numInterrupt$ is, the slower \rtwop{} may be for more complex queries with many turning points due to frequent open-list queuing. The default value is $numInterrupt = 10$.
\begin{algorithm}[!ht]
\begin{algorithmic}[1]
\caption{Interrupt Rule}
\label{suppalg:interrupt}
\Function{InterruptRule}{$\mtrace = (\mxtrace, \msidetrace, \cdots)$}
    
    \Comment{(A) Return if trace has only checked a few corners, and trace is not prog. w.r.t. all nodes.}
    \State $allProgS \gets $ $\mnlet_S \mdot \mcprog = 0 $ for all $ \mnlet_S \in \mtrace\mdot\mnlets_S$
    \State $allProgT \gets $ $\mnlet_T \mdot \mcprog = 0 $ for all $ \mnlet_T \in \mtrace\mdot\mnlets_T$
    \If{$\mtrace\mdot\mnumcrns < numInterrupt$ \Or $allProgS = \mfalse$ \Or $allProgT = \mfalse$}
        \State \Return $\mfalse$
    \EndIf

    \Comment{(B) Otherwise, interrupt and create nodes and re-anchor trace-links.}
    \State $\mnode_\mathrm{newSVu} \gets$ \Call{GetNode}{$\mxtrace, \mnvu, \msidetrace, S$}
    \State $\mlink_\mathrm{newS} \gets$ $\fnlet(S, \mtrace)\mdot\mtlink$ 
    \State \Call{Anchor}{$\mlink_\mathrm{newS}, \mnode_\mathrm{newSVu}$}
    \State $\mlink_\mathrm{newS}\mdot\mcost \gets$ \Call{Cost}{$\mlink_\mathrm{newS}$}
    
    \State $\mnode_\mathrm{newTTm} \gets$ \Call{GetNode}{$\mxtrace, \mntm, \msidetrace, T$}
    \For{$\mnlet_T \in \mtrace\mdot\mnlets_T$}
        \State $\mlink_\mathrm{newT} \gets \mnlet_T\mdot\mtlink$
        \State \Call{Anchor}{$\mlink_\mathrm{newT}, \mnode_\mathrm{newTTm}$}
        \State $\mlink_\mathrm{newT}\mdot\mcost \gets$ \Call{Cost}{$\mlink_\mathrm{newT}$}
        \State \Call{Connect}{$T, \mlink_\mathrm{newS}, \mlink_\mathrm{newT}$}
    \EndFor

    \Comment{(C) Mark for overlap-rule check, or queue a tracing query}
    \If {$\mtrace\mdot\moverlap = \mtrue$}
        \State \Call{PushOverlap}{\Call{GetPos}($\mxtrace$)}
    \Else
        \State $f \gets \mlink_\mathrm{newS}\mdot\mcost + $ \Call{GetMinCost}{$T, \mlink_\mathrm{newS}$}
        \State \Call{Queue}{$\mqtrace, f, \mlink_\mathrm{newS} $}
    \EndIf
    \State \Return $\mtrue$
\EndFunction
\end{algorithmic}
\end{algorithm}

%%%%%%%%%%%%%%%%%%%%%%%%%%%%%%%%% ALG: PLACERULE  %%%%%%%%%%%%%%%%%%%%%%%%%%%%%%%%%%%%%
\clearpage
\subsubsection{\textproc{PlaceRule}: Places Points and Checks Castability}
\textproc{PlaceRule} implements the placement rule and tries to place a turning point or phantom point.
The function exits after a phantom point is placed, or if no point can be placed.

If a turning point is placed, the function examines if the trace has crossed a path from a different query.
If a different path is crossed, the overlap rule will check the placed links after the trace ends.
The trace will end with the placement rule if all target nodes can be cast to.
% If there are other links, the trace has crossed a the trace is marked for an overlap-rule check that engages after all target nodes become castable.
% The rule subsequently tries to cast to each target node, queuing a casting query for each castable node.

The function is aided by four helper functions, \textproc{IsRev} (Alg. \ref{suppalg:isrev}), \textproc{IsVis} (Alg. \ref{suppalg:isvis}), \textproc{PlaceNode} (Alg. \ref{suppalg:placenode}), and \textproc{CastFromTrace} (Alg. \ref{suppalg:castfromtrace}).

\begin{algorithm}[!ht]
\begin{algorithmic}[1]
\caption{Placement rule and casting.}
\label{suppalg:place}
\Function{PlaceRule}{$\mtrace = (\mxtrace, \msidetrace, \cdots)$}
    \State $\mtdirpar \gets S$ if corner at $\mxtrace$ is convex, $T$ otherwise

    \State $\mnlet_\mnew \gets $ \Call{PlaceNode}{$\mtdirpar, \mtrace$}

    \If {$\mnlet_\mnew = \varnothing$ \Or $\mtdirpar = T$}
        \State \Return $\mfalse$
        \Comment{Turning point not placed. Nothing else to do.}
    \ElsIf {$\sum \lvert \mnode\mdot\mlinks_\mnode \rvert > 1$ for all $\mnode \in $ \Call{GetPos}{$\mxtrace$}$\mdot\mnodes$ }
        \State $\mtrace\mdot\moverlap \gets \mtrue$
        \Comment{Mark for overlap rule if multiple links are anchored at $\mxtrace$.}
    \EndIf

    \If {$\mnlet_\mnew \ne \varnothing$ \An $\mtdirpar = S$}
        \Comment{Try casting if a turning point is placed.}
        \State \Call{CastFromTrace}{$\mtrace, \mnlet_\mnew$}
    \EndIf

    \State \Return $\mtrace\mdot\mnlets_T = \{\}$ 
    \Comment{Return $\mtrue$ if no more target trace-nodes.}
\EndFunction
\end{algorithmic}
\end{algorithm}

%%%%%%%%%%%%%%%%%%%%%%%%%%%%%%%%% ALG: ISREV  %%%%%%%%%%%%%%%%%%%%%%%%%%%%%%%%%%%%%
\paragraph{\textproc{IsRev}: \textit{Checks if Angular Progression Will Reverse}}
The helper function \textproc{IsRev} (Alg.\ref{suppalg:isrev}) returns $\mtrue$ if tracing the next edge from the current corner at $\mxtrace$ causes the angular progression to reverse with respect to the parent node.
\begin{algorithm}[!ht]
\begin{algorithmic}[1]
\caption{Checks if angular progression reverses at next edge.}
\label{suppalg:isrev}
\Function{IsRev}{$\mtdirpar, \msidetrace, \mvpar, \mvnext$}
    \State \Return $\mtdirpar \msidetrace (\mvpar \times \mvnext) < 0$
\EndFunction
\end{algorithmic}
\end{algorithm}

%%%%%%%%%%%%%%%%%%%%%%%%%%%%%%%%% ALG: ISVIS %%%%%%%%%%%%%%%%%%%%%%%%%%%%%%%%%%%%%
\paragraph{\textproc{IsVis}: \textit{Checks if a Target Node is Castable}}
The helper function \textproc{IsVis} (Alg. \ref{suppalg:isvis}) returns $\mtrue$ if a target-tree node is potentially visible from the traced position and can be cast to.
\renewcommand{\thealgorithm}{\theparagraph} %.\arabic{algorithm}}
\begin{algorithm}[!ht]
\begin{algorithmic}[1]
\caption{Checks if a target node is potentially visible and castable.}
\label{suppalg:isvis}
\Function{IsVis}{$\msidetrace, \mvpar, \mvnext$}  
    \Comment{Par. node is target-tree node.}
    \State \Return $\msidetrace (\mvpar \times \mvnext) \le 0$ \Comment{Assumes that trace has progressed w.r.t. target-tree node.}
\EndFunction
\end{algorithmic}
\end{algorithm}

%%%%%%%%%%%%%%%%%%%%%%%%%%%%%%%%% ALG: PLACENODE  %%%%%%%%%%%%%%%%%%%%%%%%%%%%%%%%%%%%%
\paragraph{\textproc{PlaceNode}: \textit{Tries to Place a Node}}
The helper function \textproc{PlaceNode} (Alg. \ref{suppalg:placenode}) tries to place a turning point or phantom point.
A turning point ($\mnvu$ or $\mneu$-node) is placed if the current corner is convex, and the angular progression will reverse with respect to the source node.
A phantom point ($\mntm$-node) is placed if the current corner is non-convex, and the angular progression reverses with respect to at least one target node.
\begin{algorithm}[!ht]
\begin{algorithmic}[1]
\caption{Places a $\mntm$, $\mnvu$ or $\mneu$ node if possible.}
\label{suppalg:placenode}
\Function{PlaceNode}{$\mtdirpar, \mtrace = (\mxtrace, \msidetrace, \cdots)$}
    \State $\mvnext \gets $ \Call{GetEdge}{$\mxtrace, \msidetrace$}
    \State $\mnlet_\mnew = \varnothing$
    \For{$\mnlet \in \fnlets(\mtdirpar, \mtrace)$}
        % \State $(\sim, \sim, \mnodepar, \mvpar, \sim, \mvnext) \gets $\Call{TraceCalc}{$\mtrace, \mnlet$}
        \State $\mnodepar \gets \flink(\mnlet\mdot\mtlink)$
        \State $\mvpar \gets \mxtrace - \fx(\mnodepar)$
        
        \Comment{(A) Do not place if not prog. w.r.t. par. node, or ang. prog. increases at next edge.}
        \If {$\mcprog > 0$ \Or \Call{IsRev}{$\mtdirpar, \msidetrace, \mvpar, \mvnext$} = $\mfalse$}
            \State \Continue
        \EndIf

        \Comment{(B) Otherwise, place a new phantom ($\mntm$) or turning point ($\mneu$ or $\mnvu$).}
        \State $\mlink_\mathrm{newPar} \gets \mnlet\mdot\mtlink$
        \If {$\mnlet_\mnew = \varnothing$}
            \Comment{Create the new node and reuse current nodelet if new node not yet placed.}
            \State $\mnlet_\mnew \gets \mnlet$
            \State $\mntype_\mathrm{newPar} \gets \mnvu$
            \Comment{Determine the new node type.}
            \If {$\mtdirpar = T$}
                \State $\mntype_\mathrm{newPar} \gets \mntm$
            \ElsIf{$\mnodepar\mdot\mntype \in \{\mneu, \mney\}$}
                \State $\mntype_\mathrm{newPar} \gets \mneu$
            \EndIf
            \Comment{Create the new node and reuse current nodelet.}
            \State $\mnode_\mathrm{newPar} \gets $ \Call{GetNode}{$\mxtrace, \mntype_\mathrm{newPar}, \msidetrace, \mtdirpar$}
            \State $\mnlet_\mnew\mdot\mtlink \gets$ \Call{CreateLink}{$\ftnode(\mtdirpar, \mtrace)$}
            \State $\mnlet_\mnew\mdot\mvprog \gets \mvnext$
        \Else
            \Comment{Delete the nodelet if new node has been placed.}
            \State Remove $\mnlet$ from $\fnlets(\mtdirpar, \mtrace)$
        \EndIf

        \State \Call{Anchor}{$\mlink_\mathrm{newPar}, \mnode_\mathrm{newPar}$}
        \State $\mlink_\mathrm{newPar}\mdot c \gets$ \Call{Cost}{$\mlink_\mathrm{newPar}$}
        \State \Call{Connect}{$\mtdirpar, \mnlet_\mnew\mdot\mtlink, \mlink_\mathrm{newPar}$}
    \EndFor
    \State \Return $\mnlet_\mnew$
\EndFunction
\end{algorithmic}
\end{algorithm}

%%%%%%%%%%%%%%%%%%%%%%%%%%%%%%%%% ALG: CastFromTrace  %%%%%%%%%%%%%%%%%%%%%%%%%%%%%%%%%%%%%
\paragraph{\textproc{CastFromTrace}: \textit{Tries to Cast From a Placed Turning Point.}}
The helper function \textproc{CastFromTrace} tries to cast to each target node after a turning point is placed at the current corner.
If castable, a casting query is queued for each target trace-link. The target trace-link is re-anchored on a target-tree $\mnvu$ node.
However, if the trace had crossed a path from a different query, it would have been marked for an overlap-rule check in \textproc{PlaceRule} (Alg. \ref{suppalg:place}). 
No new casting queries are queued by the trace, as the overlap-rule will queue a new casting query in the most recent source-tree node with cumulative visibility.
\begin{algorithm}[!ht]
\begin{algorithmic}[1]
\caption{Try casting to all target links after placing a turning point.}
\label{suppalg:castfromtrace}
\Function{CastFromTrace}{$\mtrace=(\mxtrace, \msidetrace, \cdots), \mnlet_\mnew$}
    \State $\mvnext \gets $ \Call{GetEdge}{$\mxtrace, \msidetrace$}
    \State $\mlink_S \gets \flink(S, \mnlet_\mnew\mdot\mtlink)$
    \State $\mnode_S \gets \mnode(\mlink_S)$  
    \State $\mnode_\mathrm{newVu} \gets \varnothing$
    \For {$\mnlet_T \in \mtrace\mdot\mnlets_T$}

        \Comment{(A) For each target node, check if it is castable from new turning point at $\mx$.}
        % \State $(\sim, \sim, \mnodepar, \mvpar, \sim, \mvnext) \gets $\Call{TraceCalc}{$\mtrace, \mnlet_T$}
        \State $\mlink_\mathrm{newT} \gets \mnlet_T\mdot\mtlink$
        \State $\mnode_T \gets \fnode(\flink(T, \mlink_\mathrm{newT}))$
        \State $\mvpar \gets \mxtrace - \fx(\mnode_T)$
        % \If{$\mnlet_T\mdot\mcprog > 0$ \Or $\mnode_T\mdot\mntype = \mnph$ \Or \Call{IsVis}{$\msidetrace, \mvpar, \mvnext$} $= \mfalse$}
        %     \State \Continue 
        %     \Comment{Ignore target-tree node if not prog. w.r.t. node, node is $\mnph$, or node is not visible.}
        % \EndIf
        \If {\Call{IsVis}{$\msidetrace, \mvpar, \mvnext$} $= \mfalse$}
            \State \Continue 
            \Comment{Skip non-castable tgt node.}
        \EndIf
        
        \Comment{(B) Target node is castable, re-anchor target trace-link to target-tree $\mnvu$ node.}
        \If {$\mnode_\mathrm{newVu} = \varnothing$}
            \State $\mnode_\mathrm{newVu} \gets $ \Call{GetNode}{$\mxtrace, \mnvu, \msidetrace, T$}
        \EndIf
        \State \Call{Anchor}{$\mlink_\mathrm{newT}, \mnode_\mathrm{newVu}$}
        \State $\mlink_\mathrm{newT}\mdot\mcost \gets $ \Call{Cost}{$\mlink_\mathrm{newT}$}
        \State \Call{Connect}{$T, \mlink_S, \mlink_\mathrm{newT}$}

        \Comment{(C.1) Mark for overlap-rule if source node is expensive or trace crossed paths with other queries.}
        \If {$\mtrace\mdot\moverlap = \mtrue$ \Or $\mnode_S\mdot\mntype = \mneu$}
            \State $\mtrace\mdot\moverlap \gets \mfalse$
            \State \Call{PushOverlap}{\textproc{GetPos}($\mxtrace$)}
        
        \Comment{(C.2) Otherwise, queue a casting query.}
        \Else
            \State $f \gets \mlink_S\mdot\mcost + \mlink_\mathrm{newT}\mdot\mcost$ 
            \State \Call{Queue}{$\mqcast, f, \mlink_T'$}
        \EndIf
        \State Remove $\mnlet_T$ from $\mtrace\mdot\mnlets_T$
    \EndFor
\EndFunction

\end{algorithmic}
\end{algorithm}
\renewcommand{\thealgorithm}{\thesubsubsection} %.\arabic{algorithm}}

%% file: tex_supp_overlap.tex
\subsection{Overlap Rule} \label{suppsec:overlap}
The overlap rule reduces the number of queries by verifying line-of-sight and cost-to-come for nodes in overlapping paths.

The overlap rule consists of two major components implemented by the functions \textproc{ShrinkSourceTree} (Alg. \ref{suppalg:shrinksourcetree}) and \textproc{ConvToExBranch} (Alg. \ref{suppalg:convtoexbranch}). 
Both components that allow \rtwop{} to identify and discard some expensive queries. 
The \textproc{ShrinkSourceTree} function shifts overlapping queries from the leaf nodes to the first source-tree node with cumulative visibility ($\mney$ or $\mnvy$ node), shrinking the source-tree and expanding the target-tree in the process. 
By shifting the queries, \textproc{ShrinkSourceTree} allow \rtwop{} to verify the cost-to-come of overlapping nodes.
The \textproc{ConvToExBranch} (Alg. \ref{suppalg:convtoexbranch}) marks existing nodes at a location $\mx$ as expensive when a cast reaches $\mx$ with a shorter path than the existing paths.

% processes overlapping queries after a queued query is processed (Alg. \ref{suppalg:run}).

% Fig. \ref{suppfig:overlaprule}

%%%%%%%%%%%%%%%%%%%%%%%%%%%%%%%%% ALG: SHRINKSOURCETREE  %%%%%%%%%%%%%%%%%%%%%%%%%%%%%%%%%%%%%
\subsubsection{\textproc{ShrinkSourceTree}: \textit{Shifts Queries to Verify Cost-to-come}} 
The \textproc{ShrinkSourceTree} function (Alg. \ref{suppalg:shrinksourcetree}) shifts overlapping queries to a source-tree node along their path, simultaneously shrinking the source-tree and expanding the target-tree. 
The source-tree node is a $\mnvy$ or $\mney$ node that is furthest along the examined path from the start node.
Let $\mnode_\mathrm{SY}$ represent the source-tree $\mnvy$ or $\mney$ node.

Overlaps are identified in a trace or reached cast. 
In a trace, overlaps are identified when a turning point ($\mneu$ or $\mnvu$ node) is placed. 
If the corner at the turning point contain nodes that anchor other links, an overlap is identified (see \textproc{PlaceNode}, Alg. \ref{suppalg:placenode}). 
Likewise, in a reached cast, an overlap is identified at the source or target node if there are other links at the node's corner (see \textproc{FinishReachedCast}, Alg. \ref{suppalg:finishreachedcast}).
Once overlaps are identified, the corners where the overlaps occur are pushed in to the overlap-buffer.
They corners are subsequently processed by \textproc{ShrinkSourceTree} at the end of an iteration in \textproc{Run} (Alg. \ref{suppalg:run}).

The reader may notice that the corner pushed into the overlap-buffer for the tracing query may not be the point where overlap occurs (see \textproc{CastFromTrace}, Alg. \ref{suppalg:castfromtrace}). 
While this may cause some overlapping queries to be missed, this reduces the size of the overlap buffer, and subsequent casts will still identify the missed queries.
For simplicity, the reader may choose to push the corner directly into the overlap-buffer in \textproc{PlaceNode} instead.

For each corner in the overlap-buffer, \textproc{ShrinkSourceTree} identifies source-tree $\mneu$ and $\mnvu$ nodes.
For each link anchored at the nodes, the path in the source direction is searched, and $\mnode_\mathrm{SY}$ is identified.
The function \textproc{ConvToTgtBranch} is called from the $\mnode_\mathrm{SY}$, converting all source-tree nodes in the target direction of $\mnode_\mathrm{SY}$ to target-tree $\mnvu$ nodes.
Any query in the target links are removed.
A casting query is subsequently queued from the $\mnode_\mathrm{SY}$ node.
\begin{algorithm}[!ht]
\begin{algorithmic}[1]
\caption{Process overlapping queries by shrinking the source-tree and requeue at $\mnvy$ source-tree node.}
\label{suppalg:shrinksourcetree}
\Function{ShrinkSourceTree}{\null}
    \For {$\mpos \in $ overlap-buffer}
        \For {$\mnode \in \mpos\mdot\mnodes$}
            \If {$\mnode\mdot\mtdirnode = T$ \Or $\mnode\mdot\mntype \notin \{\mneu, \mnvu\}$}
                \State \Continue
            \EndIf
            \For {$\mlink \in \mnode\mdot\mlinks_\mnode$}
                \State $\mlink_i \gets \mlink$
                \State $\mlink_{SY} \gets \flink(\mlink_i)$
                \While {$\fnode(\mlink_S)\mdot\mntype \notin \{\mnvy, \mney\}$}
                    \State $\mlink_i \gets \mlink_S$
                    \State $\mlink_{SY} \gets \flink(S, \mlink_{SY})$
                \EndWhile
                \State \Call{ConvToTgtBranch}{$\mlink_i$}
                \State \Call{Queue}{$\mqcast, \mlink_{SY}\mdot\mcost + \mlink_i\mdot\mcost, \mlink_i$}
            \EndFor
        \EndFor
    \EndFor
\EndFunction
\end{algorithmic}
\end{algorithm}

%%%%%%%%%%%%%%%%%%%%%%%%%%%%%%%%%%%%%%%%%% ALG: CONVTOTGTBRANCH %%%%%%%%%%%%%%%%%%%%%%%%%%%%%%%%%%
\paragraph{\textproc{ConvToTgtBranch}: \textit{Converts Source-tree Nodes to Target-tree Nodes}} 
The function \textproc{ConvToTgtBranch} (Alg. \ref{suppalg:convtotgtbranch}) converts source-tree nodes to target-tree $\mnvu$ nodes, re-anchoring links in the process.
It is an auxiliary recursive function of \textproc{ShrinkSourceTree} (Alg. \ref{suppalg:shrinksourcetree}).
If a query is found at a leaf node, the query is removed to avoid data races with the new casting query in \textproc{ShrinkSourceTree}.
\renewcommand{\thealgorithm}{\theparagraph} %.\arabic{algorithm}}
\begin{algorithm}[!ht]
\begin{algorithmic}[1]
\caption{Converts source-tree branches to target-tree branches.}
\label{suppalg:convtotgtbranch}
\Function{ConvToTgtBranch}{$\mlink$}
    \If {$\fquery(\mlink) \ne \varnothing$}
        \State \Call{Unqueue}{$\fquery(\mlink)$}
        \Comment{Unqueue and remove any query on leaf links to avoid data races.}
    \EndIf
    \If {$\fnode(\mlink)\mdot\mtdirnode = T$}
        \State \Return
        \Comment{Reached leaf node of target-tree.}
    \EndIf
    \For{$\mlink_T \in \mlink\mdot\mlinks_T$}
        \State \Call{ConvToTgtBranch}{$\mlink_T, \fnode(\mlink)$}
    \EndFor

    \State $\mnode_S \gets \fnode(\flink(S, \mlink))$
    \State $\mnode_\mnew \gets $ \Call{GetNode}{$\fx(\mnode_S), \mnvu, \mnode_S\mdot\msidenode, T$}
    \State \Call{Anchor}{$\mlink, \mnode_\mnew$}
    \State $\mlink\mdot\mcost \gets$ \Call{Cost}{$\mlink$}
\EndFunction
\end{algorithmic}
\end{algorithm}
\renewcommand{\thealgorithm}{\thesubsubsection} %.\arabic{algorithm}}

%%%%%%%%%%%%%%%%%%%%%%%%%%%%%%%%%%%%%%%%%% ALG: CONVTOEXBRANCH %%%%%%%%%%%%%%%%%%%%%%%%%%
\subsubsection{\textproc{ConvToExBranch}: \textit{Identify Expensive Nodes After a Reached Cast.}}
The function \textproc{ConvToExBranch} converts existing nodes to expensive nodes when a reached cast finds a shorter path than other paths at a corner. 
The corner is represented by the Position object ($\mpos$).

The function is called by a reached cast when reaching a node results in the shortest path so far (see \textproc{SingleCumulativeVisibility} Alg. \ref{suppalg:singlecumulativevisibility}).
When $\mtdir = S$, $\mpos$ describes the corner at the target node, and cost-to-come is examined at the target node.
When $\mtdir = T$, $\mpos$ describes the corner at the source node, and the examined cost is cost-to-go is examined at the source node.
Links anchored on the existing $\mtdir$-tree $\mnvy$ node are identified, and connected $-\mtdir$ direction links are re-anchored on expensive nodes.
The function relies on the auxiliary recursive function \textproc{ConvToExBranchAux} to perform the conversion.
Expensive paths that are guaranteed to cross the current shortest path will be discarded.

\begin{algorithm}[!ht]
\begin{algorithmic}[1]
\caption{Converts a branch of $\mtdir$-tree $\mnvy$ nodes to $\mney$ nodes.}
\label{suppalg:convtoexbranch}
\Function{ConvToExBranch}{$\mtdir, \mpos$}
    \For {$\mnode \in \mpos\mdot\mnodes$}
        \If {$\mnode\mdot\mntype \ne \mnvy$ \Or $\mnode\mdot\mtdirnode \ne \mtdir$}
            \State \Continue 
            \Comment{ Check only $L$ or $R$-sided $\mtdir$-tree $\mnvy$ node/}
        \EndIf

        \State $\mbest \gets \fbest(\mtdir, \mpos)$
        \State $\mside_\mbest \gets \mbest\mdot\mnode_\mbest\mdot\msidenode$ 
        \State $\mv_\mbest \gets \mpos\mdot\mx - \mbest\mdot\mxbest$
        \For {$\mlink \in \mnode\mdot\mlinks_\mnode$}
            \State $\mlinkpar \gets \flink(\mtdir, \mlink)$
            \State $\mnodepar \gets \fnode(\mlinkpar)$
            \State $\mvpar \gets \mpos\mdot\mx - \fx(\mnodepar)$
            
            \If {$\mnode\mdot\msidenode = \mside_\mbest$ \An $\mtdir \mside_\mbest(\mv_\mbest \times \mvpar) > 0$}
            \Comment{Discard ex. path from $\mlink$ if it will cross the shortest path.}
                \State \Call{Disconnect}{$\mtdir, \mlink, \mlinkpar$}
                \State \Call{EraseTree}{$-\mtdir, \mlink$}
            \Else
                \State \Call{ConvToExBranchAux}{$\mtdir, \mlinkpar, \mlink, \mnode\mdot\msidenode$}
            \EndIf
            \State \Call{EraseTree}{$\mtdir, \mlinkpar$}
            \Comment{Discard $\mlinkpar$ if no more $(-\mtdir)$-links.} 
        \EndFor
    \EndFor
\EndFunction
\end{algorithmic}
\end{algorithm}

%%%%%%%%%%%%%%%%%%%%%%%%%%%%%%%%%%%%%%%%%% ALG: CONVTOEXBRANCH %%%%%%%%%%%%%%%%%%%%%%%%%%
\paragraph{\textproc{ConvToExBranchAux}: \textit{Create Expensive Nodes and Discard Expensive Paths}}
The function \textproc{ConvToExBranchAux} (Alg. \ref{suppalg:convtoexbranch}) converts $\mtdir$-tree $\mnvy$ nodes to $\mney$ nodes and reanchors links to the $\mney$ nodes.
It is the auxiliary recursive function for \textproc{ConvToExBranch} (Alg. \ref{suppalg:convtoexbranch}).

\textproc{ConvToExBranchAux} discards expensive paths if they will remain expensive in subsequent queries.
The function stops once a node that is not $\mnvy$ and not in the $\mtdir$-tree is reached.
If the reached node is a source-tree $\mnvu$ node, the function \textproc{ConvToTgtBranch} (Alg. \ref{suppalg:convtotgtbranch}) is used to shift queries to the most recent $\mney$-node to verify cost-to-come.

\renewcommand{\thealgorithm}{\theparagraph} %.\arabic{algorithm}}
\begin{algorithm}[!ht]
\begin{algorithmic}[1]
\caption{Converts $\mnvy$ nodes to $\mney$ nodes and discard paths that are guaranteed to be expensive.}
\label{suppalg:convtoexbranchaux}
\Function{ConvToExBranchAux}{$\mtdir, \mlinkpar, \mlink, \mside$}
    \State $\mnode \gets \fnode(\mlink)$

    \Comment{(A) Discard if par. node and current $\mnvy$ or $\mney$ node have different sides.}
    \If {$\mnode\mdot\msidenode \ne \mside$ \An $\mnode\mdot\mntype \in \{\mnvy, \mney\}$}
        \State \Call{Disconnect}{$\mtdir, \mlink, \mlinkpar$}
        \State \Call{EraseTree}{$-\mtdir, \mlink$}
        \State \Return

    \Comment{(B) Stop recursion once a node that is not $\mtdir$-tree and not $\mnvy$ is reached.}
    \ElsIf {$\mnode\mdot\mntype \ne \mnvy$} 
        \If {$\mtdir = S$ \An $\mnode\mdot\mntype = \mnvu$}
            \Comment{Shift queries down to most recent $\mney$ if a $S$-tree $\mnvu$ node is reached.}
            \State \Call{ConvToTgtBranch}{$\mlink$}
            \State \Call{Queue}{$\mqcast, \mlink\mdot\mcost + \mlinkpar\mdot\mcost, \mlink$}
        \EndIf
        \State \Return
    \EndIf
    \Comment{$\mnode$ is $\mside$-sided, $\mtdir$-tree $\mnvy$ node at this point.}
    
%     \If {$\mtdir = S$ \An $\mnode\mdot\mntype \ne \mnvu$}
% %         \If {$\mnode\mdot\msidenode \ne \mside$ \An $\mnode\mdot\mntype \in \{\mnvy, \mney\}$}
% %         \Comment{Discard if previous node is $\mney$, current node is going to be $\mney$, and both have 
% % different sides.}
% %             \State \Call{Disconnect}{$T, \mlink_\mpar, \mlink$}
% %             \State \Call{EraseTree}{$T, \mlink$}
% %             \State \Return
% %         \ElsIf {$\fquery(\mlink) \ne \varnothing$ \Or $\mnode\mdot\mntype = \mney$} 
% %         \Comment{Stop recursion if a leaf link or a $S$-tree expensive node is reached.}
% %             \State \Return
%         % \Els
    
%         \State \Return
%     \ElsIf {$\mtdir = T$ \An $\mnode\mdot\mntype \ne $} 
%         % \If {$\mnode\mdot\mntype \ne \mnvy $}
%         %     \State \Return 
%         % \ElsIf {$\mnode\mdot\msidenode \ne \mside$}
%         %     \State \Call{Disconnect}{$T, \mlink, \mlinkpar$}
%         %     \State \Call{EraseTree}{$S, \mlink$}
%         % \EndIf
%     \EndIf

    \Comment{(C) Recurse.}
    \For {$\mlink_\mathrm{chd} \in \flinks(-\mtdir, \mlink) $}
        \State \Call{ConvToExBranchAux}{$\mtdir, \mlink, \mlink_\mathrm{chd}, \mside$}
    \EndFor

    \Comment{(D) Delete or re-anchor the current link.}
    \If {$\flinks(-\mtdir, \mlink) = \{\}$} 
    \Comment{Delete $\mlink$ if its branch is discarded}
        \State \Call{Disconnect}{$\mtdir, \mlink, \mlinkpar$}
        \State Remove $\mlink$ from $\fnode(\mlink)\mdot\mlinks_\mnode$
    \Else 
    \Comment{Anchor $\mlink$ to $\mside$-sided, $\mtdir$-tree, $\mney$ node.}
        \State $\mnode_\mnew \gets $ \Call{GetNode}{$\fx(\mnode), \mney, \mside, \mtdir$}
        \State \Call{Anchor}{$\mlink, \mnode_\mnew$}
    \EndIf
\EndFunction
\end{algorithmic}
\end{algorithm}
\renewcommand{\thealgorithm}{\thesubsubsection} %.\arabic{algorithm}}